\documentclass[twoside,11pt]{article}

\usepackage{jmlr2e}

\usepackage{amsmath}
\usepackage[mathscr]{eucal}
\usepackage{graphics}
\usepackage{multicol}
\usepackage{algorithm2e}
\usepackage{xcolor}
\usepackage{enumerate}
\usepackage{subcaption}
\usepackage[pagewise,displaymath, mathlines]{lineno}
    \usepackage[mathscr]{eucal}
    \usepackage{verbatim}
    \usepackage{cases}
    \usepackage{graphicx}
    \usepackage{url}
    \usepackage{tensor}
    \usepackage{float}
    \usepackage{url}
    \usepackage{setspace} 
    \usepackage{stackengine}
    \usepackage[font=small,format=plain,labelfont=bf,textfont=it]{caption}

    \newcommand\barbelow[1]{\stackunder[1.2pt]{$#1$}{\rule{.8ex}{.075ex}}}

\renewenvironment{proof}{\paragraph{Proof:}}{$\;$\\ \hbox{$\;$}\hfill $\blacksquare$\\}

    \newtheorem{innercustomthm}{Fact}
    \newenvironment{customfact}[1]
    {\renewcommand\theinnercustomthm{#1}\innercustomthm}
    {\endinnercustomthm}

    \numberwithin{equation}{section}
    \numberwithin{figure}{section}

    \newcommand{\argmax}{\mathop{\rm argmax}}

    \newcommand{\R}{\mathbb{R}}

    \newcommand{\wh}{\widehat}

    \newcommand{\lamx}[2][f]{{\lambda_{#2}^{#1}(x)}}
     
    \newcommand{\lam}[2][f]{{\lambda_{#2}^{#1}}}
    \newcommand{\gamxt}[1][f]{{\gamma^{#1}(x,t)}}
    \newcommand{\gamx}[3][]{{\gamma_{#1}^{#2}(x,#3)}}
    \newcommand{\Vbot}[2][x]{V_{\bot}^{#2}(#1)}
    \newcommand{\Vx}[2][f]{V_{#2}^{#1}(x)}
    \newcommand{\Vy}[2][f]{V_{#2}^{#1}(y)}
    \newcommand{\V}[2][f]{V_{#2}^{#1}}
    \newcommand{\Proj}[2][x]{\Pi^{#2}(#1)}
    \newcommand{\xix}[1]{\xi^{#1}(x)}
    \newcommand{\xiy}[1]{\xi^{#1}(y)}
    \newcommand{\xii}[2]{\xi^{#1}(#2)}
    \newcommand{\xiii}[3][]{{\xi_{#1}^{#2}(x,#3)}}
    \newcommand{\xinox}[1]{\xi^{#1}}
    
    \newcommand{\Seps}[1]{S_\epsilon^{#1}}
    \newcommand{\Sep}[2]{S_{#2}^{#1}}
    
    \newcommand{\wt}{\widetilde}

    \newcommand{\newchapter}[3] 
	{                           
        \chapter[#2]{#3}
        \chaptermark{#1}
        \thispagestyle{myheadings}
	}
\allowdisplaybreaks
\ShortHeadings{Ridge Estimation Algorithms}{Qiao and Polonik}
\firstpageno{1}

\begin{document}

\title{Algorithms for ridge estimation with convergence guarantees \\[20pt]}
%
\author{\name Wanli Qiao \email wqiao@gmu.edu\\
    \addr Department of Statistics\\
    George Mason University\\
    4400 University Drive, MS 4A7\\
    Fairfax, VA 22030, USA
    \AND
    \name Wolfgang Polonik \email wpolonik@ucdavis.edu\\
    \addr Department of Statistics\\
    University of California\\
    One Shields Ave.\\
    Davis, CA 95616, USA}

\maketitle
\vspace*{-1cm}

\centerline{\today}
\vspace*{0.5cm}
\begin{abstract} \noindent The extraction of filamentary structure from a point cloud is discussed. The filaments are modeled as ridge lines or higher dimensional ridges of an underlying density. We propose two novel algorithms, and provide theoretical guarantees for their convergences, by which we mean that the algorithms can asymptotically recover the full ridge set. We consider the new algorithms as alternatives to the Subspace Constrained Mean Shift (SCMS) algorithm for which no such theoretical guarantees are known.
\end{abstract}
%
%
\reversemarginpar
\section{Introduction} \label{intro}
The geometric interpretation of a ridge in $\R^d$ is that of a low-dimensional structure along which the density is higher than in the surrounding area when moving away from the set in an orthogonal direction. Blood vessels, road system, DNA strands or fault lines appearing in 2D or 3D images can be modeled as filaments, or maybe better as unions of filaments that might intersect, or that have a common starting point. We sometimes call such unions {\em filamentary structures}. Another example is provided by the filamentary structure that can be observed in the distribution of galaxies in the universe, the `cosmic web'. Cosmologists are interested in finding rigorous geometric and topological descriptions of the filamentary structures~\citep{novikov2006skeleton}. Usually, the first step is the extraction of this structure, and this is the topic discussed below.

Ridges characterize the low-dimensional structures as collection of local maxima of probability density functions in local orthogonal subspace. See \cite{genovese2014nonparametric}, \cite{chen2015asymptotic}, \cite{qiao2016theoretical}, and \cite{qiao2021asymptotic} for statistical analyses of ridge estimation. %

The estimation of ridges is related to the problem of manifold learning, for which it is assumed that data are observed near a manifold with noise and the task is to recover the manifold. It has been shown in \cite{genovese2014nonparametric} that ridges can be used as surrogates for manifold estimation. Some useful references for manifold learning include \cite{niyogi2008finding}, \cite{genovese2012minimax}, \cite{fefferman2020reconstruction}, \cite{yao2023manifold} and the references therein. %

Ridges can also be used for the purpose of (non-linear) dimension reduction. This falls under the more general umbrella of statistical embedding where the goal is to find a low-dimensional representation of the data that is not necessarily in the ambient space but preserves the geometric and/or topological structures in the original data. See \cite{tjostheim2023statistical} for a recent survey on this topic.

The literature on the estimation of low-dimensional structures (or filaments) is rich, and different approaches use different geometric ideas. For example, the local principal curves~\citep{einbeck2005local,einbeck2008representing} are formed by tracking the localized version of the first principal component directions, but the method requires selection of good starting points lying on or near the filaments already.  The candy model~\citep{stoica2007three} uses possibly connected cylinders (in 3D) of a fixed radius and
height to represent the filaments. The medial axis of the data
distribution's support~\citep{genovese2012geometry} can also be used to estimate filaments, under the assumption that the noise around the filaments is symmetric. The multiscale method developed in~\cite{arias2006adaptive} can be used to detect the presence of a single filament in a noise background. However, the method does not provide a low-dimensional filament estimate. \cite{genovese2009path} proposed the concept of the path density to characterize filaments, which does not lend itself to a low-dimensional estimator, either.

One of the most well-known concepts for the extraction of low-dimensional features is principal curves~\citep{hastie1989principal}, which generalized PCA in the nonlinear setting. The principal curve is a smooth curve that passes through the middle of a data set. Any point on a principal curve is defined as the conditional expectation of all the data that project to that point, which is a property called self-consistency. Following this concept, a lot of research work has been generated to investigate the properties and algorithms for principal curves. See, for example, \cite{banfield1992ice}, \cite{tibshirani1992principal}, \cite{stanford2000finding} and \cite{verbeek2002k}. The fact that principal curves do not always exist, motivated a related line of work on modified principal curves, which started with \cite{kegl2000learning}. See also \cite{biau2011parameter} and \cite{delattre2020principal} and references therein. 

If the idea of local averaging in the original definition of principal curves is replaced by that of taking a local maximum in the orthogonal subspace, then we obtain the concept of ridges, which first appears in the literature of image analysis. See \cite{eberly1996ridges}. Ridges have a mathematical definition using derivatives up to second order and they come with intuitive geometric interpretation (see Section~\ref{defsec}). In practice, ridges can be used to estimate filaments with flexible shapes and structures without strong requirements on the starting points of algorithms for ridge extraction. As shown in \cite{ozertem2011locally}, the ridge estimators can perform well even when there are loops, bifurcations, and self intersections in data, while these are difficult to handle for the principal curve method.

In this work, we are concerned with the actual extraction of ridges, i.e. we will construct and analyze algorithms. One such algorithm is the popular SCMS (Subspace Constrained Mean Shift) algorithm developed by \cite{ozertem2011locally}, which extracts $k$-dimensional ridges of a $d$-dimensional density $f$ from a point cloud sampled from $f$. The algorithm consists in running a corrected (i.e., subspace constrained) mean shift algorithm starting at a data point. For each data point the algorithm then provides an estimate of a point on the ridge. 

However, there does not seem to exist theoretical guarantees for the SCMS algorithm to consistently estimate the full ridge set, and, as discussed below (see Section~\ref{SCMSTheory} and the Appendix), the SCMS algorithm might miss some parts of the ridge, although the point-wise convergence property of SCMS is studied in \cite{zhang2023linear}. In other words, it is not entirely clear what the SCMS is actually estimating. Even though it appears that this does not have a serious impact in many  practical examples, this theoretical gap provides a motivation for developing alternative ridge finding algorithms that (i) come with theoretical guarantees offering deeper insights to their behavior, and (ii) do not suffer from potentially missing some parts of the ridge. We mention in passing that there exists another ridge estimation algorithm developed in \cite{pulkkinen2015ridge}, which tracks the ridge lines. However, it relies on a starting point that has been on or close to the ridge.

The remaining part of the paper is organized as follows. In Section \ref{Extraction} we introduce the formal definition of ridges. This is followed by our extraction algorithms, whose performance is illustrated using some numerical studies in $\mathbb{R}^2$. The main theoretical results are given in Section~\ref{mainresults}, where we give the convergence results of our algorithms. The mathematical framework for the theoretical analyses is provided in Section~\ref{maththeory}. In Appendix~\ref{appendix:a} we give an example for which the SCMS algorithm fails to detect a part of the ridge while our algorithms do not miss it. All the proofs are provided in Appendix~\ref{appendix:b}. 

\section{Extraction of filamentary structures}\label{Extraction}

\subsection{Definition}
\label{defsec}
Let $f$ denote a density on $\R^d$ from which data will be drawn. We will assume that $f$ is (at least) twice differentiable. The definition of ridge (or filament) points is as follows:

\begin{definition}\label{FilaDef1} (Ridge in $\mathbb{R}^d$). Let $(\lambda_i^f(x), \Vx{i}), i = 1,\ldots,d$ be eigenpairs of the Hessian $\nabla^2f(x)$ with $\lambda^f_1(x) \geq \cdots \geq \lambda^f_d(x)$. Let $1\leq k \leq d-1$. With $\Vbot{f} = [\Vx{k+1},\cdots,\Vx{d}] \in {\mathbb R}^{d \times (d-k)}$ the matrix of the trailing $(d-k)$  eigenvectors, we define  
\begin{align}\label{DEF}
\text{\rm Ridge}_k(f) = \big\{x\in\mathbb{R}^d:\;  \Vbot{f}^\top \nabla f(x) ={\bf 0}\; \text{and } \; \lamx{k+1} <0\big\}.
\end{align}
\end{definition}
The geometric intuition underlying this definition is the following: Since the (first order) directional derivative of $f(x)$ along $V^f_i(x)$ is $\langle \nabla f(x), \Vx{i}\rangle$ and the second order directional derivative of $f(x)$ along $V^f_i(x)$ is $\lamx{i}$, the two conditions in (\ref{DEF}) mean that a point $x$ on the ridge is a local mode in the linear subspace spanned by $\Vx{i}, \; i=k+1,\cdots,d$, for which the density has the largest concavity (see Eberly, 1996).

Given an iid sample $X_1,\cdots,X_n$  from $f$,  we estimate Ridge$_k(f)$ by 
$$\text{Ridge}_k(\wh f\,) = \big\{x\in\mathbb{R}^d:\;  \Vbot{\wh f}^\top \nabla \wh f(x) ={\bf 0}\; \text{ and } \;  \lambda^{\wh f}_{k+1}(x)<0\big\},$$
where $\wh f$ is a KDE, $\big(\lamx[\wh f]{i},\Vx[\wh f]{i}\big),i=1,\ldots,d$ are the eigenpairs of the Hessian $\nabla^2 \wh f(x)$, where we assume the eigenvalues to be sorted as $\lamx[\wh f]{1} \ge \cdots \ge \lamx[\wh f]{d}$. Furthermore, let 
$\Vbot{\wh f} = \big[\Vx[\wh f]{k+1},\cdots,\Vx[\wh f]{d}\big]$ be the matrix of the $(d-k)$ trailing orthonormal eigenvectors of $\nabla^2\wh f(x)$. We recall that the KDE has the form 
$$\wh f(x) = \frac{1}{nh^d} \sum_{i=1}^n K\left( \frac{X_i - x}{h}\right),$$
where $h>0$ is a bandwidth, $K \ge 0$ and $\int_{\mathbb{R}^d} K(u)du = 1$.
The goal now is to construct and study algorithms to extract the estimated ridges from data. It is well known that KDE suffers from the curse of dimensionality and so we do not consider high dimensions in this paper.

Note that the matrices $\Vbot{f}$ and $\Vbot{\wh f}$ (and similar matrices defined below) depend on $k$ (and $d$). Since $k$ (and $d$) are held fixed in the theoretical and methodological developments of this paper, this dependence is not indicated in the notation.

\subsection{Mean shift algorithm and subspace constrained mean shift algorithm}
The popular mean shift algorithm \citep{fukunaga1975estimation} is being used for mode finding and clustering, e.g. see \cite{cheng1995mean} and \cite{comaniciu1999mean, comaniciu2002mean}. It is tracking non-parametric estimates of gradients of a KDE. Using the KDE with differentiable $K(x)=\phi(\|x\|^2)$, the vector 
\begin{align}\label{MeanShift1}
m(x):=\sum_{i=1}^n w_i(x)X_i - x\qquad\text{with}\qquad w_i(x) = \frac{\phi^\prime\big(\big\|\frac{x-X_i}{h}\big\|^2\big)}{\sum_{i=1}^n\phi^\prime\big(\big\|\frac{x-X_i}{h}\big\|^2\big)}
\end{align}
is called \emph{mean shift}. It is well known that the direction of the mean shift is an estimator of the direction of the gradient of $f$ at $x$. Given some initial position $y_0$, the basic mean shift algorithm iteratively finds a sequence of points $y_1, y_2, \cdots$ by
\begin{align}\label{update}
y_{j+1}=m(y_j) + y_j = \sum_{i=1}^n w_i(y_j)X_i,\;\; j=0,1,2,\cdots. 
\end{align}
Successively connecting these points provides an estimate of the integral curve driven by the gradient, starting from $y_0$. See \cite{arias2016estimation} and \cite{arias2022clustering}. The endpoint of this iteration (after applying some stopping criterion) is considered to be an estimate of a mode (local maximum) of $f$. When running this algorithm repeatedly with each data point as a starting point, clusters can be formed by grouping all the data points for which the mean shift algorithm has the same endpoint. 

Subspace constraints come into the picture when the target is a ridge rather than a mode. The construction of the SCMS algorithm modifies the just described hill climbing algorithm by following the direction of the gradient projected on the subspace spanned by the trailing $(d-k)$ eigenvectors of the Hessian of the KDE. The gradient direction is approximated by the mean shift. More specifically, given a starting point $y_0$, the SCMS generates a sequence
\begin{align}\label{scms}
y_{j+1} = \Proj[y_j]{\wh f}\; m(y_j) + y_j,\quad j = 0, 1,2,\ldots
\end{align}
where $m(y)$ is the mean shift vector, and $\Proj[y]{\wh f} =  \Vbot[y]{\wh f} \Vbot[y]{\wh f}^\top$ is the projection matrix onto the subspace spanned by the trailing $(d-k)$ eigenvectors of the Hessian of the KDE evaluated at $y$. Notice that for $y \in {\rm Ridge}(\wh f\,)$, this space is orthogonal to $\nabla \wh f(y)$. This motivates that the endpoint of this iteration (after applying some stopping criterion) is considered to be an estimated ridge point. We note that in the original SCMS algorithm proposed in \cite{ozertem2011locally}, $\Proj[y_j]{\wh f}$ is replaced by $\Proj[y_j]{\log\wh f}$, and this corresponds to the ridge estimation of $\log f$. We focus on the ridge estimation of $f$ in this paper, although the analysis can be easily adapted to that of $\log f$. The SCMS algorithm is very popular, and it gives nice results in practice. However, as discussed next (and in the Appendix~\ref{appendix:a}), the SCMS algorithm might miss some parts of the ridge. 

\subsection{The SCMS algorithm might miss some parts of the ridge}\label{SCMSTheory}
By definition, a ridge point $x_0$ is a local maximum in the directions given by the columns of $\Vbot[x_0]{f}$. So the goal of the SCMS algorithm is to stay in this subspace by projecting the mean shift vectors back into this space in each iteration step. Not knowing the ridge points (i.e. not knowing the target subspace), the algorithm projects the gradient at the current iterate $y_j$ on the subspace spanned by the trailing eigenvectors of the Hessian at this point $y_j$. Thus, the subspace to project on changes with each iteration step. Indeed these subspaces are tangent spaces to the integral curve/surface traced by the SCMS algorithm. It turns out that, because of this, the signs of the directional derivatives taken along these curves/surfaces are  not necessarily determined by the signs of the eigenvalues of the Hessian. As a consequence, ridge points are not necessarily local maximizers when traveling along the integral curves/surfaces traced by the SCMS algorithm, but they can also be local minimizers or saddle points, and if they are, they are not identified as a ridge point by the algorithm. More details along with an example are provided in Appendix~\ref{appendix:a}.

\subsection{Measuring ridgeness}

The new ridge finding algorithms proposed in this work are based on the following ridgeness measure (actually measuring `non-ridgeness'):
\begin{align}\label{TargetFun}
\eta(x) = -\frac{1}{2}\big\|\Vbot{f}^\top\nabla f(x)\big\|^2.
\end{align}
Since we will assume that $\lamx{k} \ne \lamx{k+1}$ for all $x$ (see assumption {\bf (A2)} below), $\eta(x)$ is well-defined. According to their definition, ridge points can be described as 
\begin{align}
&\eta(x) = 0\;\;\;\text{s.t.} \;\;\;\lamx{k+1} <0,\label{eigenvalue-cond}
\end{align}
or, since $\eta\le 0,$ ridge points are global maximizers of $\eta$. In other words, ridge points can be described as 
\begin{align}
&\argmax_{x}\eta(x)\;\;\;\text{s.t.} \;\;\;\lamx{k+1} <0.\nonumber
\end{align}
Note that here `$\argmax$' denotes the entire set of maximizers. From this point of view, ridge finding algorithms can be obtained as algorithms maximizing the data-dependent version of $\eta(x)$ given by
\begin{align}\label{KernelTargetFun}
\text{find {\em all} maximizers of }\;\;\;\wh \eta(x) = -\frac{1}{2}\big\| \Vbot{\wh f}^\top\nabla \wh f(x)\big\|^2,\quad \text{subject to }\;\;\lamx[\wh f]{k+1}<0.
\end{align}
The constraint on the eigenvalue will simply be enforced by excluding maximizers violating this condition.

\subsection{Notation}\label{notation}
Here we collect some important notation used in this work. Let $g,h:\R^d \to \R$ be a twice differentiable function. Then we use the following notation:
\begin{enumerate}[$\bullet$]
    \item $\lamx[g]{1} \ge \cdots \ge \lamx[g]{d}$:\;\;sorted eigenvalues of Hessian of $g$
    \item $\Vx[g]{i}$:\;\;eigenvector of Hessian of $g$ corresponding to $\lamx[g]{i}, i = 1,\ldots,d$ 
    \item $\Vbot{g}$:\; $(d \times (d-k))$-matrix formed by the $(d-k)$ trailing eigenvectors of Hessian of $g$
    \item $\Proj{\,g} = \Vbot{g}\Vbot{g}^\top$:\; projection matrix onto linear space spanned by columns of $\Vbot{g}$
    \item $\xix{g} = \Proj{\,g}\nabla g(x)$:\;\;projection of the gradient of $g$ onto the linear space spanned by\\ \hspace*{3.6cm} $(d-k)$ trailing eigenvectors of Hessian of $g$
    \item $\eta(x) = -\frac{1}{2}\big\| \Vbot{f}^\top\nabla f(x)\big\|^2 = -\frac{1}{2} \|\xix{f}\|^2$:\;ridgeness function of $f$
    \item $\wh \eta(x) = -\frac{1}{2}\big\| \Vbot{\wh f}^\top\nabla \wh f(x)\big\|^2 = -\frac{1}{2} \|\xix{\wh f}\|^2$:\;ridgeness function of KDE $\wh f$
    \item $\wh \eta_\tau(x)$:\; smoothed ridgeness function, where $\tau$ is a smoothing parameter (see \ref{smoothed-ridgeness}).
    \item \mbox{\;}\vspace*{-0.5cm} 
\begin{align}\label{Sepsilon}
    \Seps{g,h}:=\{x\in [0,1]^d:\; g(x)\geq -\epsilon, \lamx[h]{k+1}<0\},
\end{align}
\end{enumerate}
Recall that $k \in \{0,\ldots,d-1\}$ is fixed throughout the manuscript (except for the numerical illustrations).
\subsection{Algorithms}

To compute the maximizers of the ridgeness function we now consider two algorithms based on $\wh \eta(x)$. 
As can be seen below, our proposed algorithms target the ridge of the ridgeness function, and we will see below (see Lemma~\ref{ridge-ridgeness}) that the ridge of the ridgeness function essentially equals the original ridge of $f$.

In the following, we assume that all the functions considered are defined on $[0,1]^d.$ \\

{\bf Basic Algorithm 1}: {\em Alternative SCMS approach using an estimated ridgeness function.}\\

Observing that  $\nabla \wh\eta(x)= - [\nabla \xix{\wh f}]^\top \xix{\wh f},$ we have the following SCMS-type algorithm: Given $a>0$ (step size) and a starting point $y^0$, we update through
\begin{align*}
y^{j+1} &= y^j + a\, \xii{\wh \eta}{y^j}\\
&= y^j - a\, \Proj[y^j]{\wh \eta}\;[\nabla \xii{\wh f}{y^j}]^\top\; \xii{\wh f}{y^j}\\
&= y^j - a\, \Proj[y^j]{\wh \eta}\; [\nabla  \xii{\wh f}{y^j}]^\top  \Proj[y^j]{\wh f}\;\nabla \wh f(y^j),\; j=0,1,2,\cdots 
\end{align*}
More precisely, the structure of the algorithm is as follows:\\ 

{\bf Input:} $y_i^{0}= X_i$, $i=1,\cdots,n$, $a>0$, $h>0$.\\

{\bf Update}: For $i = 0, 1,2,\ldots,n$, for $j=1,2,\ldots,$ \\[5pt]
\hspace*{1cm} {\bf while $y_i^j\in[0,1]^d\,$}:\\[-20pt] 
\begin{align}\label{algorithm1update}
y_i^{j+1} = y_i^j - a\, \Proj[y_i^j]{\wh \eta}\;[\nabla  \xii{\wh f}{y_i^j}]^\top\; \Proj[y^j_i]{\wh f}\; \nabla \wh f(y_i^j),
\end{align}
\hspace*{1cm}{{\bf else:} discard the entire sequence $y_i^0,y_i^1,\ldots$\;\;}\\[-10pt]

{\bf Output:} $\qquad\qquad\{y_i^{\infty}:\;  \wh\eta(y_i^{\infty})=0, \; \lam[\wh f]{k+1}(y_i^{\infty})<0 \}.$\\

In the output step, we remove points that do not comply with the condition for eigenvalues in the definition of ridges, because this condition is not being taken into account when constructing the iterations of the algorithm. \\

{\bf Basic Algorithm 2}: {\em Alternative SCMS approach using a smoothed estimated ridgeness function.}\\

Let $\tau>0$ be another bandwidth and $ L : \R^d \to \R_{\geq0}$ be a twice differentiable kernel function. Define a smoothed version of the ridgeness function $\wh \eta(x)$ as
\begin{align}\label{smoothed-ridgeness}
\wh \eta_{\tau}(x) = \frac{1}{\tau^d} \int_{\mathbb{R}^d} L\left( \frac{x-u}{\tau} \right) \wh \eta(u) du.  
\end{align}
Our algorithm will approximate the ridge of this smoothed version of the ridgeness function. Using this smoothed version has some computational advantages (see Section~\ref{more-pract} below). 
Let $\Vbot{\wh\eta_\tau}$ be the matrix whose columns are the trailing $(d-k)$ orthonormal eigenvectors of $\nabla^2 \wh \eta_\tau(x)$, and let $\xix{\wh \eta_\tau} = \Proj{\wh \eta_\tau}\;\nabla  \wh \eta_\tau(x)$, where $\Proj{\wh \eta_\tau} =  \Vbot{\wh \eta_\tau} \Vbot{\wh  \eta_\tau}^\top.$ With this notation, the structure of the algorithm is as follows: \\

{\bf Input:} Given $y_i^{0}= X_i$, $i=1,\cdots,n$, $a>0$, $\tau>0$, $h>0$.\\

{\bf Update}: For $i = 0,1,2,\ldots,n$, for $j=1,2,\ldots,$ \\[5pt]
\hspace*{1cm} {\bf while $y_i^j\in[0,1]^d\,$}:\\[-20pt] 
\begin{align*}
y_i^{j+1} &= y_i^j + a\, \xii{\wh \eta_\tau}{y_i^j}\\
&= y_i^j + a\,\Proj[y^j_i]{\wh \eta_\tau}\; \nabla  \wh \eta_\tau(y_i^j),
\end{align*}
\hspace*{1cm}{{\bf else:} discard the entire sequence $y_i^0,y_i^1,\ldots$\;\;}\\[-10pt]

{\bf Output:} 
$\qquad\qquad
\{y_i^{\infty}:\; \wh\eta_\tau(y_i^{\infty})=0, \; \lam[\wh f]{r+1}(y_i^{\infty})<0 \}.$

\begin{remark}
    While the algorithms above are defined by using the data as starting points, we do have other options. What we need is a set of starting points that becomes dense in the support $[0,1]^d$, such as a fine grid, where, for the theory, the grid size would tend to zero. The theory presented in this work is using a continuous set. It seems possible to extend this theory to the case of the data being the starting points, but we do not pursue this case here.
\end{remark}

\subsection{Practical implementation and illustration of the algorithms}\label{pruning}

For practical implementation, the basic algorithms given above require a stopping criterion, choice of tuning parameters, and some additional  pre- and post-processing. This, along with some other aspects that are of some importance in the actual implementation of the algorithms, are discussed in the following. In the above two basic algorithms, for each $i$, we stop the iterations when $\|y_i^{j}-y_i^{j-1}\|\leq \varepsilon_{\text{tol}}$ for some small tolerance $\varepsilon_{\text{tol}}>0$, and we use $y_i^{j}$ as the final point, denoted $y_i^*$, for the sequence starting from $y_i^{0}$. Here $\varepsilon_{\text{tol}}$ can be chosen as small as possible, with smaller $\varepsilon_{\text{tol}}$ giving better accuracy and heavier computation cost.

The bandwidth $h$ in the kernel estimators impacts the rate of convergence for ridge estimation. Since the ridge is determined by up to the second order derivatives of $f$, it is recommended that the optimal bandwidth for the estimation of the Hessian of $f$ is used for ridge estimation. However, if the density on the ridge is completely flat, the ridge becomes a set of degenerate local maxima and then the optimal bandwidth for gradient estimation should be used. See \cite{qiao2023confidence}. Plug-in and cross-validation approaches are most commonly used for selecting such bandwidths \cite[see][]{chacon2013data}. 

The smoothing parameter $\tau$ used in the Basic Algorithm 2 can be chosen as small as possible. In fact, our theoretical results in Theorem~\ref{main1} suggests that $\tau$ will not change the rate of convergence for ridge estimation as along as $\tau=O(h)\to0$. However, if the goal is to approximate $\text{Ridge}(\wh f)$ through the Basic Algorithm 2, then $\tau$ should be much smaller than $h$. In practice, the computation of the integral in \eqref{smoothed-ridgeness} can be done using numerical methods. For example, one can evaluate $\wh \eta$ on a grid, and based on this, \eqref{smoothed-ridgeness} can be approximated by a Riemann sum (see \eqref{RiemannSum}). If the kernel $L$ has bounded support, then $\tau$ determines the number of grid points near $x$ effectively used in the approximation. The interplay between the grid size and $\tau$ again leads to a question between balancing accuracy and computation cost.

The step size $a$ plays a similar role as the learning rate in gradient descent. It is well-known that if such a hyper-parameter is too large, then a overshooting problem can occur. Usually it is safer to use a small $a$ to achieve convergence of the algorithms, with the downside of slow convergence speed. To find a balance and determine a good value of $a$, one can use trial and error, which is in fact a commonly used approach for choosing the learning rate for gradient descent~\citep{bengio2012practical}. 

In practice our algorithms can encounter the following two challenges. The first is posed by low density regions, where the estimated density tends to be flat leading to possible spurious ridge points identified by the algorithms. We note that this challenge is faced by all ridge estimation algorithms, due to the fact that ridges are local features which may arise in any low-estimated-density regions as long as the density is positive, even when true ridges do not exist in these regions. A second challenge is possible local (but non-global) modes of our ridgeness function $\eta$, which again might lead to spurious ridge points. This challenge is relative easy to handle, because the global maximum of $\eta$ is known, which is zero. This known maximum provides a way to distinguish between local and global modes of $\eta$. We address these challenges by introducing the following pre-processing and post-processing steps. 

\subsubsection{Pre-processing}
In our basic algorithms (and also in the theory presented below), we assume that all the ridges considered are defined on $[0,1]^d,$ which is supposed to be contained in the support of $f$. In the actual implementation, however, we are replacing $[0,1]^d$ by the set $\{x \in \R^d: \wh f(x) \ge \varepsilon_f\},$ for a given threshold $\varepsilon_f > 0.$ We choose $\varepsilon_f$ as an $\alpha$-quantile of the distribution of $\{\wh f(X_1), \ldots,\wh f(X_n)\}.$ In our numerical studies, we used $\alpha = 5\%$. A similar idea of using $\varepsilon_f$ has been suggested in~\cite{genovese2014nonparametric}. Notice that under weak assumption, the upper-level set $\{x \in \R^d: \wh f(x) \ge \varepsilon_f\}$ is known to be a consistent estimate of $\{x \in \R^d: f(x) \ge \varepsilon_f\}$ (see, e.g., \cite{qiao2019nonparametric}), and in our theoretical investigations, we could replace the compact set $[0,1]^d$ by $\{x \in \R^d: f(x) \ge \varepsilon_f\},$ and assume that the density at the ridge points is larger than $2\varepsilon_f,$ say. Using a consistent data-dependent estimate for $\varepsilon_f$ could also be dealt with theoretically, even though we are not explicitly considering this in the theory section. Note that the estimated upper level set plays the role of $[0,1]^d$ in the algorithms, so that if a sequence moves out of this region, it will be discarded. Alternatively, one can also use an estimate of the density support, for example, by using the alpha-convex hull~\citep{rodriguez2007set}.  
\subsubsection{Post-processing}
Low density points are now excluded by our pre-processing step. 

To address the problem of possible local maxima of the functions $\wh\eta$ and $\wh \eta_\tau$, we remove an output point $y_i^*$, if $\wh \eta(y_i^*) < - \varepsilon_\eta$, and $\wh\eta_\tau(y_i^*) < - \varepsilon_\eta$, respectively, for some small $\varepsilon_\eta > 0$. Note that in an ideal scenario where we can obtain $y_i^\infty$ as in the two basic algorithms, we can choose $\varepsilon_\eta=0$, because these two algorithms converge to the estimated ridge as $j\to\infty$, as shown in our theoretical results in Theorem~\ref{main2}. However, in practice, we impose a stopping criterion with tolerance $\varepsilon_{\text{tol}}$, which prevents the algorithms from reaching the estimated ridge points exactly. This also explains the necessity of $\varepsilon_\eta > 0$, which intrinsically depends on $\varepsilon_{\text{tol}}$. If $\varepsilon_{\text{tol}}$ is chosen very small, so should be $\varepsilon_\eta$. However, we do not know a fixed relation between these two parameters so that the latter can automatically chosen based on the former.

A practical choice of $\varepsilon_\eta$ is as follows. Consider the quantile function of the (one-dimensional, discrete) distribution given by the values $\{\wh \eta(y_1^*), \ldots, \wh \eta(y_n^*)\}$ assuming that all the $y_i^*$ are included in the final output, and choose $\varepsilon_\eta$ as the location of the `last significant jump' of the quantile function. In other words, we are removing `outliers' in this distribution. It is possible to automatically identify such jumps by using change point detection techniques, however we do not pursue this further. In our numerical experiments, this first significant jump was clearly visible. An example is given in Figures~\ref{cross-sin} and \ref{post-processing} below, which are based on Algorithm 2 with step length $a=0.005$, while Algorithm 1 gives very similar results.

\begin{figure}[H]
\centering
\includegraphics[scale=0.22]{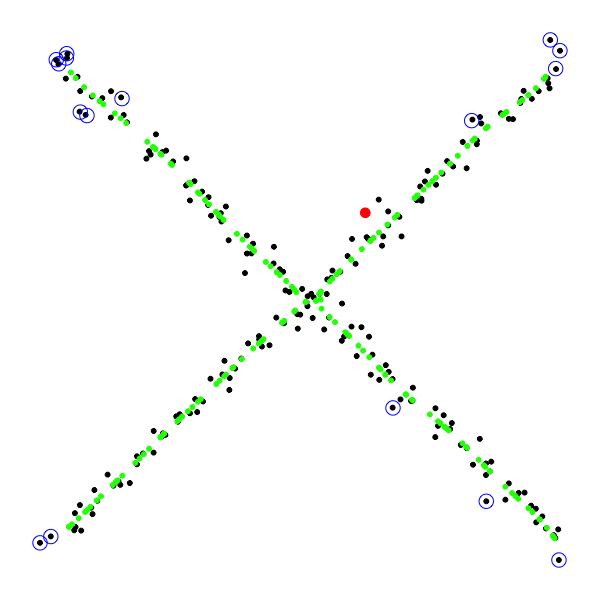}
\includegraphics[scale=0.22]{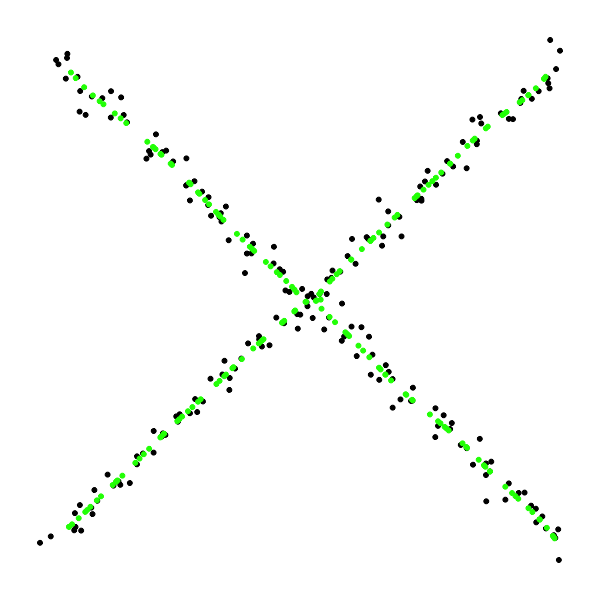}
\caption{X-cross example with 200 data points, where the black solid dots are data points; the blue circles are the points removed by pre-processing; the red dot is the point removed by post-processing; the green dots are the final output of the algorithm. The right panel shows the final result with the red dot and blue circles removed from the left panel.}
\label{cross-sin}
\end{figure}

\begin{figure}[H]
\centering
\includegraphics[scale=0.3]{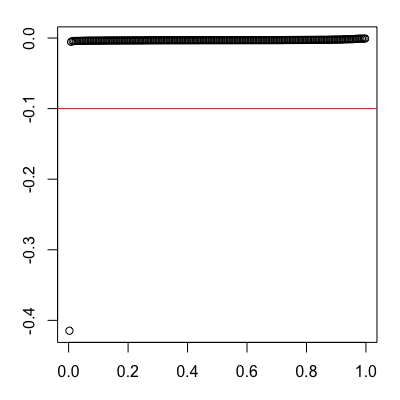}
\caption{Sorted ridgeness values of the output of Algorithm 2 run on the 200 data points with X-cross as ridges  without imposing the threshold $\varepsilon_{\eta}$. A clear jump is visible between the ridgeness values close to zero  and those of spurious ridge points. This observation then informed our choice of  $\varepsilon_{\eta}=-0.1$ (red line). The point below this threshold corresponds to the red dot in the left panel of Figure~\ref{cross-sin}. }
\label{post-processing}
\end{figure}

Our algorithms are also implemented using two additional data sets, as shown Figures~\ref{circle} and~\ref{spiral}, for which we only present the final results after the pre-processing and post-processing steps, with the black and green dots having the same meaning as in Figure~\ref{cross-sin}.

\begin{multicols}{2}
\begin{center}
\includegraphics[height=4cm]{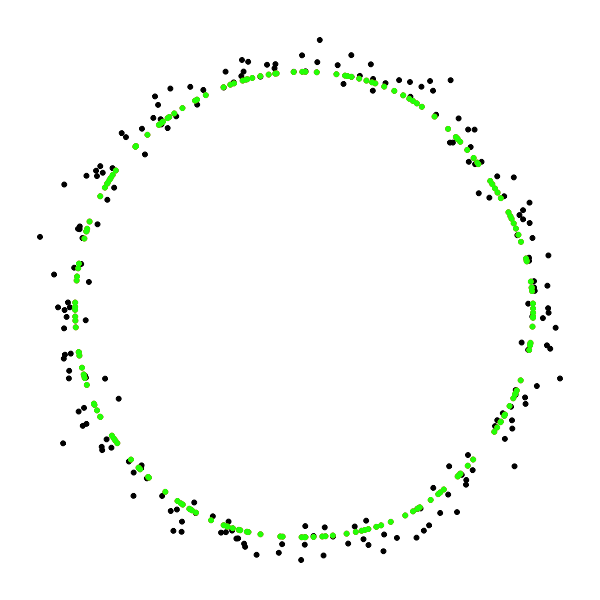}
\captionof{figure}{Circle: 200 data points.}\label{circle}
\includegraphics[height=4cm,angle=90,origin=c]{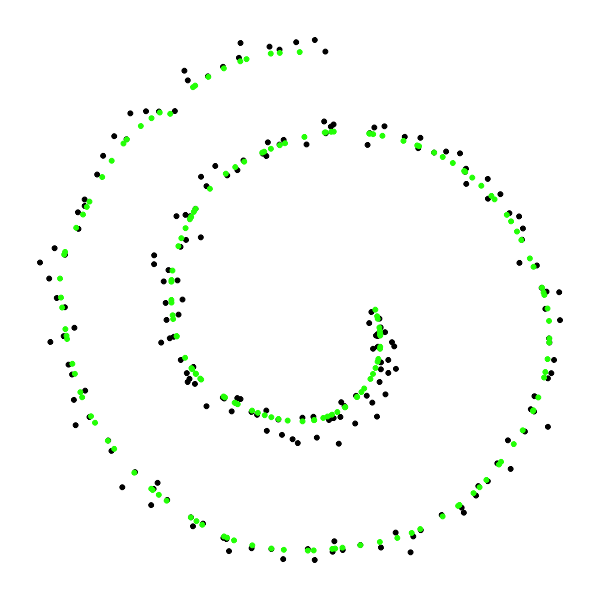}
\captionof{figure}{Spiral: 200 data points.}\label{spiral}
\end{center}
\end{multicols}

\subsubsection{Other more practical aspects}\label{more-pract}

In Algorithm 1, we need to compute $\nabla \wh \xi(x)$ and eigenvalues and eigenvectors of $\nabla^2 \wh\eta(x).$ This requires some matrix algebra, and we present some explicit formulas in Appendix~\ref{appendix:a}. 

Algorithm 2 avoids direct evaluation of these formulas, and thus has some computational advantages in practice. Indeed the connection between the two algorithms can be understood in such a way that the symbolic computation of $\nabla\wh\eta$ and $\nabla^2\wh\eta$ in Algorithm 1 are replaced by their numerical approximations ($\nabla\wh\eta_\tau$ and $\nabla^2\wh\eta_\tau$) based on the evaluations of $\wh\eta$ in Algorithm 2. In practice, we implement the computation of $\nabla\wh\eta_\tau$ and $\nabla^2\wh\eta_\tau$ as follows. Let $\{x_i,i=1,2,\cdots\}$ be a grid over $\mathbb{R}^d$ with grid length $\rho<\tau$. Then $\nabla\wh\eta_\tau$ is approximated by 
\begin{align}
\label{RiemannSum}
\frac{\rho^d}{\tau^d} \sum_{i} \nabla \Big[L\left( \frac{x-x_i}{\tau} \right)\Big] \wh \eta(x_i),  
\end{align}
and $\nabla^2\wh\eta_\tau$ can be approximated in a similar way. The kernel $L$ is often chosen with bounded support or as the Gaussian kernel with truncation so that there are only a limited number of grid points involved in the above summation.

\subsection{Simulation study}\label{simulation}

We ran a small simulation study for the performance of our algorithms 1 and 2. We consider the ridge set of the density $f$ of a bivariate random vector $X = Y + Z$, where $Y$ has a distribution restricted on a circle $\mathcal{C}$ with center at the origin and unit radius, and $Z\sim N(0,0.05^2)$. Two models depending on the distributions of $Y=(\cos(\Theta),\sin(\Theta))$ were used: $\Theta$ has a uniform distribution on $[0,2\pi]$ for Model 1 and $\Theta$ has a density $[\sin(\theta)+2]/4\pi$ for $\theta\in[0,2\pi]$. Note that the distribution of $X$ is a convolution of those of $Y$ and $Z$, and therefore $\text{ridge}(f)$ slightly deviates from $\mathcal{C}$. We use the Hausdorff distance (see \eqref{hausdorffdef} below for the definition) between the algorithm outputs and the true ridge $\text{ridge}(f)$ as the error in the estimation. We compare the performance of the our algorithms and SCMS using this error based on 200 random samples from each of the two models. The bandwidth $h$ used in the kernel estimates is determined as $(h_{\text{opt},1} + h_{\text{opt},2})/2$, where $(h_{\text{opt},1} , h_{\text{opt},2})$ is the optimal bandwidth vector selected by optimizing the density Hessian estimation using the plug-in approach. All the three algorithms use the same starting points for each model, which are grid points near the ridges. The results are shown in Figure~\ref{fig:simul} and Table~\ref{tab:simul}, where we see that in this example all the three algorithms are able to estimate the true ridges well and their performances are comparable. In Appendix \ref{SCMS_fail} we provide an example for which SCMS fails to detect a part of a ridge while our algorithms are able to capture it.

\begin{figure}[h]
\centering
\includegraphics[scale=0.27]{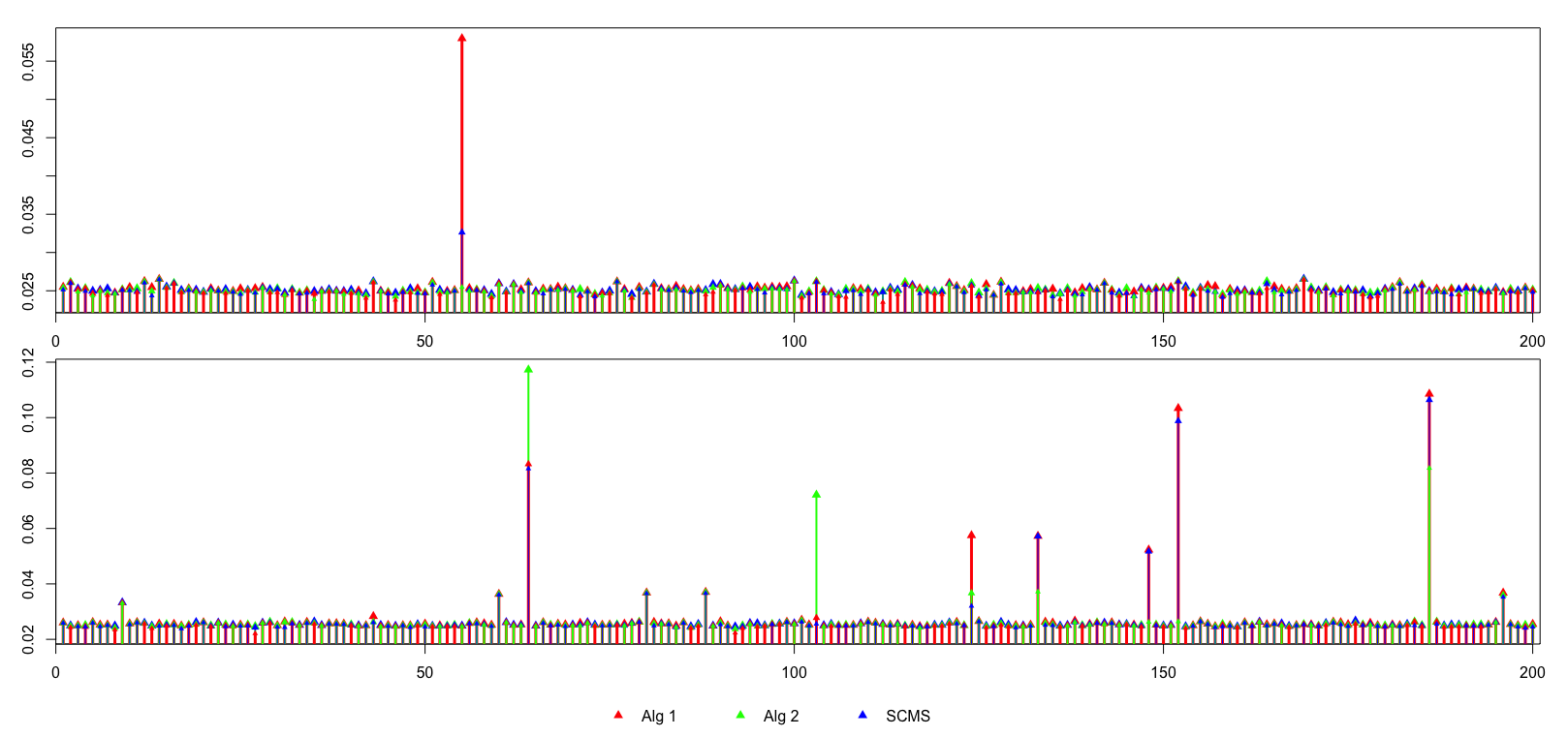}
\caption{Plots of errors for Model 1 (upper panel) and Model 2 (lower panel), each using 200 random samples, where the heights of the red, green and blue triangles represent the errors in the estimation for Algorithm 1, Algorithm 2 and SCMS, respectively.} 
\label{fig:simul}
\end{figure}

\begin{table}[H]
    \centering
    \begin{tabular}{c|ccc}
         & Alg 1 & Alg 2 & SCMS\\\hline
       Model 1  & 0.0253 (0.0024) & 0.0251 (0.0005) & 0.0251 (0.0007)\\\hline
       Model 2  & 0.0271 (0.0098) & 0.0266 (0.0085) & 0.0269 (0.0092)\\
    \end{tabular}
    \caption{This table contains the means and standard deviations (in parentheses) of the errors shown in Figure~\ref{fig:simul}.}
    \label{tab:simul}
\end{table}

\subsection{Real Data Application}

We apply our algorithms to a data set of active and extinct volcanoes in Japan available at \url{https://en.wikipedia.org/wiki/List_of_volcanoes_in_Japan}. The locations of these volcanoes exhibit a clear filamentary structure with three major branches sharing an intersection. The results using SCMS and our algorithms are shown in Figure \ref{fig:volcanoes}. We used the same bandwidth for all the three algorithms based on an optimal selection for the second derivatives of the kernel density estimation.  Using all the sample points as starting points, the outputs of the algorithms are shown in the three left panels. It can be seen that all the three algorithms can capture the three major branches in the data, however, a careful examination reveals that the output of the SCMS algorithm seems to have big gap near the intersection of the three branches. To further investigate this issue, we ran each algorithm with a new set of starting values while keeping all the tuning parameters the same. The results are shown in the three right panels. The new starting points are constructed as follows: For each of the $n$ outputs of an algorithm, connect each of the 20-nearest neighbors among the original data points to the output point by a line segment. On each of these 20 line segments choose 10 equidistant points. The resulting $200*n$ points are the new starting values of the respective algorithm. The idea underlying this construction is to find starting values that form a dense neighborhood of the true ridge lines. We observe that, although these start points fill the gap near the intersection well, the detected branches of SCMS algorithm are still clearly separated, while the branches are connected better in the outputs of our algorithms. Also see Section~\ref{SCMS_fail} for similar results in simulations. This is consistent with our arguments that SCMS algorithm may miss some parts of the ridges. 

\begin{figure}[!t]
\begin{subfigure}{0.5\textwidth}
\includegraphics[scale = 0.23]{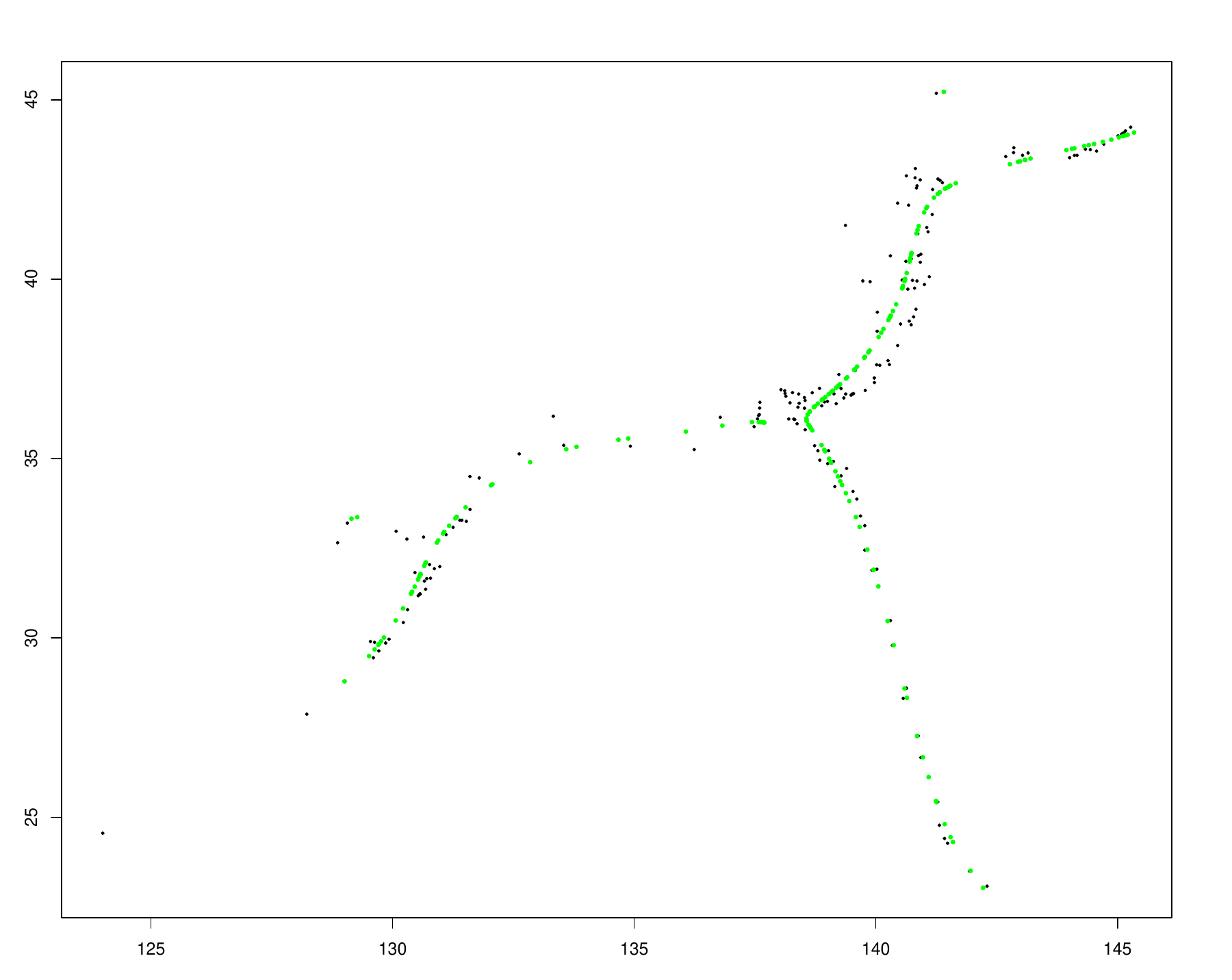}
\caption{SCMS: data as starting values} 
\end{subfigure}
\begin{subfigure}{0.5\textwidth}
\includegraphics[scale = 0.23]{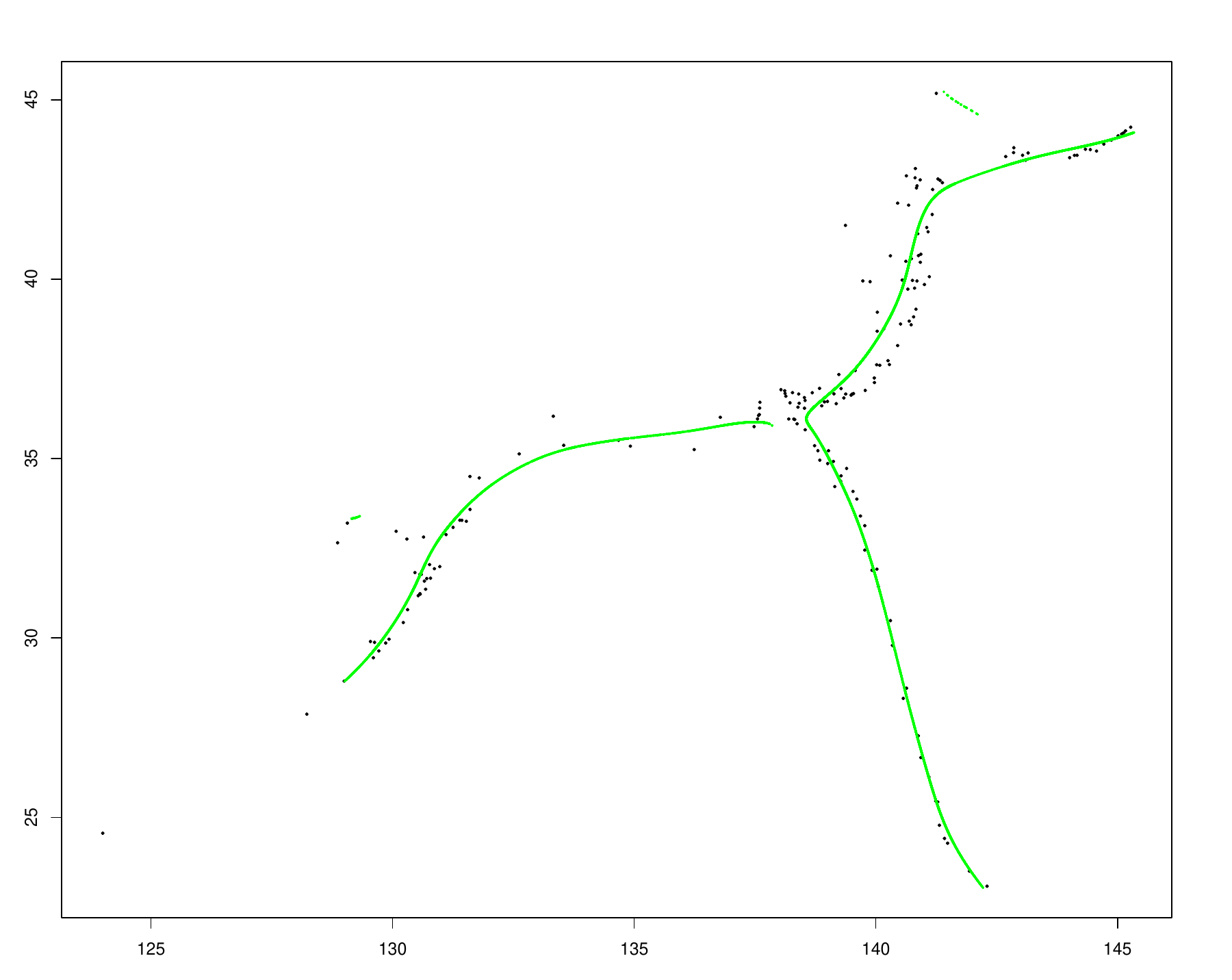}
\caption{SCMS: new starting values} 
\end{subfigure}
\begin{subfigure}{0.5\textwidth}
\includegraphics[scale = 0.23]{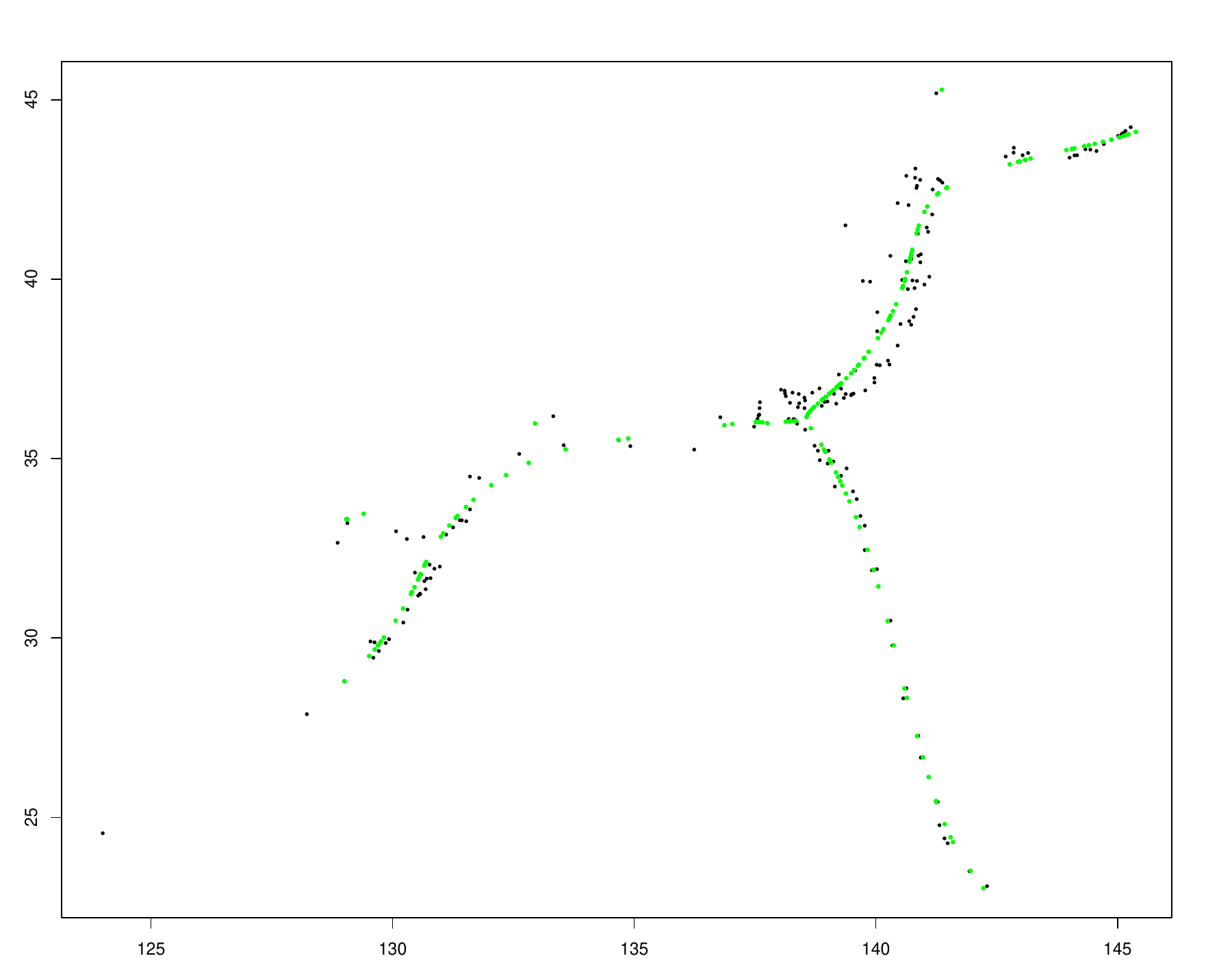}
\caption{Algorithm 1: data as starting values} 
\end{subfigure}
\begin{subfigure}{0.5\textwidth}
\includegraphics[scale = 0.23]{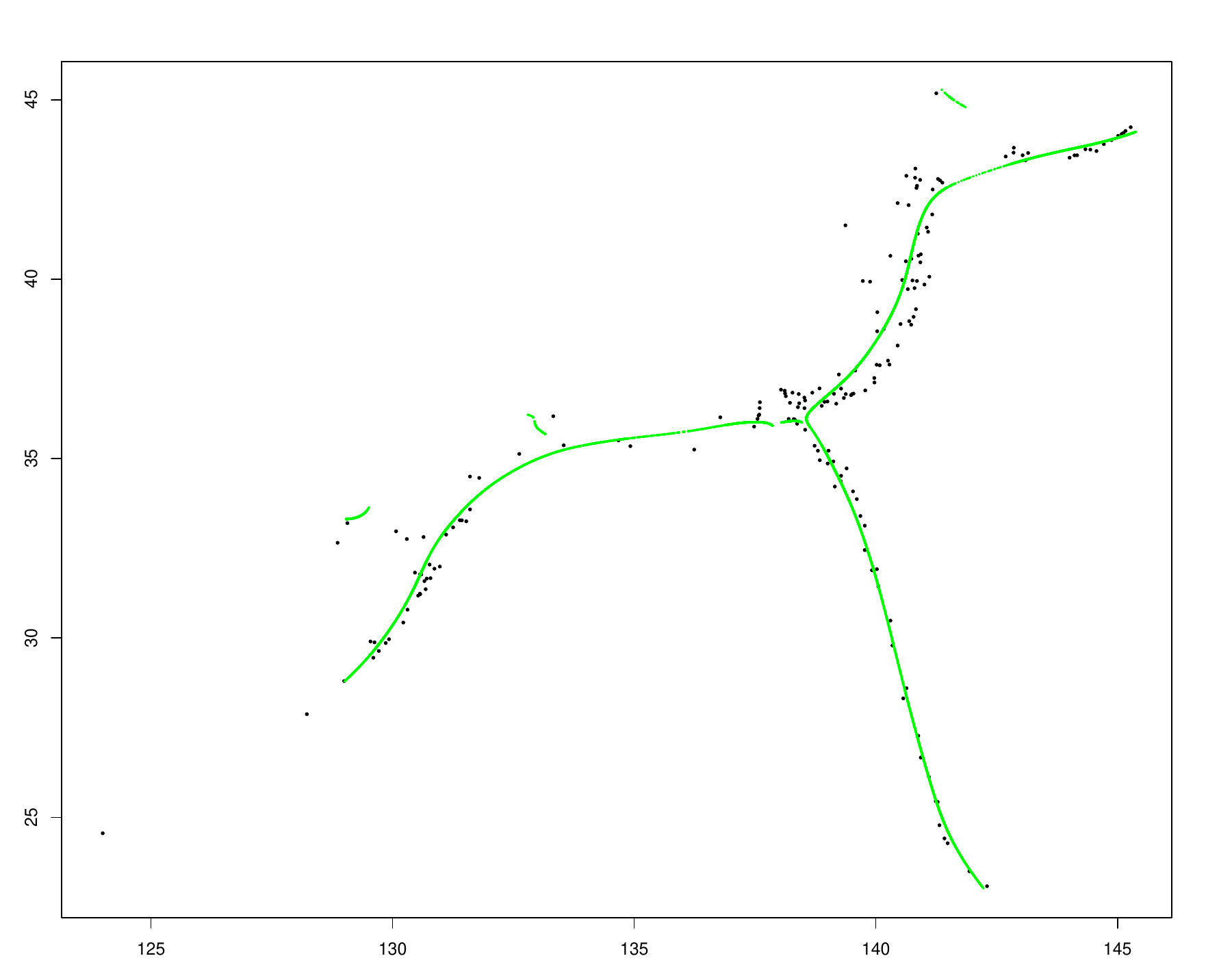}
\caption{Algorithm 1: new starting values} 
\end{subfigure} 
\begin{subfigure}{0.5\textwidth}
\includegraphics[scale = 0.23]{volcano_Alg1_original}
\caption{Algorithm 2: data as starting values} 
\end{subfigure}
\begin{subfigure}{0.5\textwidth}
\includegraphics[scale = 0.23]{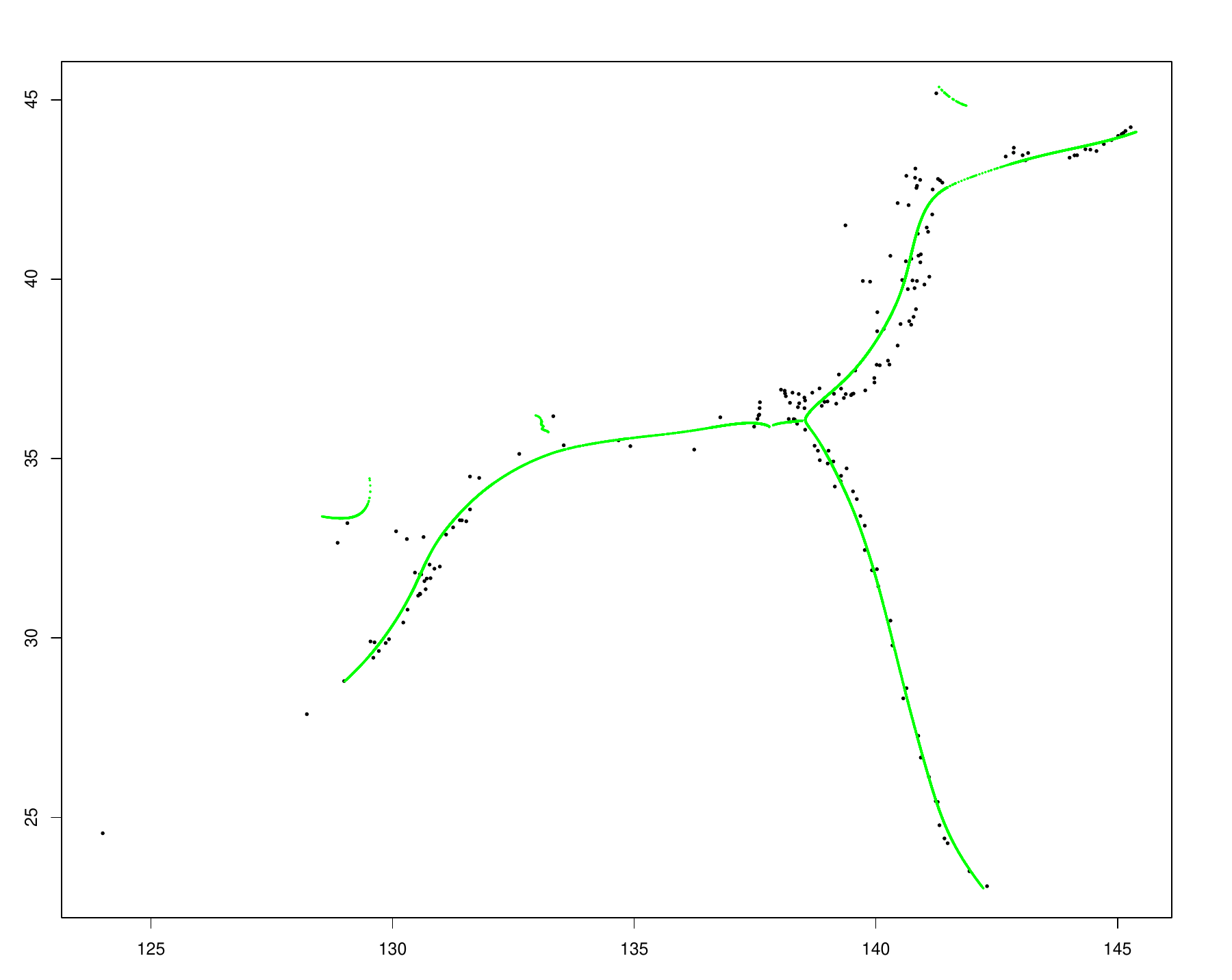}
\caption{Algorithm 2: new starting values} 
\end{subfigure}
\caption{
Outputs of the three algorithms using the original sample as initialization (left panels) and new initilization (right panels). The black dots are the locations of the volcanoes in Japan, and the green dots are the estimated ridge points. 
}
\label{fig:volcanoes}
\end{figure}

\section{Main results}
\label{mainresults}
\subsection{Assumptions and some technical implications}
\label{assumptions}
Before we state our assumptions, we first introduce some notation. For $\alpha=(\alpha_1,\cdots,\alpha_d)^T\in\mathbb{N}^d$, let $|\alpha| = \alpha_1+\cdots+\alpha_d$. For a function $g:\mathbb{R}^d\rightarrow\mathbb{R}$ with partial derivatives of order $|\alpha|$, define 
\begin{align*}
\partial^{(\alpha)} g(x) = \frac{\partial^{|\alpha|} }{\partial^{\alpha_1}x_1\cdots \partial^{\alpha_d}x_d} g(x),\; x\in\mathbb{R}^d.
\end{align*}

For any subset $\mathcal{A}\subset\mathbb{R}^d$, and $x\in\mathbb{R}^d$, let $d(x,\mathcal{A})=\inf_{y\in\mathcal{A}}\|x-y\|$. For another subset $\mathcal{B}\subset\mathbb{R}^d$, the Hausdorff distance between $\mathcal{A}$ and $\mathcal{B}$ is defined as
\begin{align}
\label{hausdorffdef}
    d_H(\mathcal{A},\mathcal{B}) = \max\{\sup_{x\in\mathcal{A}}d(x,\mathcal{B}),\;\sup_{x\in\mathcal{B}}d(x,\mathcal{A})\}.
\end{align}
Let $\partial\mathcal{A}$ be the boundary of $\mathcal{A}$. For $\delta > 0$ we define $\mathcal{A}^\delta = \bigcup_{x \in \mathcal{A}}\mathcal{B}(x,\delta),$ where $\mathcal{B}(x,\delta)$ denotes the (open) ball with midpoint $x$ and radius $\delta.$ For any $c_1,c_2\in\mathbb{R}$, let $c_1\vee c_2=\max(c_1,c_2)$ and $c_1\wedge c_2=\min(c_1,c_2)$. Finally, we will use $\| A\|_F$ to denote Frobenius-norm of a matrix $A$.

Recall that $\xix{\wh \eta}$  and $ \xix{\wh \eta_\tau}$ denote the projections of the gradients of $\wh \eta(x)$ and $\wh \eta_\tau(x)$, respectively, onto the subspace spanned by the trailing $(d-k)$ eigenvectors of $\nabla^2\wh \eta$ and $\nabla^2\wh\eta_\tau$, respectively. Similarly, we define the corresponding population quantities $\xix{f}$ and $\xix{\eta}$ (see Section~\ref{notation}). With slight abuse of notation, we use $\text{Ridge}(g)$ to denote the ridge of any twice differentiable function $g$ restricted to $[0,1]^d.$

In the formulations of our theoretical results, we will use the following assumptions. Let $m$ be a positive integer (where our results will require $m\geq4$).\\

{\bf (A1)$_{f,m}$} $f > 0$ is a density on $\mathbb{R}^d$ such that, for some $\epsilon_0 > 0$,  $[-\varepsilon_0,1+\varepsilon_0]^{d}$ is contained in the support of $f$; the partial derivatives of $f$ up to order $m$ exist and are bounded and continuous.\\

{\bf (A2)$_f$} There exist $\beta, \delta > 0$ such that:  $|\lamx{k+1}|  > \beta$ for all $x \in \{t\in[0,1]^d: \eta(t) = 0\}^\delta$, and $\lamx{k} - \lamx{k+1} < \beta,$ and for all $x \in [0,1]^d$. Furthermore, $\text{Ridge}(f)^\delta \subset [0,1]^d$. \\

{\bf (A3)$_f$} For all $x \in \text{Ridge}(f),$ $\nabla \xix{f} \in \R^{d \times d}$ has rank $d-k.$\\ 

{\bf (K)$_m$} The kernel function $K:\R^d \to [0,\infty)$ is spherically symmetric and integrates to 1. All the partial derivatives up to order $m$ exist and are continuous and bounded. Moreover, for any $\alpha\in\mathbb{N}^d$ with $|\alpha|\leq m$, the class of functions
\begin{align*}
    \Big\{ x\mapsto \partial^{(\alpha)}K\Big(\frac{x-y}{h}\Big): \, y\in \mathbb{R}^d, h > 0\Big\}
\end{align*}
is a VC-class \cite[see][]{vaart1996weak}. Also assume $\int_{\mathbb{R}^d} K(x)\|x\|^2dx<\infty.$\\

{\bf (L)} The kernel function $L:\R^d \to [0,\infty)$ is spherically symmetric with bounded support and integrates to 1. The partial derivatives of $L$ up to order four exist and are continuous.\\

{\em Discussion of the assumptions:} (i) Assumption {\bf (A1)$_{f,m}$} is made to avoid boundary issues of kernel density estimation on $[0,1]^d$. The unit cube can of course be replaced by any compact set in $\R^d$. (ii) Assumption {\bf (A2)$_f$} is typical in the literature of ridge estimation \cite[see, e.g. Assumption A1 in][]{genovese2014nonparametric}. It avoids spurious ridge points under small perturbation. (iii) Assumption {\bf (A3)$_f$} implies that $\text{Ridge}(f)$ is a manifold (without boundary). E.g., see Theorem 5.12 in \cite{lee2013smooth}. In particular, this precludes the existence of intersections of the ridge. For ridges with boundary or intersections, all our results except for Theorem~\ref{discretecomp}(iii) and Corollary~\ref{perturbeddiscrete}(iii) still apply to the part on the ridges strictly bounded away from the boundary and intersections. (v) The VC-class assumption in {\bf (K)}$_m$ holds if $K$ is of the form $K(x) = \phi(p(x))$ with $\phi:\R \to \R_{\geq0}$ of bounded variation, and $p(x)$ a polynomial \cite[see][]{nolan1987u}.  In particular this is the case for the Gaussian kernel. (vi) Under the above assumptions, $\eta(x)$ is twice differentiable and the second derivatives are Lipschitz continuous.  In particular, the Hessian $\nabla^2 \eta(x)$ is well-defined, and we have the following properties:

\begin{lemma}\label{misc1} Under {\bf (A1)$_{f,4}$},  {\bf (A2)$_f$}, and {\bf (A3)$_f$}, $\text{Ridge}(f)$ is a compact set; we have $\lamx[\eta]{1} = \cdots = \lamx[\eta]{k}=0$ for all $x \in \text{Ridge}(f)$, and there exist positive constants $A, \alpha$ and $\delta^\prime \le \delta$, such that for all $x \in \text{Ridge}(f)^{\delta^\prime},$ 
\begin{align}\label{eigenrange}
 - \alpha \ge \lamx[\eta]{k+1}\cdots\geq\lamx[\eta]{d} \ge - A.
\end{align}
Moreover, the columns of $V^\eta_\perp(x)$ span the normal space of $\text{Ridge}(f)$ at $x\in\text{Ridge}(f)$.
\end{lemma}

Recall the definition of $\Seps{\eta,f}$ given in (\ref{Sepsilon}) and notice that $\Sep{\eta,f}{0} = \text{Ridge}(f)$. The following lemma states that $\Seps{\eta,f}$ is in a neighborhood $\text{Ridge}(f)$ of radius in the order of $\sqrt{\epsilon}$ when $\epsilon\to0$.  
\begin{lemma}
\label{ridgetube}
Assuming {\bf (A1)$_{f,4}$},  {\bf (A2)$_f$}, and {\bf (A3)$_f$}, there exist constants  $L_1,L_2,\epsilon_0 > 0$ such that for all $0 < \epsilon \le \epsilon_0$,
\begin{align}\label{Lipschitz}
    L_1\sqrt{\epsilon} \leq d_H(\partial \Sep{\eta,f}{\epsilon},\text{Ridge}(f)) \leq L_2\sqrt{\epsilon},
\end{align}
where $d_H$ is the Hausdorff distance. 
\end{lemma}

Since $\nabla \eta(x) = {\bf 0}$ for $x\in\text{Ridge}(f)$, Lemma~\ref{misc1} implies that $\text{Ridge}(f)\subset \text{Ridge}(\eta)$. The following lemma states that in a neighborhood of the ridge of $f$, the ridge of the ridgeness function $\eta$ equals the ridge of $f$. Recall that $\eta(x) =  -\frac{1}{2}\|\Vbot{f}^\top\nabla f(x)\|^2 = -\frac{1}{2}\|\xix{f}\|^2$.

\begin{lemma}\label{ridge-ridgeness}
Assuming {\bf (A1)$_{f,4}$},  {\bf (A2)$_f$}, and {\bf (A3)$_f$}, there exists an $\epsilon_0 > 0$ such that 
$$\text{\rm Ridge}(f) = \text{\rm Ridge}(\eta) \cap \Seps{\eta,f}\qquad \forall\;0 < \epsilon \le \epsilon_0.$$ 
\end{lemma}

\subsection{Convergence results}\label{mainres}

Brief outline of this section: As the algorithms are targeting Ridge$(\wh f)$, while the theoretical target is Ridge$(f),$ we first control the distance between these two sets (see Theorem~\ref{ridgeHausRate}). Then, in Theorem~\ref{main1}, we consider the continuous version of the algorithms, where the paths traced by the algorithm are replaced by the integral curves generated by our ridgeness functions. Finally, we consider the discrete version (see Theorem~\ref{main2}), i.e. the actual algorithms, and control the distance between the limit points of the algorithms and Ridge$(\wh f)$.

\begin{theorem}\label{ridgeHausRate}
Assume {\bf (A1)$_{f,4}$},  {\bf (A2)$_f$}, {\bf (A3)$_f$}, {\bf (K)$_4$}, $h = o(1),$ and $\frac{\log n}{nh^{d+8}}=o(1)$ as $n\rightarrow\infty$. Then, for any $B>0$ and $n$ large enough, there exists a constant $C>0$ such that with probability at least $1-n^{-B}$:
\begin{align}
    d_H(\text{\rm Ridge}(f),\text{\rm Ridge}(\wh f\,)) \le C \Big(\Big(\frac{\log n}{nh^{d+4}}\Big)^{1/2} + h^2\Big).
\end{align}
\end{theorem}

\begin{remark}
\cite{genovese2014nonparametric} also gives the same rate of convergence for ridge estimation using the Hausdorff distance. However, their assumptions and methods of proof are different from ours. In particular, they require that $\|[I_d-\Proj{f}]\nabla f(x)\| f_{\max,3}(x)\leq \beta^2/(2d^3)$ for all $x\in\text{Ridge}(f)^\delta$ in their assumption (A2), where $f_{\max,3}(x)$ is the maximum of the absolute values of all the third partial derivatives of $f$ at $x$, and $\beta$ is essentially the same one as given in our {\bf (A2)$_f$}. Instead of this assumption, we use {\bf (A3)$_f$}, which is weaker \cite[see][Lemma 2]{chen2015asymptotic}.%
\end{remark}

\subsubsection{Continuous versions of the algorithms}\label{continuous}

{\bf Further notation, Part 1.}  Recall the definition of the vector field $\xi^g:\R^d\to\R^d$ given in Section~\ref{notation} for a twice differentiable function $g:\R^d\to\R$, and let $x_0 \in \R^d.$ Consider the system of ODEs 
$$\frac{dv(t)}{dt} = \xi^g(v(t)),\; t\in\mathbb{R},$$
where $v:\mathbb{R}\rightarrow\mathbb{R}^d$ with $v(0) = x_0$. We denote the flow generated by this system of ODEs driven by $\xi^g$ as 
$$\gamma^g:\mathbb{R}^d\times\mathbb{R} \rightarrow \mathbb{R}^d,$$ 
i.e. $\gamma^g(x,0)=x$ for $x \in \R^d$, $\gamma^g(\gamma^g(x,t),s)=\gamma^g(x,t+s)$ and $\frac{\partial}{\partial t}\gamma^g(x,t) = \xi^g(\gamma^g(x,t))$, for $s,t\in\mathbb{R}$.

With this notation applied to $g = \wh\eta$ and $g = \wh \eta_\tau$, we now have the following result:

\begin{theorem}\label{main1}
Assume {\bf (A1)$_{f,4}$},  {\bf (A2)$_f$}, {\bf (A3)$_f$}, {\bf (K)}$_4$, $h = o(1)$ and $\frac{\log n}{nh^{d+8}}=o(1)$ as $n\rightarrow\infty$. Then, for any $B>0$ and $n$ large enough, with probability at least $1-n^{-B}$:
\begin{enumerate}[(i)]
\item $\wh f$ satisfies assumptions {\bf (A1)}$_{\wh f,4\;}$, {\bf (A2)}$_{\wh f}$, and {\bf (A3)}$_{\wh f}$;
\item Continuous version of Algorithm 1: there exists $\epsilon_0>0$ such that, for all $0\leq\epsilon\leq\epsilon_0$,
\begin{itemize}
    \item[] $\text{\rm Ridge}(\wh f\,) = \{\lim_{t\rightarrow\infty} \gamxt[\wh\eta],x\in\partial\Seps{\wh \eta, \wh f}\}=\text{\rm Ridge}(\wh \eta\,)\cap \Seps{\wh \eta, \wh f}.$
\end{itemize}
\item Continuous version of Algorithm 2: under the additional assumptions {\bf (A1)$_{f,6}$}, {\bf (K)}$_6$, {\bf (L)} and $\frac{\log n}{nh^{d+12}}=o(1)$, and for $\epsilon > 0$ small enough, 
\begin{itemize}
    \item[a)] $\text{\rm Ridge}(\wh \eta_\tau)\cap\Seps{\wh \eta_\tau,\wh f}=\{\lim_{t\rightarrow\infty} \gamxt[\wh\eta_\tau],x\in\partial\Seps{\wh \eta_\tau,\wh f}\}$ for $\tau>0$ small enough;
    \item[b)] $d_H(\text{\rm Ridge}(\wh \eta_\tau)\cap\Seps{\wh \eta_\tau, \wh f},\text{\rm Ridge}(\wh f\,)) \le C_1\tau^2$ for a constant $C_1$ and $\tau$ small enough;
    \item[c)] $d_H(\text{\rm Ridge}(\wh \eta_\tau)\cap \Seps{\wh \eta_\tau,\wh f},\text{\rm Ridge}(f)) \le C_2((\frac{\log n}{nh^{d+4}})^{1/2} + h^2 + \tau^2)$ for a constant $C_2$ and $\tau$ small enough.
\end{itemize}
\end{enumerate}
\end{theorem}

\begin{remark}
Without the additional assumptions in part (iii) of the above theorem, we can still show that $ \sup_{x\in[0,1]^d} |\partial^{(\alpha)}\wh \eta(x) - \partial^{(\alpha)}\wh \eta_\tau(x)| = o_p(1)$ for all $|\alpha|=0,1,2,$ and further $d_H(\text{\rm Ridge}(\wh \eta_\tau)\cap\Seps{\wh \eta_\tau,\wh f},\text{\rm Ridge}(f)) =o_p(1)$ as $\tau\rightarrow0$.
\end{remark}

\subsubsection{Discrete approximation: the Euler method}\label{discrete}

Here we study discrete versions of the algorithms given above. To this end, we need a discrete version of the continuous flows defined above:

{\bf Further notation, Part 2.} For a twice differentiable function $g:\R^d\to \R$ and a constant $a > 0$ let the sequence $\gamx[a]{g}{\ell}$, $\ell=0,1,\cdots,$ be defined as  
\begin{align*}
\gamx[a]{g}{0}=x\qquad\text{and}\qquad \gamx[a]{g}{\ell+1} = \gamx[a]{g}{\ell} + a\, \xii{g}{\gamx[a]{g}{\ell} },\quad \ell = 1,2,\ldots
\end{align*}
This is a discrete approximation to  $\gamxt[g],$ $t\geq 0$ using Euler's method.

The following result applies this notation with $g = \wh \eta$ and $g = \wh \eta_\tau$. It says that these discretized approximations can be used to recover the corresponding ridges.
\begin{theorem}\label{main2}
Assume {\bf (A1)$_{f,6}$},  {\bf (A2)$_f$}, {\bf (A3)$_f$}, {\bf (K)}$_6$, {\bf (L)}, $h = o(1)$ and $\frac{\log n}{nh^{d+12}}=o(1)$ as $n\rightarrow\infty$. Then, for any $B>0$ and $n$ large enough, with probability at least $1-n^{-B}$:
\begin{enumerate}[(i)]
\item Algorithm 1: for $\epsilon\geq0$ small enough, $\lim_{\ell \to \infty}\gamx[a]{\wh\eta}{\ell}$ exists for all $x\in\partial\Seps{\wh \eta,\wh f}$, and 
\begin{align}\label{main2sub1}
d_H(\text{\rm Ridge}(\wh f\,),R_a(\wh f\,)) \le C_1 a^{1-\sigma_0-\mu}, 
\end{align}
for a constant $C_1$ and $\tau$ small enough, where $R_a(\wh f\,)=\{\lim_{\ell\rightarrow\infty} \gamx[a]{\wh\eta}{\ell}, x\in\partial\Seps{\wh \eta,\wh f}\}.$
\item Algorithm 2: for $\epsilon\geq0$ small enough, $\lim_{\ell\rightarrow\infty}\gamx[a]{\wh\eta_\tau}{\ell}$ exists for all $x\in\partial\Seps{\wh\eta_\tau,\wh f}$, and 
\begin{align}\label{main2sub2}
d_H(\text{\rm Ridge}(\wh f\,),R_{\tau,a}(\wh f\,))= C_2(a^{1-\sigma_0-\mu} + \tau^2), 
\end{align}
for a constant $C_2$ and $\tau$ and $a$ small enough, where $R_{\tau,a}(\wh f\,)=\{\lim_{\ell\rightarrow\infty} \gamx[a]{\wh\eta_\tau}{\ell},x\in\partial\Seps{\wh\eta_\tau,\wh f}\}.$
\end{enumerate}
In the above rates, $0<\sigma_0<1$ is given in (\ref{sigma0exp}), and $\mu>0$ is arbitrarily small. 
\end{theorem}
\begin{remark}
By using Corollary~\ref{perturbeddiscrete} given below, one can also show that, for $d = 1$, ${\rm Ridge}(\wh f\,) = R_a(\wh f\,)$ and $\text{\rm Ridge}(\wh \eta_\tau)\cap \Seps{\wh \tau,\wh f} = R_{\tau,a}(\wh f\,)$ with probability at least $1-n^{-B}$ for any $B>0$ and $n$ large enough. 
\end{remark}

\section{The mathematical framework for the ODE-based algorithms of ridge extraction}\label{maththeory}
This important section can be interpreted as providing population level versions of our main convergence results for the proposed algorithms presented above. Indeed, the algorithms can be interpreted as `perturbed versions' of corresponding population level versions. We will discuss the precise meaning of this in what follows, and we also indicate how this correspondence is being used to prove the convergence results for the algorithms. This section will also provide additional insights into why the algorithms proposed in this work do not suffer from the theoretical gaps of the SCMS algorithm - cf. Section~\ref{insights}.

The population level results are using theory for Ordinary Differential Equations (ODE). These ODEs provide the mathematical (population level) model for our algorithm. For the original mean shift algorithm, this analogy has been used in \cite{arias2016estimation} and \cite{arias2022clustering}. For the Subspace Constraint Mean Shift algorithm, see \cite{genovese2014nonparametric} and \cite{qiao2016theoretical}.

In the following $f$ denotes a generic positive function on $[0,1]^d$. While we apply the below results to densities (e.g. to our kernel density estimates), $f$ does not have to integrate to 1. 

\subsection{Some useful background knowledge of ODEs}
A reference for the following material is \cite{wigginsintroduction}. As above consider
\begin{align*}
    \frac{dx(t)}{dt} = U(x(t)),\; t\in\mathbb{R},
\end{align*}
where $x: \mathbb{R}\rightarrow\mathbb{R}^d$ with $x(0) = x_0 \in \R^d,$ and $U:\mathbb{R}^d\rightarrow \mathbb{R}^d$ is a vector field. Let $\pi:\mathbb{R}^d\times\mathbb{R} \rightarrow \mathbb{R}^d$ denote the corresponding flow. 

A compact set $S\subset\mathbb{R}^d$ is called a {\em positively invariant set} under the above vector field if for any $x_0\in S$ we have $\pi(x_0,t)\in S$ for all $t\geq 0$. We assume that the boundary of $S$ is a $\mathbf{C}^1$\ manifold. A point $\bar x$ is called an {\em equilibrium point}, or fixed point, if $U(\bar x)=0$.  Let $\mathcal{A}\subset S$ be the set of all equilibrium points in $S$. A continuously differentiable scalar-valued function $V$ defined on $S$ is called a {\em Lyapunov function} if it satisfies: $\frac{d V(x(t))}{dt}\leq 0$ for all $x(t)\in S$. We also define two sets 
\begin{align}
    &E = \left\{x\in S\, :\;\; \frac{d V(x(t))}{dt} = 0 \right\} \label{setE},\\
    &M = \bigcup_{x_0\in E}\big\{ \pi(x_0,t), t> 0\, :\; \pi(x_0,t)\in E \text{ for all } t>0 \big\}.\label{setM}
\end{align}
Note that $\mathcal{A}\subset M\subset E.$ We will use LaSalle's Invariance Principle \cite[see][Theorem 8.3.1]{wigginsintroduction}, which states:
\begin{theorem}
\label{LaSalle}
For all $x\in S$, $\pi(x,t)\rightarrow M$ as $t\rightarrow\infty$.
\end{theorem}
Later LaSalle's Theorem will be applied with $M = E = \text{Ridge}(f).$
\subsection{ODE theory for ridge extraction}
\label{ODEtheory}

Recall that
   $ \frac{\partial\gamxt[\eta]}{\partial t} = \xi^{\eta}( \gamma^{\eta}(x,t)), t\in\mathbb{R}$,
which is the mathematical model of Algorithms 1 and 2. There exists a positive $\epsilon>0$ such that $\Seps{\eta,f}$ is a positively invariant set corresponding to this flow. This can be seen as follows. Consider the derivative of $\eta(\gamxt[\eta])$ with respect to $t$:
\begin{align}\label{prop}
    \frac{\partial\eta(\gamxt[\eta]) }{\partial t} &= [\nabla \eta(\gamxt[\eta])]^\top V_{\bot}^{\eta}(\gamxt[\eta]) [V_{\bot}^{\eta}(\gamxt[\eta])]^\top \nabla \eta(\gamxt[\eta]) \nonumber\\
    &= \|\xi^{\eta}(\gamxt[\eta])\|^2 \geq 0.
\end{align}
In other words, $\eta$ is always non-decreasing along the trajectories of the flow. Indeed, notice that for $\epsilon > 0 $ small enough,
  $ \|\xi^{\eta}(\gamxt[\eta])\|  > 0$ 
for all $x\in \Seps{\eta,f}\backslash$Ridge($f$) and $t\in(0,\infty)$ based on Lemma~\ref{ridge-ridgeness}. It follows that $\Seps{\eta,f}$ is a positively invariant set as we have claimed. Note that Ridge($f$) is the set of all the equilibrium points in $\Seps{\eta,f}$. A natural choice for a Lyapunov function $V$ is 
\begin{align*}
    V(x) = - \eta(x).
\end{align*}
Since $\frac{\partial V(\gamxt[\eta]) }{\partial t} = - \langle \nabla\eta(\gamxt[\eta]), \xi^{\eta}(\gamxt[\eta])\rangle = -\|\xi^{\eta}(\gamxt[\eta])\|^2 \leq 0$, this indeed is a Lyapunov function. Notice further that this derivative is equal to zero  only when $x\in$ Ridge($f$). Therefore, for the sets $E$ and $M$ given in (\ref{setE}) and (\ref{setM}), we have $E=M=$ Ridge($f$).

\begin{theorem}\label{pathconvergence}
Assume that {\bf (A1)$_{f,4}$},  {\bf (A2)$_f$}, and {\bf (A3)$_f$} hold. 
There exists $\epsilon>0$ such that: 
\begin{enumerate}[(i)]
\item For $x \in (0,1)^d$, let a path $\gamma^\eta_x$ generated by $\gamxt[\eta]$ be given by the set $\gamma^\eta_x = \{\gamxt[\eta], t \in T_x\}$, where $T_x$ is the largest open interval containing $0$ such that $\gamxt[\eta] \in (0,1)^d$ for all $t \in T_x$, and let $\Gamma =\{\gamma^\eta_x, x \in (0,1)^{d}\}$ be the collection of all these paths. For each $x\in \Seps{\eta,f}\backslash$Ridge($f$), the path $\gamma_x^\eta$ is the unique path in $\Gamma$ passing through $x$. 
\item  For each $x\in \Seps{\eta,f}$ as the starting point, $\gamma^\eta(x,t)$ converges to a point on Ridge($f$), as $t\rightarrow+\infty$.
\item $\text{Ridge}(f) = \{\lim_{t\rightarrow\infty} \gamma^\eta(x,t): x\in \partial \Seps{\eta,f}\}$.
\end{enumerate}
\end{theorem}

\begin{remark}
Part (iii) of Theorem~\ref{pathconvergence} can be generalized as follows. Let $\mathcal{A}$ be a set such that for any $x\in \partial \Seps{\eta,f}$, there exists a point $y\in \mathcal{A}$ such that there is a finite $t_x\in(-\infty,\infty)$ such that $y= \gamx[]{\eta}{t_x}$. Then by (iii) in Theorem~\ref{pathconvergence}, we have $\text{Ridge}(f) = \{\lim_{t\rightarrow\infty} \gamx[]{\eta}{t}: x\in \mathcal{A}\}.$
\end{remark}

\subsubsection{Further insights into differences between SCMS and our algorithms}\label{insights} Using our notation introduced above, $\gamma^f(x,t)$ denotes the flow corresponding to $\frac{dx(t)}{dt} = \xi^f(x(t)), t\in\mathbb{R},$ which is the model for the SCMS algorithm. We can now see the major difference between using $\gamx[]{f}{t}$ and using the flow  $\gamx[]{\eta}{t}$ considered in our approach. For $\gamx[]{f}{t}$, we have
\begin{align}
\label{SCMSdensity}
    \frac{\partial f(\gamx[]{f}{t}) }{\partial t} & = [\nabla f(\gamx[]{f}{t})]^\top V^f_\perp(\gamx[]{f}{t}) [V_\perp(\gamx[]{f}{t})]^\top \nabla f(\gamx[]{f}{t}) \nonumber\\
    & = \|\xi^f(\gamx[]{f}{t}) \|^2 \geq 0.
\end{align}
In other words, it is the height of $f$ that increases along the path of $\gamx[]{f}{t}$ as $t$ increases. In contrast to that, it is the ridgeness $\eta(x)$ that increases along the path of $\gamx[]{\eta}{t}$. In general, while the height and ridgeness are closely related, they are two different quantities. This provides a different point of view for the SCMS algorithm, and also shows its difference to our approach. 

\subsection{Stability of the flows}\label{flowstability}
Now suppose we have a perturbed ridgeness function $\wt \eta$, which we assume is twice differentiable. We measure the perturbation by the following quantities.
\begin{align}
    &\delta_0 = \sup_{x\in[0,1]^d} |\eta(x)-\wt\eta(x)|,\\
    &\delta_1 = \sup_{x\in[0,1]^d} \|\nabla \eta(x)- \nabla \wt\eta(x)\|,\\
    &\delta_2 = \sup_{x\in[0,1]^d} \|\nabla^2 \eta(x)- \nabla^2 \wt\eta(x)\|_F.
\end{align}
Recall that, by definition, $\text{Ridge}(\wt \eta) = \{x\in[0,1]^d: \xix{\wt\eta}=0,\; \lamx[\wt\eta]{k+1}<0\}$ with $\xix{\tilde \eta}$ and $\lamx[k+1]{\tilde\eta}$ as defined in Section~\ref{notation}.

\begin{theorem}\label{stability}
Suppose that {\bf (A1)$_{f,4}$}, {\bf (A1)$_{\wt f,4}$}, {\bf (A2)$_f$}, and {\bf (A3)$_f$} hold. There exists $\epsilon_0>0$ such that for any $\epsilon\in(0,\epsilon_0]$ and for $\max(\delta_0,\delta_1,\delta_2)$ small enough we have the following:
\begin{itemize}
\item[(i)] For $x \in (0,1)^d$, let a path $\gamma^{\wt\eta}_x$ generated by $\gamxt[\wt\eta]$ be given by the set $\gamma^{\wt\eta}_x = \{\gamxt[\wt\eta], t \in T_x\}$, where $T_x$ is the largest open interval containing $0$ such that $\gamxt[\wt\eta] \in (0,1)^d$ for all $t \in T_x$, and let $\Gamma =\{\gamma^{\wt\eta}_x, x \in (0,1)^{d}\}$ be the collection of all these paths. For each $x\in \Seps{\wt\eta,f}\backslash$Ridge($\wt\eta$), the path $\gamma_x^{\wt\eta}$ is the unique path in $\Gamma$ passing through $x$. 
\item[(ii)] For each $x\in \Seps{\wt\eta,f}$, the path $\gamxt[\wt\eta]$ converges to a point in $\text{Ridge}(\wt\eta)\cap \Seps{\wt\eta}$, as $t\rightarrow+\infty$.
\item[(iii)] $\text{Ridge}(\wt\eta) \cap \Seps{\wt\eta,f} = \{\lim_{t\rightarrow\infty} \gamxt[\wt\eta]: x\in \partial \Seps{\wt\eta,f}\}.$%
\item[(iv)] There exists a constant $C>0$ such that $d_H(\text{Ridge}(\wt\eta) \cap \Seps{\wt\eta,f},\text{Ridge}(f))\leq C\max(\delta_1,\delta_2).$
\end{itemize}
\end{theorem}

\begin{remark}
This theorem is applied to $f, \wt\eta$ being $\wh f$ and $\wh\eta_\tau$, respectively, for which small enough values of $\delta_0, \delta_1,$ and $\delta_2$ can be found with high probability. See the proof of Theorem~\ref{main1}. 
\end{remark}

\subsection{Euler's method}

The following result about this discretization of the flow $\gamma^\eta$ is the main result of this section. Recall that  $\xix{\eta}=\Vbot{\eta}\Vbot{\eta}^\top\nabla\eta(x)$, and recall the definition of the discretized path $\gamx[a]{\eta}{\ell}$, $\ell=0,1,\cdots$ as defined in Section~\ref{discrete}.

\begin{theorem}\label{discretecomp}
Suppose that assumptions {\bf (A1)$_{f,6}$},  {\bf (A2)$_f$}, and {\bf (A3)$_f$} hold. Let $R_a(f)=\{\lim_{\ell\rightarrow\infty}\gamx[a]{\eta}{\ell},\; x\in \partial \Seps{\eta,f} \}$. There exist $\epsilon_0 >0, \delta_a>0$ such that when $0<a\leq \delta_a$ and $0 < \epsilon < \epsilon_0$, we have 
\begin{itemize}
\item[(i)]  $R_a(f) \subset \text{Ridge}(f),$\;\;\;and
\item[(ii)] there exists a constant $C>0$ such that $d_H (\text{Ridge}(f), R_a(f)) \leq C a^{1-\sigma_0}$,
where $0<\sigma_0<1$ is given in (\ref{sigma0exp}). 
\end{itemize}
Moreover,
\begin{itemize}
\item[(iii)] When $k=1$ (1-dimensional ridge), we have
\begin{align*}
\text{Ridge}(f) = R_a(f).
\end{align*}
\end{itemize}
\end{theorem}

\begin{remark} The last assertion of this theorem says that in the case of one-dimensional ridges we can recover the entire ridge of $f$ (as in the continuous case - see Theorem~\ref{pathconvergence}(iii)), even when using the discretized algorithm, provided the step-size is small enough. We conjecture that the result also holds true for multi-dimensional ridges.%
\end{remark}

Recall that we have analyzed the perturbed flow $\gamma^{\tilde\eta}$ and its convergence in Section~\ref{flowstability}. We now assume that $\wt \eta$ is generated by a density function $\wt f$, i.e.,
\begin{align}\label{TargetFun-mod}
\wt\eta(x) = -\frac{1}{2}\big\|\Vbot{\wt f})^\top\nabla \wt f(x)\big\|^2.
\end{align}

Using the same notation as defined in Section~\ref{discrete}, consider the sequence $\gamx[a]{\tilde\eta}{\ell}$, $\ell=0,1,\cdots$. 
Using Theorem~\ref{stability}, and following a very similar proof of Theorem~\ref{discretecomp}, we obtain the following discretization result of the perturbed flows and ridges.%

\begin{corollary}\label{perturbeddiscrete}
Suppose that assumptions {\bf (A1)$_{f,6}$}, {\bf (A1)$_{\wt f,6}$},  {\bf (A2)$_f$}, and {\bf (A3)$_f$} hold. Let $R_a(\wt f\,)=\{\lim_{\ell\rightarrow\infty}\gamx[a]{\wt\eta}{\ell},\; x\in \partial \Seps{\wt \eta,f}  \}$. When $\max(\delta_0,\delta_1,\delta_2)$ is small enough, there exist $\epsilon_0>0$ and $\delta_0>0$ such that for all $0<\epsilon\leq\epsilon_0,$ and $0<a\leq \delta_a$, we have that 
\begin{itemize}
\item[(i)]  $R_a(\wt f) \subset \text{Ridge}(\wt\eta\,)\cap  \Seps{\wt\eta,f},$\;\;\;and
\item[(ii)] there exists a constant $C>0$ such that $d_H (\text{Ridge}(\wt\eta\,)\cap \Seps{\wt\eta,f}, R_a(\wt f\,)) \leq C a^{1-\sigma_0-\mu}$, for an arbitrarily small $\mu>0$. 
\end{itemize}
Moreover,
\begin{itemize}
\item[(iii)] When $k=1$ (1-dimensional ridge), we have
\begin{align*}
\text{Ridge}(\wt\eta\,) \cap \Seps{\wt\eta,f} = R_a(\wt f\,).
\end{align*}
\end{itemize}
\end{corollary}

The proof of Corollary~\ref{perturbeddiscrete} is similar to that of Theorem~\ref{discretecomp}, and no details are presented.

\section{Discussion}
\subsection{Other variants}
Here we briefly discuss two other variants of our algorithms. 

The first variant is based on the logarithm transformation. The original mean shift algorithm proposed by \cite{fukunaga1975estimation} is a gradient ascent algorithm implicitly using the logarithm transformation of kernel density estimates, which is analyzed in \cite{arias2016estimation} and \cite{arias2022clustering}. More specifically, for any $x_0\in\mathbb{R^d}$, they consider the sequence $x_{j+1}=x_j + a\nabla \wh r(x_j)$, $j=0,1,2,\cdots,$ where $\wh r$ is an estimate of $\log f$. The limit of the sequence is a local mode of $\wh f,$ provided some regularity assumptions hold.

A similar idea can be applied to our algorithms as well. Let $g:[0,\infty)\rightarrow (0,\infty)$ be a (known) twice differentiable increasing positive function. In Algorithm 1, we can replace the derivatives (and the induced eigenvalues and eigenvectors) of $\wh\eta$ by those of $\log (g(\wh\eta))$. Note that $\wh\eta$ by definition is non-positive and this explains the need for a positive transformation function $g$ before applying the logarithm. More specifically, the update step (\ref{algorithm1update}) in Algorithm 1 can be replaced by 
\begin{align*}
y_i^{j+1} = y_i^j + a\xi^{\log g(\wh\eta)}(y_i^j).
\end{align*}
Algorithm 2 can be modified in a similar way. 

The second variant is based on modified objective functions in the optimization. For any $q>0$, and $x$ such that $\wh f(x) > 0$, write 
\begin{align}\label{TargetFun}
\wh s_q(x) = \frac{\wh\eta(x)}{[\wh f(x)]^q}.
\end{align}
We have that $\wh s_q \le 0$ and $\wh s_q = 0$ on $\text{Ridge}(\wh f\,)$, and finding minimizers of $\wh s_q$ is equivalent to that of $\wh\eta$. Thus we can replace the derivatives (and the induced eigenvalues and eigenvectors) of $\wh\eta$ by those of $\wh s_q$. Recall that our basic algorithms can possibly return non-global modes of the ridgeness function depending on the start points, and thus requires a post-process step (see Section~\ref{pruning}). The numerator in (\ref{TargetFun}) is a penalization for low density. The purpose underlying the introduction of this penalization is to sharpen the ridgeness function near the ridge, so as to potentially accelerate the algorithm and enlarge the set of starting points whose corresponding limit points are global maxima. %

The theoretical analyses of these two variants are not explicitly given in this paper, although they are straightforward extensions by following the same procedure as presented for Algorithms 1 and 2 above. 

\subsection{Mathematical models of SCMS algorithm and its convergence}

The analyses and simulations in this paper show that the set that SCMS converges to is not always exactly the ridge set under its original definition, although they are clearly related. It is an interesting open problem to find the mathematical definition of the set of the limit points of SCMS. Once this is discovered, we believe the convergence analyses established in this paper can be useful to show the convergence of SCMS in the sense of recovering the entire set. 

\appendix

\section{More details of algorithms} 
\label{appendix:a}

\subsection{Discussion of why SCMS algorithm might miss parts of ridges} 
\label{SCMS_fail}
Here we provide some more details and an example for the observation made in Section~\ref{SCMSTheory} that the SCMS algorithm might miss some parts of a ridge.

We consider the case $d=2.$ First we show that a ridge point does not necessarily have to be a local maximum of the integral curve traced by the SCMS algorithm. Such a point will thus not be identified as a ridge point (except in the trivial case where the starting point happens to be this ridge point).  For a given point $x_0$ near the ridge, consider an integral curve $x(t)$ defined as
\begin{align}\label{example}
\frac{dx(t)}{dt} = V_{\perp}(x(t)),\;\; x(0)=x_0,
\end{align}
where $V_{\perp}(x)$ is the second unit eigenvector of the Hessian of $f$ at $x$. We assume that the direction (sign) of $V_\perp(x)$ is determined such that it varies continuously with $x$. Note that $V_{\perp}$ is parallel to $V_{\perp}V_{\perp}^\top \nabla f = $ and so $x(t)$ has the same trajectory as the integral curve driven by $V_{\perp}V_{\perp}^\top \nabla f$. Using $V_\perp$ allows tracking integral curves both forward and backward. Indeed, the vector field $V_{\perp}V_{\perp}^\top \nabla f$ vanishes on the ridge, while $V_\perp$ always has a unit length. Suppose that there exists an interval $(a,b)$ such that $\{x(t): \;t\in(a,b)\}$ intersects with Ridge$(f)$. Then the first and second order derivatives of $f(x(t))$ with respect to $t$ are 
\begin{align}\label{DirDerivative0}
\frac{df(x(t))}{dt} = \big\langle \nabla f(x(t)),\; V_{\perp}(x(t)) \big\rangle
\end{align}
and
\begin{align}\label{DirDerivative}
\frac{d^2f(x(t))}{dt^2} &=\Big\langle \nabla \big\langle \nabla f(x(t)),\; V_{\perp}(x(t)) \big\rangle,\; V_{\perp}(x(t)) \Big\rangle \nonumber\\
& = \nabla f(x(t))\; \nabla V_{\perp}(x(t))\; V_{\perp}(x(t)) + \Big\langle \nabla^2f(x(t))\; V_{\perp}(x(t)), \; V_{\perp}(x(t)) \Big\rangle\nonumber\\
& = \nabla f(x(t))\; \nabla V_{\perp}(x(t))\; V_{\perp}(x(t)) + \lambda_2(x(t)).
\end{align}
If $x(t)$ is a ridge point, then $\frac{df(x(t))}{dt}=0$ in (\ref{DirDerivative0}) by Definition~\ref{FilaDef1}, and the second term $\lambda_2(x(t))$ in (\ref{DirDerivative}) is negative. 
In general, however, the right-hand side of (\ref{DirDerivative}) may be not negative. In other words, depending on the sign in (\ref{DirDerivative}), ridge points on the trajectory driven by $V_{\perp}$ (or equivalently, by $V_{\perp}V_{\perp}^\top\nabla f$) can be local maxima, local minima or even saddle points. An example is given below, showing that following the direction of $V_{\perp}V_{\perp}^\top\nabla f$, a part of the ridge will be missed if the starting points are not chosen exactly on that part.\\

{\bf Example}  
Consider the density function
\begin{align}\label{counterex}
f(u,v)=\frac{3}{8}(1-p)(1-u^2)v + \frac{1}{4}p,\quad u\in[-1,1],\quad v\in [0,2],
\end{align}
where $p\in(0,1)$. Using the ridge point definition, two curves are detected:  $\{u=0,v\in(0,2)\}$ and $\{(u,v): v=\frac{1-u^2}{\sqrt{2+2u^2}}, u\in(-1, 1)\}$. The intersection of the two curves is at $\big(0,\frac{1}{\sqrt{2}}\big)$. The plots of this function when $p=0$ and the ridge are given in Figures~\ref{FigExample1} and~\ref{FigExample2}. Note that here $p$ presents the proportion of a uniform background noise and does not affect the location of the ridge set in the model.

In Figures~\ref{scmsfig} and~\ref{newalgfig} we compare the ideas behind the SCMS algorithm and our new algorithms. It can be found that a piece of the ridge $S_{\text{missing}} := \{0\}\times (0,1/\sqrt{2})$ is failed to be detected using the idea of the SCMS algorithm in this example.

We also tested our new algorithms and SCMS on a random sample generated from the density in \eqref{counterex} with $p=0.3$. For all the three algorithms, we used the same sample of size 10000, the same bandwidth 0.3, and the same $25\times 25$ grid points are the starting points. To handle the boundary effect of the kernel estimation near the top boundary $[-1,1]\times\{2\}$, caused by the abrupt change in density beyond the boundary, we doubled the original sample size by adding a reflected data set across this boundary, and then only kept the ridge estimation below $u=1.9999$, just slightly lower than the boundary. The results are given in Figure~\ref{fig:counterex}. It is clear to see that the ridge is estimated well by our Algorithms 1 and 2. Notably, our algorithms can find a piece of ridge corresponding to $S_{\text{missing}}$ but SCMS fails to find it. In addition, the estimated ridge pieces by the SCMS algorithm near near $(0,1/\sqrt{2})$ are severely broken, while they are almost connected in the outputs of our algorithms. 

\setlength{\columnsep}{20pt}

\begin{multicols}{2}
\begin{center}
\includegraphics[height=6cm]{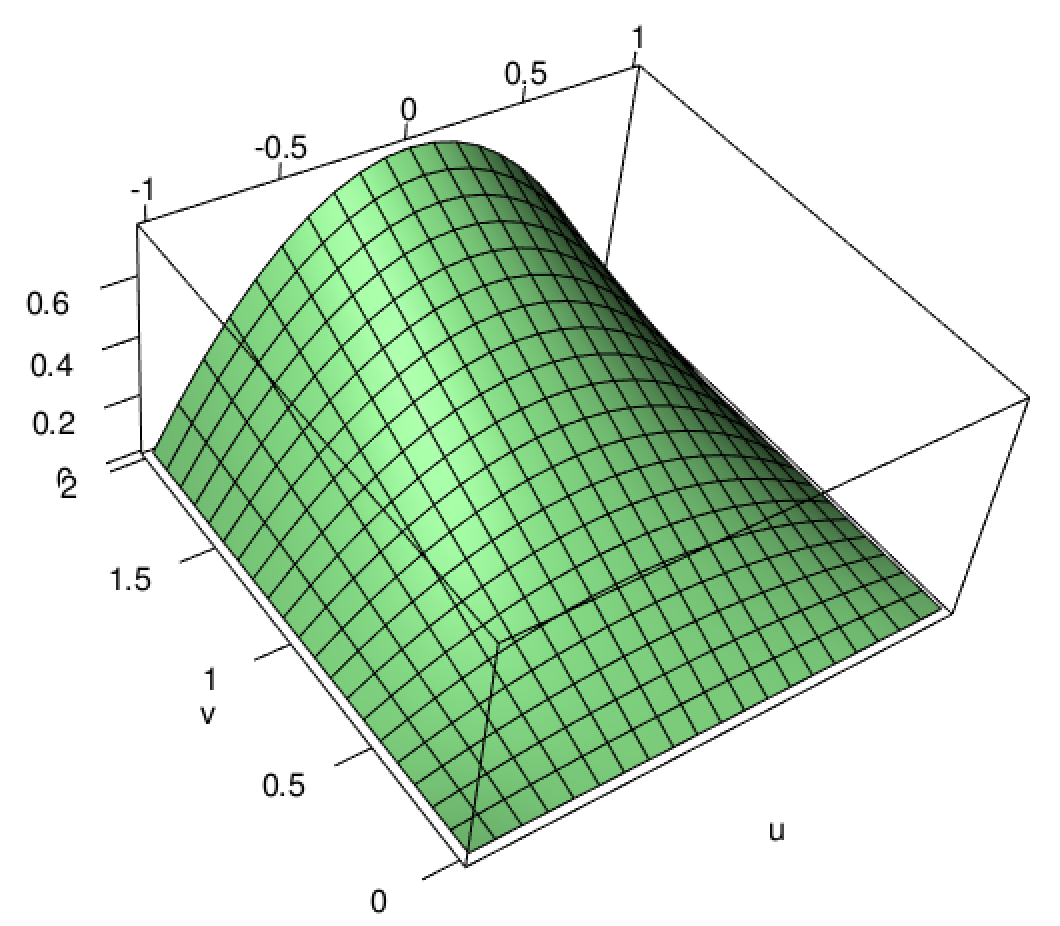}
\captionof{figure}{Surface plot of the function $f$ in~(\ref{counterex}).}\label{FigExample1}
\includegraphics[height=6cm]{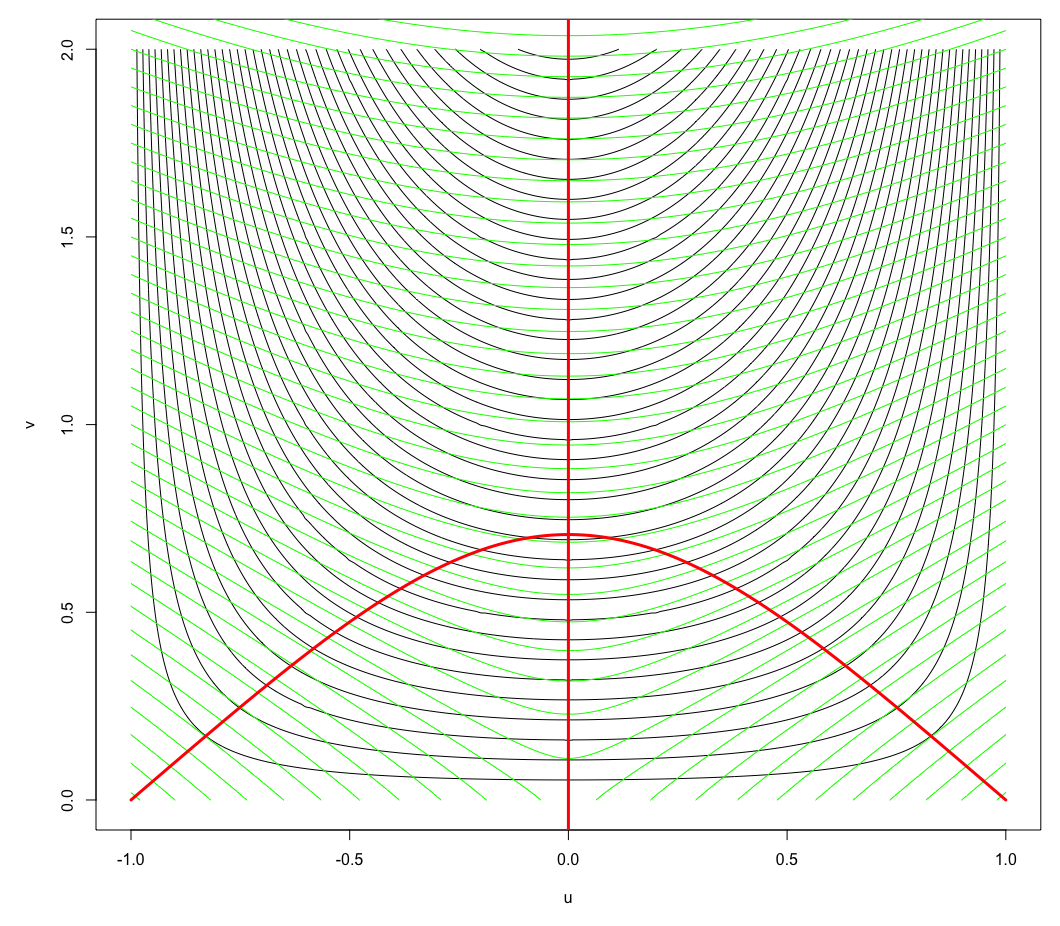}
\captionof{figure}{contour lines (black), ridge lines (red), and integral curves driven by $V_\perp$ (green).}\label{FigExample2}
\end{center}
\end{multicols}

\begin{multicols}{2}
\begin{center}
\includegraphics[width=6cm]{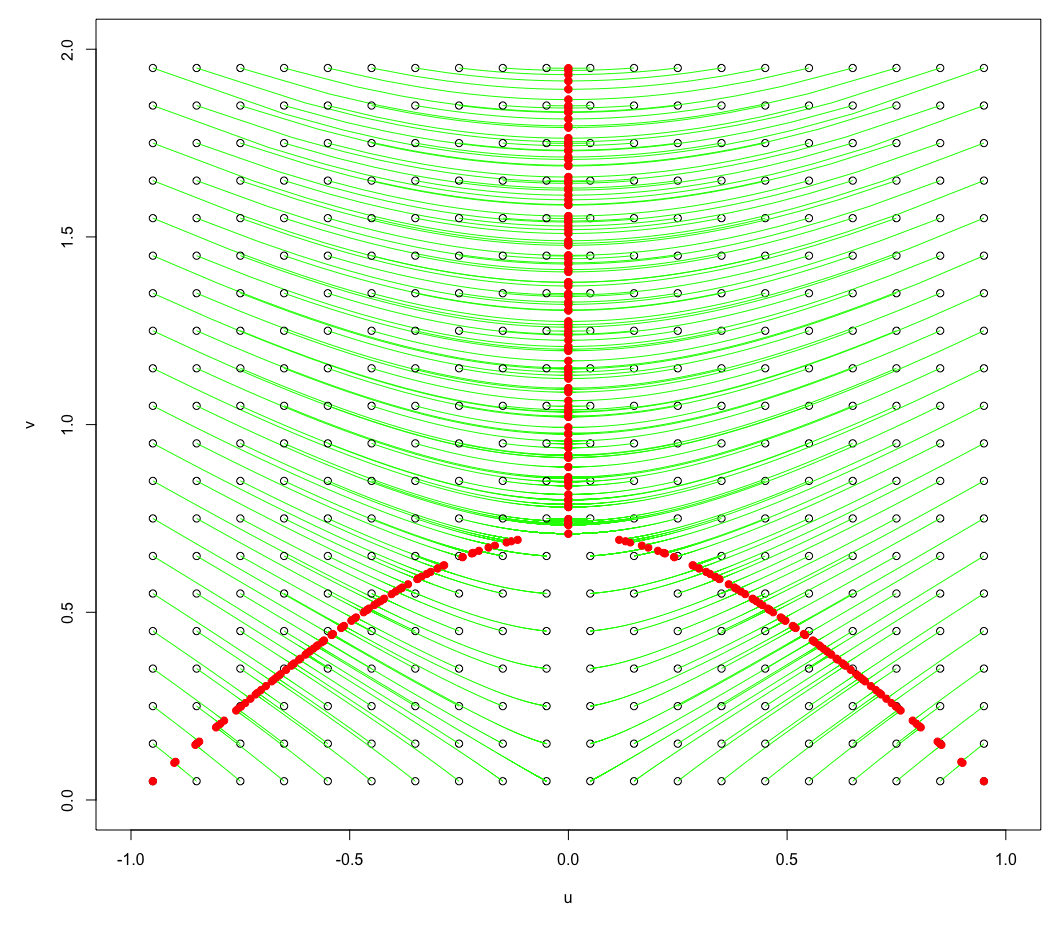}
\captionof{figure}{green curves are the trajectories of $\xi^f$, and the red dots are the limit points.}\label{scmsfig}
\includegraphics[width=6cm]{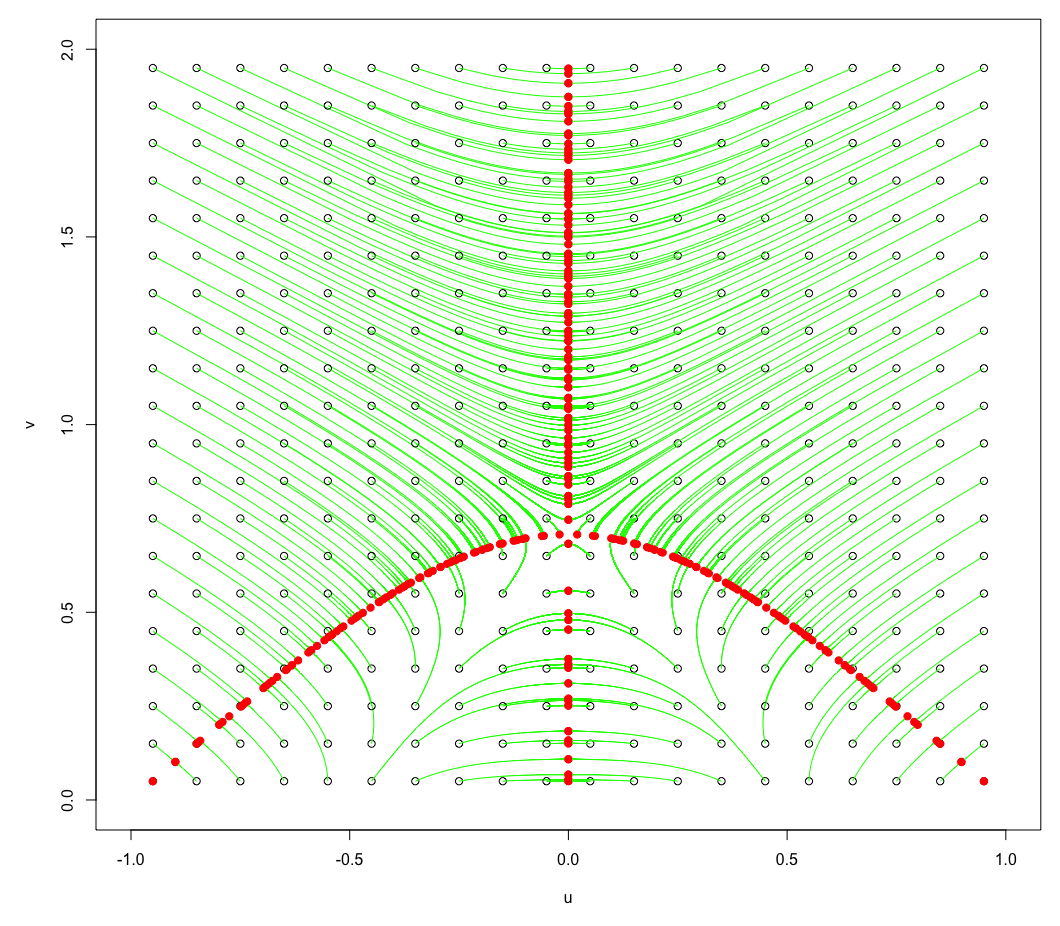}
\captionof{figure}{green curves are the trajectories of $\xi^\eta$, and the red dots are the limit points.}\label{newalgfig}
\end{center}
\end{multicols}
\setlength{\columnsep}{10pt}

\begin{figure}[h]
    \centering
    \includegraphics[scale=0.11]{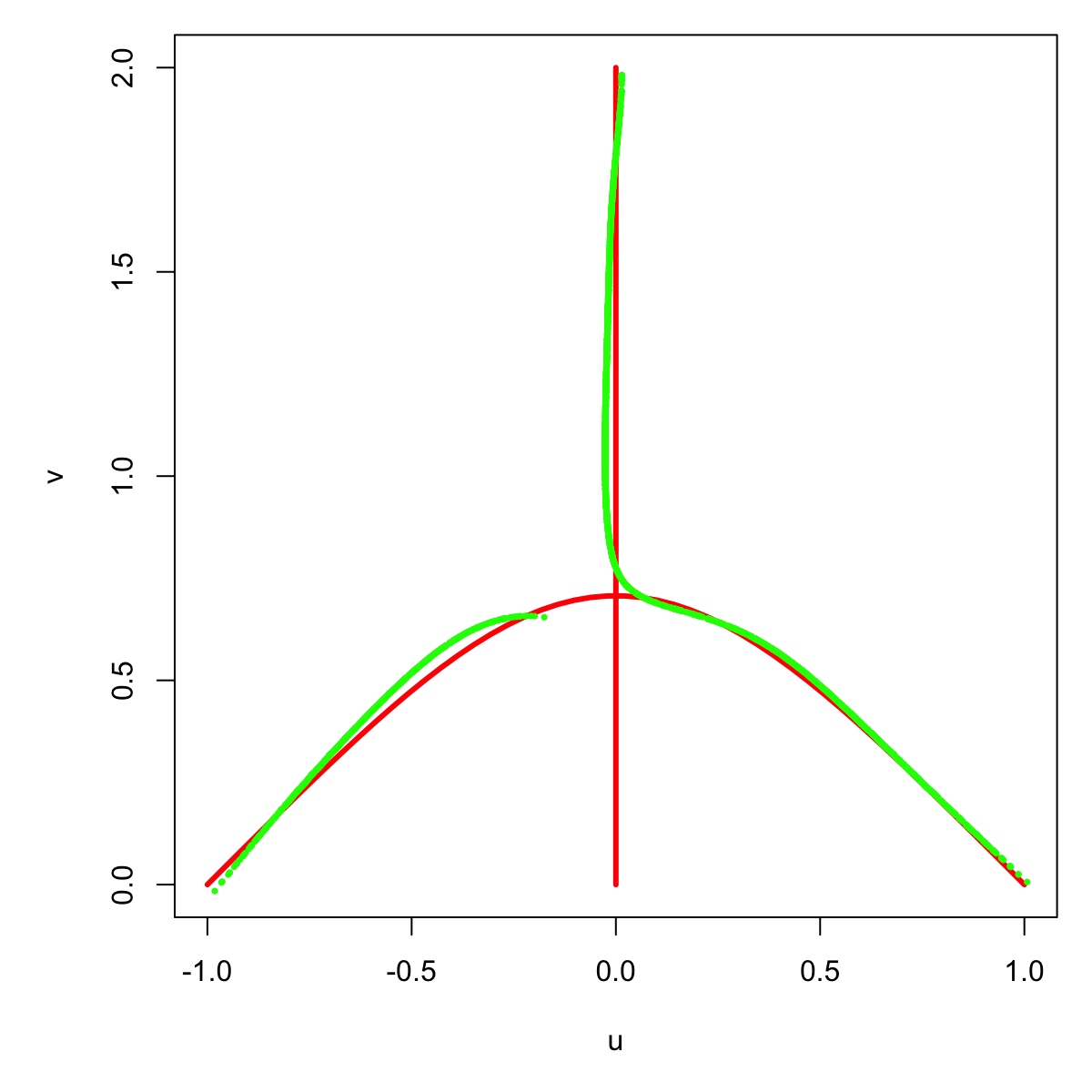}
    \includegraphics[scale=0.11]{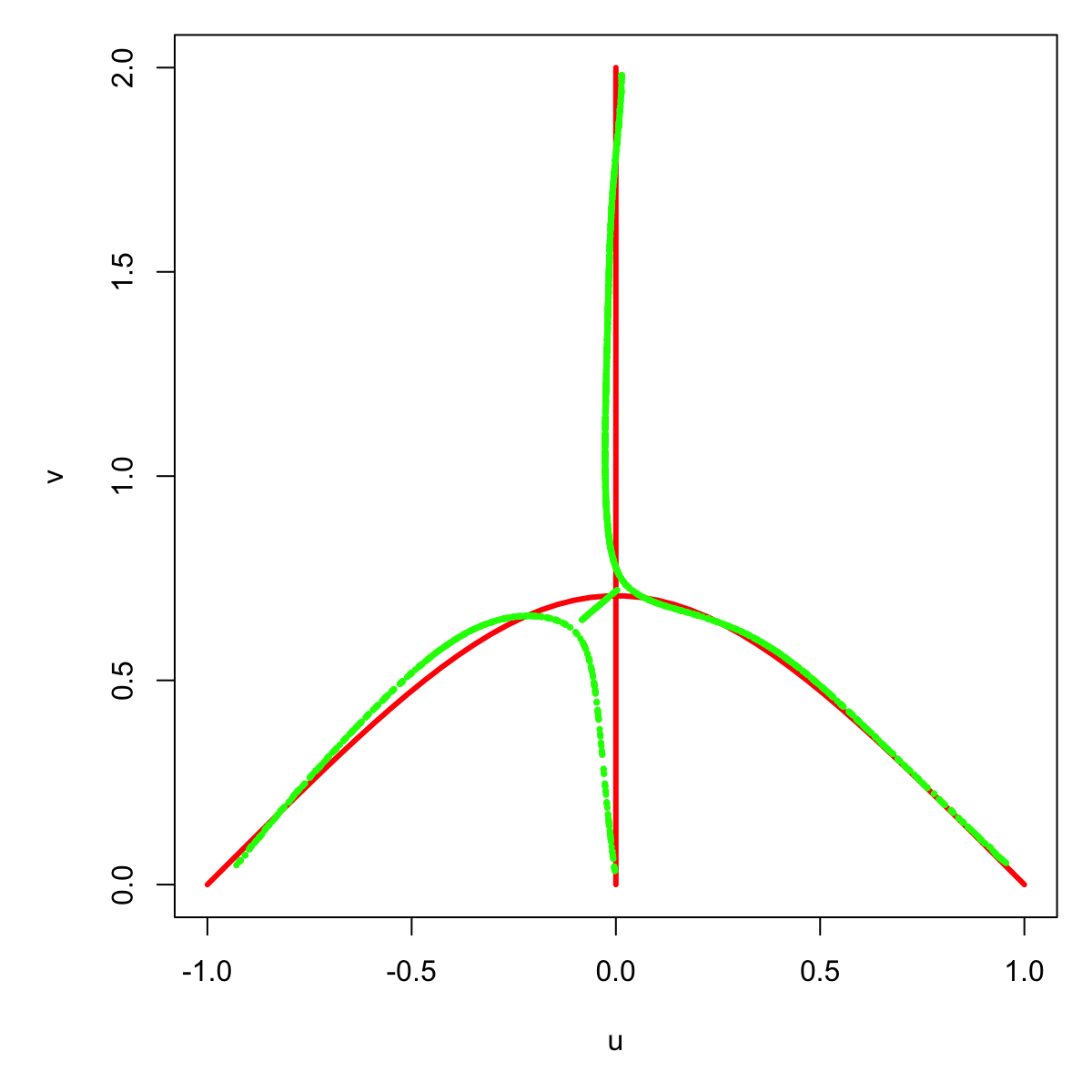}
    \includegraphics[scale=0.11]{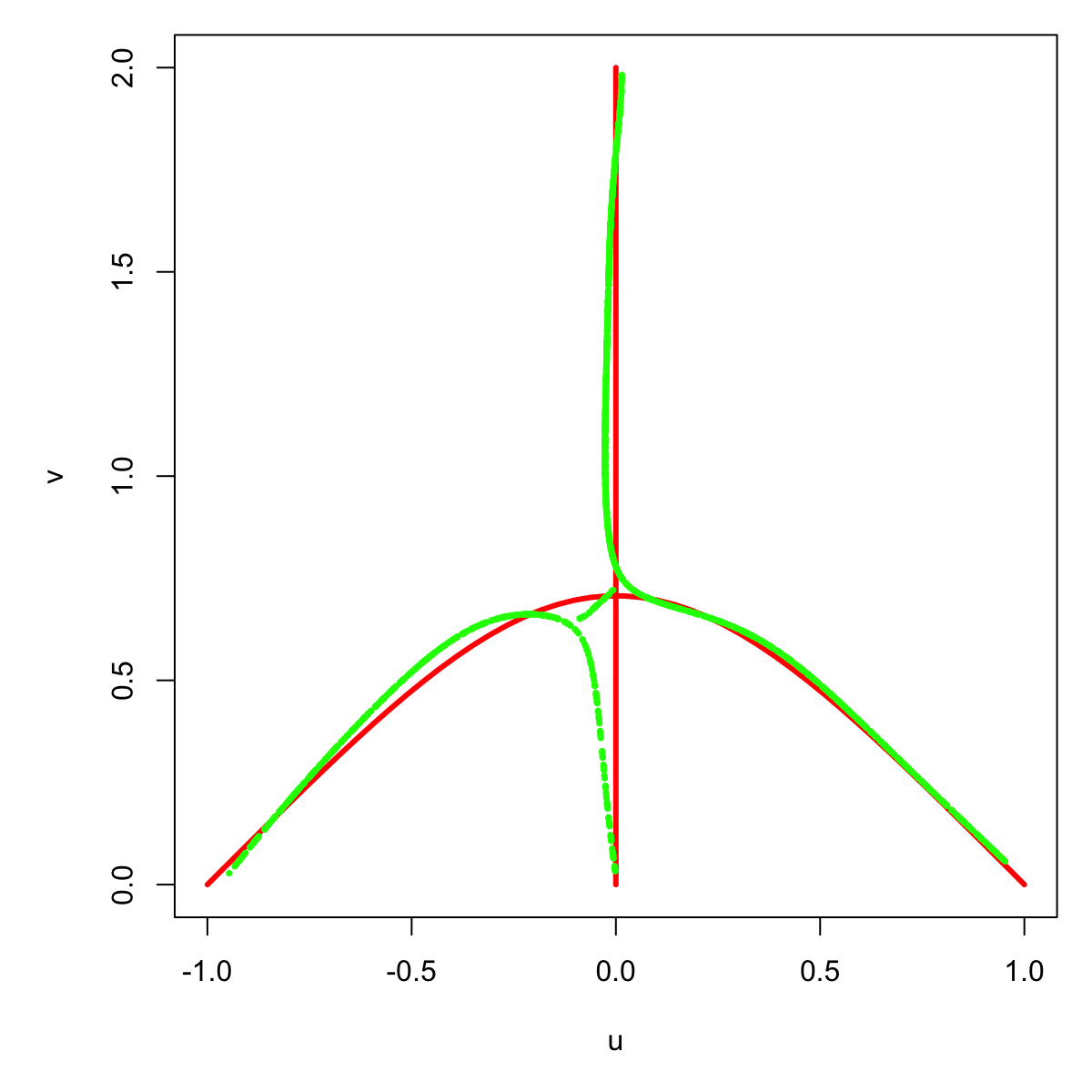}
    \caption{Plots from left to right show the outputs from the SCMS, our Algorithms 1 and 2. The red lines are the ridge model as given in Figure~\ref{newalgfig}, and the green dots are the outputs from the algorithms.}
        \label{fig:counterex}
    
\end{figure}

\subsection{Formulas needed for implementing Algorithm 1} 

In Algorithm 1 we need to compute $\nabla \xix{\wh f}$ and also the Hessian of $\wh \eta(x)$ (in order to find its eigenvectors). In the following we provide some formulas that are useful for the implementation. 

Let \textsc{vec} be the matrix vectorization operator such that $\textsc{vec}(A)$ stacks all the columns of a matrix $A$ into a vector.

As for $\nabla \xix{\wh f}$, we have the following by using the product rule for matrix calculus: 
\begin{align}\label{nablaxiexp}
\nabla \xix{\wh f} &= (\nabla \wh f(x)^\top \otimes \mathbf{I}_d) \nabla \Proj{\wh f} + \Proj{\wh f} \nabla^2 \wh f(x) \nonumber\\
& = (\nabla \wh f(x)^\top \otimes \mathbf{I}_d) \nabla \Proj{\wh f} + \Vbot{\wh f} \Lambda^{\wh f}_\bot(x) \Vbot{\wh f}^\top,
\end{align}
where $\Lambda^{\wh f}_\perp(x)=\text{diag}(\lam[\wh f]{r+1},\cdots,\lam[\wh f]{d})$, $\nabla \Proj{\wh f} = \frac{d\textsc{ vec}[\Proj{\wh f}]}{d x^\top} \in \R^{d^2 \times d}$, and $\otimes$ denotes Kronecker product. 

The Hessian $\nabla^2 \wh\eta(x)$ is given by
\begin{align}\label{Hessian}
    \nabla^2 \wh\eta(x) =  -(\mathbf{I}_d\otimes \xix{\wh f}^\top ) \nabla(\nabla \xix{\wh f})  - \nabla \xix{\wh f})^\top\nabla \xix{\wh f},
\end{align}
where $\nabla(\nabla \xix{\wh f})$ can be found by using the product rule:
\begin{align*}
    \nabla(\nabla \xix{\wh f}) = & [(\nabla \Proj{\wh f})^\top\otimes \mathbf{I}_d] \nabla (\nabla \wh f(x)^\top \otimes \mathbf{I}_d) \\
    & + (\mathbf{I}_d\otimes (\nabla \wh f(x)^\top \otimes \mathbf{I}_d))\nabla (\nabla \Proj{\wh f})\\
    & + (\nabla^2\wh f(x)\otimes\mathbf{I}_d)\nabla \Proj{\wh f}\\
    & + (\mathbf{I}_d\otimes \Proj{\wh f}) \nabla(\nabla^2 \wh f(x)).
\end{align*}
Here $\nabla (\nabla \wh f(x)^\top \otimes \mathbf{I}_d) = (\mathbf{I}_d\otimes\textsc{vec}(\mathbf{I}_d))\nabla^2\wh f(x)$. Furthermore, we can explicitly find the expressions of $\nabla \Proj{\wh f}$ and $\nabla (\nabla \Proj{\wh f})$, which by the chain rule involves the third and fourth derivatives of $\wh f$, respectively. For example, for $\ell=1,\cdots,d$, the $\ell$-th column of $\nabla \Proj{\wh f}$ is given by the vectorization of
\begin{align}\label{xigrad}\resizebox{0.999\hsize}{!}{$
    \sum\limits_{i=1}^{k}\sum\limits_{j=k+1}^d \frac{1}{\lamx[\wh f]{j}-\lamx[\wh f]{i}} [\Vx[\wh f]{j}\Vx[\wh f]{j}^\top \wh H_\ell(x)\Vx[\wh f]{i}\Vx[\wh f]{i}^\top + \Vx[\wh f]{i}\Vx[\wh f]{i}^\top \wh H_\ell(x)\Vx[\wh f]{j}\Vx[\wh f]{j}^\top],$}
\end{align}
where $\wh H_\ell(x)$ is the partial derivative of $\nabla^2\wh f(x)$ with respective to the $\ell$-th component of $x$. See \cite{qiao2023confidence} for some relevant calculations.

\section{Proofs}
\label{appendix:b}
{\bf Proof of Lemma~\ref{misc1}.} First we show that $\text{Ridge}(f)$ is a compact set. Notice that by assumption {\bf (A2)$_f$}, we can write
$$\text{Ridge}(f)=\{x\in [0,1]^d: \xix{f}=0, \;\lam[f]{k+1}(x)\leq 0\},$$ 
because the set $\{x\in[0,1]^d: \xix{f}=0, \;\lam[f]{k+1}(x) = 0\}$ is empty. If we treat $\Proj{f}$ as a function of $\nabla^2f(x)$, due to the assumed positive gap between $\lamx{k}$ and $\lamx{k+1}$ in assumption {\bf (A2)$_f$}, $\Proj{f}$ is an analytical matrix-valued function on the space of real symmetric matrices by the classical matrix perturbation theory~\cite[see][]{kato2013perturbation}, which further implies that $\Proj{f}$ is a Lipschitz function of $x$ by assumption {\bf (A1)$_{f,4}$}. Since both $\xix{f}$ and $\lam{k+1}$ are continuous functions, $\text{Ridge}(f)$ is closed, and hence compact because it is defined on the compact set $[0,1]^d.$

We have $\nabla^2 \eta(x) = -(\mathbf{I}_d \otimes \xix{f}^\top ) \nabla(\nabla \xix{f}) -\nabla \xix{f}^\top\nabla \xix{f}$. Thus, for $x\in\text{Ridge}(f)$, $\xi^f(x)=0$ and $\nabla^2 \eta(x) = -\nabla \xix{f}^\top\nabla \xix{f},$ which, by using assumption {\bf (A3)$_f$},  implies that the rank of $\nabla^2 \eta(x)$ is $d-k$. Hence we have for $x\in\text{Ridge}(f)$,
$$0 = \lamx[\eta]{1} = \cdots = \lamx[\eta]{k} > \lamx[\eta]{k+1}\cdots\geq\lamx[\eta]{d}.$$
Because of the compactness of $\text{Ridge}(f)$ and the continuity of $\lambda_j^\eta$ on $\text{Ridge}(f)$, there exist positive constants $\alpha^\prime$ and $A^\prime$ such that 
\begin{align}
\label{eigenvaluebounds}
-\alpha^\prime \ge \lamx[\eta]{k+1}\cdots\geq\lamx[\eta]{d}\ge - A^\prime
\end{align}
for all $x\in\text{Ridge}(f)$.

Since $\xix{f}=0$ for all $x\in \text{Ridge}(f)$, under assumption {\bf (A1)$_{f,4}$}, when $\delta^\prime$ is small enough, for all  $x\in \text{Ridge}(f)^{\delta^\prime}$,
\begin{align*}
\iota(x)& :=\|(\mathbf{I}_d \otimes \xix{f}^\top ) \nabla(\nabla \xix{f})\|_F \\
& \leq \sqrt{d}\| \xix{f} \|_F \| \nabla(\nabla \xix{f})\|_F\\
& \le 2\sqrt{d}\sup_{x\in \text{Ridge}(f)^{\delta^\prime}}[\| \nabla \xix{f} \|_F \| \nabla(\nabla \xix{f})\|_F]\, d(x,\text{Ridge}(f))\\
& = : c_0 d(x,\text{Ridge}(f)).
\end{align*}
Also for all $x,y\in\text{Ridge}(f)^{\delta^\prime},$ 
\begin{align*}
&\|\nabla \xix{f}^\top\nabla \xix{f} - \nabla \xii{f}{y}^\top\nabla \xii{f}{y}\|_F \\
&\le \|[\nabla \xix{f} -\nabla \xii{f}{y}]^\top \nabla \xix{f}\|_F + \|[\nabla \xix{f} -\nabla \xii{f}{y}]^\top \nabla \xii{f}{y}\|_F\\
&\le 2 \sup_{x\in \text{Ridge}(f)^{\delta^\prime}}[\| \nabla \xix{f} \|_F \| \nabla(\nabla \xix{f})\|_F]\, \|x-y\|\\
&=c_0 \|x-y\|.
\end{align*}
For any $x\in\text{Ridge}(f)^{\delta^\prime}$, let $y_x\in\text{Ridge}(f)$ be such that $\|x-y\|=d(x,\text{Ridge}(f))$. By using Weyl's inequality~\cite[see][page 15]{serre2002matrices}, we have for all $x\in\text{Ridge}(f)^{\delta^\prime},$ and $j=k+1,\cdots,d,$
\begin{align*}
|\lambda_j^\eta(x) - \lambda_j^\eta(y_x)| & \leq \| \nabla^2 \eta(x) - \nabla^2 \eta(y_x)\|_F \\
& \le \iota(x) + \|\nabla \xix{f}^\top\nabla \xix{f} - \nabla \xii{f}{y}^\top\nabla \xii{f}{y}\|_F\\
&\le c_0 [d(x,\text{Ridge}(f)) + \|x-y\|]\\
& \le 2 c_0 \delta^\prime.
\end{align*}
Therefore (\ref{eigenrange}) holds for $\delta^\prime > 0$ small enough by noticing \eqref{eigenvaluebounds}. 

Since $\xix{f}=0$ for all $x\in\text{Ridge}(f),$ the row vectors of $\nabla \xix{f}$ span the normal space of $\text{Ridge}(f)$ at $x$, denoted by $\mathcal{N}(x)$. Note that $\mathcal{N}(x)=\{\nabla \xix{f}^\top a:\; a\in\mathbb{R}^d\}=\{\nabla \xix{f}^\top\nabla \xix{f}a:\; a\in\mathbb{R}^d\}$, which is the same as the space spanned by the eigenvectors of $\nabla^2\eta(x)$ corresponding to the non-zero eigenvalues, i.e., the space spanned by the columns of $\Vbot{\eta}$.\\ 
\hbox{$\;$}\hfill$\blacksquare$\\

{\bf Proof of Lemma~\ref{ridgetube}.}
For any $b\in\mathbb{R}^d$, consider
$$M^f_b:=\{x\in[0,1]^d:\; \xix{f} =b, \lamx{k+1}<0\},$$ 
and note that $\Sep{\eta,f}{\epsilon} =\bigcup_{\{b:\frac{1}{2}\|b\|^2\leq \epsilon\}} M^f_b$ for all $\epsilon\geq0$. Recall $\delta^\prime$ defined in Lemma~\ref{misc1} and let $\mathcal{A} = [0,1]^d\backslash \{t\in[0,1]^d:\eta(t)=0\}^{\delta^\prime}$. Since $\mathcal{A}$ is a compact set, we have $r:=- \sup_{x\in\mathcal{A}} \eta(x) >0.$ Then assumption {\bf (A2)$_f$} implies that $M^f_b\subset \text{Ridge}(f)^{\delta^\prime}$ when $\frac{1}{2}\|b\|^2 <r$, which is also what we assume below. We will first show that for all $\|b\|$ small enough, 
\begin{align*}
C_1\|b\|  \le \sup_{x\in M_b^f} d(x,\text{Ridge}(f)) \le C_2\|b\|,
\end{align*}
for some positive constants $C_1$ and $C_2$. 

For any $x\in M^f_b$, let $x_0$ be its projection point onto $\text{Ridge}(f)$. There exists a unit vector $u=u(x_0)\in\mathscr{N}(x_0)$ where ${\mathscr N}(x_0)$ the normal space to the ridge at $x_0$, and $\delta_x>0$ such that $x=x_0+\delta_x u$. Using a Taylor expansion, we get
\begin{align*}
\xix{f} = \xii{f}{x_0} + \nabla\xii{f}{x_0}(x - x_0) + R_0(x)= \delta_x\nabla\xii{f}{x_0}u + R_0(x),
\end{align*}
where $\|R_0(x)\|\leq \kappa_0\delta_x^2,$ for a constant $\kappa_0>0$ for all $x\in M^f_b$. Therefore 
\begin{align}
\label{bnorm2}
\|b\|^2 = \|\xix{f}\|^2 = \delta_x^2 u^\top \nabla \xii{f}{x_0}^\top\nabla \xii{f}{x_0} u + R_0^\prime(x),
\end{align}
where $|R_0^\prime(x)|\leq \kappa_0^\prime\delta_x^3,$ for a constant $\kappa_0^\prime>0$ for all $x\in M^f_b$. Let $\omega_{\max}(x)$ and $\omega_{\min}(x)$ be the largest and the $(d-k)$th largest eigenvalues of $\nabla \xix{f}^\top\nabla \xix{f}$, respectively. For $x_0\in\text{Ridge}(f)$, we have $\omega_{\min}(x_0)=-\lam[\eta]{k+1}(x_0)$ and $\omega_{\max}(x_0)=-\lam[\eta]{d}(x_0)$ because $\nabla^2\eta(x_0)=-\nabla \xii{f}{x_0}^\top\nabla \xii{f}{x_0}.$ From Lemma~\ref{misc1} we know that for all $x_0\in\text{Ridge}(f)$, $0<\alpha\le\omega_{\min}(x_0)\le \omega_{\max}(x_0) \le A <0$ for some positive constants $A$ and $\alpha$. 
Notice that $u$ is in the space spanned by $\V[\eta]{k+1}(x_0),\cdots,\V[\eta]{d}(x_0)$. So when $\|b\|$ is small enough (and hence $\delta_x$ is small enough), it follows from \eqref{bnorm2} that
\begin{align}\label{b2bound}
2A\delta_x^2 \ge 2|\lam[\eta]{d}(x_0))|\delta_x^2 = 2\omega_{\max}(x_0)\delta_x^2 \geq \|b\|^2 \geq \frac{\omega_{\min}(x_0)}{2} \delta_x^2 =  \frac{|\lam[\eta]{k+1}(x_0)|}{2} \delta_x^2 \ge 2\alpha \delta_x^2.
\end{align}
Since $\delta_x =d(x,\text{Ridge}(f))=\inf_{y\in\text{Ridge}(f)}\|x-y\|$, we have for all $b\in\mathbb{R}^d$ such that $\epsilon=\frac{1}{2}\|b\|^2$ is small enough,
\begin{align}
\label{hausdist1}
\sqrt{\frac{1}{A}}\sqrt{\epsilon} = \sqrt{\frac{1}{2A}} \|b\|  \le \sup_{x\in M_b^f} d(x,\text{Ridge}(f)) \le \sqrt{\frac{2}{\alpha}} \|b\|=\sqrt{\frac{4}{\alpha}}\sqrt{\epsilon}.
\end{align}

The other direction can be proved in a similar way, as given as follows. Now let $x$ be a point on $\text{Ridge}(f)$, and $x_0$ be its projection onto $M_b$, such that $x=x_0+\delta_x u$, where $\delta_x=\|x-x_0\|>0$, and $u\in\mathscr{N}(x_0)$ is a unit normal vector of $M_b^f$ at $x_0$. Following the above analysis (in particular (\ref{b2bound})), when $\|b\|$ is small enough, we still have $2\omega_{\max}(x_0)\delta_x^2\geq \|b\|\geq \frac{1}{2}\omega_{\min}(x_0)\delta_x^2$. Since both $\omega_{\max}$ and $\omega_{\min}$ are continuous functions in a neighborhood of $\text{Ridge}(f)$, we have that $4A\delta_x^2\geq \|b\|^2 \geq \frac{1}{4}\alpha\delta_x^2$ for all $\|b\|$ small enough. Therefore for all $b\in\mathbb{R}^d$ such that $\epsilon=\frac{1}{2}\|b\|^2$ is small enough,
\begin{align}\label{hausdist2}
\sqrt{\frac{1}{2A}}\sqrt{\epsilon} \leq  \sup_{x\in \text{Ridge}(f)} d(x,M_b)  \leq \sqrt{\frac{8}{\alpha}}\sqrt{\epsilon}.
\end{align}
Combining (\ref{hausdist1}) and (\ref{hausdist2}), we obtain
\begin{align}\label{hausdist}
\sqrt{\frac{1}{2A}}\sqrt{\epsilon} \leq d_H(\partial \Sep{\eta,f}{\epsilon},\,\text{Ridge}(f))  \leq \sqrt{\frac{8}{\alpha}}\sqrt{\epsilon}.
\end{align}
\\ \hbox{$\;$}\hfill$\blacksquare$\\

{\bf Proof of Lemma~\ref{ridge-ridgeness}.}
First we show that $\text{\rm Ridge}(f) \subset \text{\rm Ridge}(\eta)\cap \Seps{\eta,f}$. We have $ \nabla \eta(x) = -\nabla \xi^f(x)^\top \xi^f(x)$  for $x\in[0,1]^d.$ Thus, for $x\in \text{Ridge}(f)$ (implying that $\xi^f(x) = 0$), we have $\nabla \eta(x) = 0,$ and hence $\xix{\eta}=0$. Then $\text{\rm Ridge}(f) \subset \text{\rm Ridge}(\eta)\cap \Seps{\eta,f}$ for $\epsilon$ small enough is a direct consequence of Lemma~\ref{misc1}. 

Next we show that $\text{Ridge}(\eta)\cap \Seps{\eta,f} \subset \text{Ridge}(f)$ for $\epsilon$ small enough. To this end we show that
\begin{align}\label{toShow}
    \|\xix{\eta}\| > 0 \quad \text{for all } x\in \Seps{\eta,f}\backslash\text{Ridge}(f),
\end{align}
which implies the result. It follows from Lemma~\ref{ridgetube} that for every $\epsilon>0$ is small enough, there exists $\delta(\epsilon)>0$ such that $\Seps{\eta,f}\subset\text{Ridge}(f)^{\delta(\epsilon)}$ where $\delta(\epsilon)\to 0$ as $\epsilon\to0$. For any $x\in \Seps{\eta,f} \backslash \text{Ridge}(f)$, let $x_0\in\text{Ridge}(f)$ be the projection of $x$ onto $\text{Ridge}(f)$, that is, there exists a $\delta_x \in (0,\delta(\epsilon))$, such that we can write $x=x_0+\delta_x u$, where $u = u(x_0)\in\mathscr{N}(x_0)$ is a unit vector with ${\mathscr N}(x_0)$ the normal space to the ridge at $x_0$. Notice that by using Lemma~\ref{misc1}, ${\mathscr N}(x_0) = \{\Vbot[x_0]{\eta}a:\, a \in \R^{d-k}\}$. We will show that there exists an $\epsilon > 0$, such that $\|\nabla \eta(x)\| > 0$ for all $x \in \Seps{\eta,f}\setminus {\rm Ridge}(f)$. Using Lipschitz continuity of the fourth partial derivatives of $f$, there exists a constant $\kappa_1>0$ such that $\|\nabla^2\eta(x_0) - \nabla^2\eta(x_1)\|_F\leq \kappa_1 \|x_0-x_1\|$ for all $x_1\in \mathcal{B}(x_0,\delta(\epsilon)).$ Then we can write
\begin{align}\label{nablaetax}
\nabla\eta(x) &= \nabla\eta(x_0) + \nabla^2\eta(x_0)(x-x_0) + R_1(x) \nonumber\\
&= \delta_x \nabla^2\eta(x_0) u + R_1(x),
\end{align}
where $\|R_1(x)\|\leq \kappa_1 \delta_x^2$. Note that $\nabla^2\eta(x_0) = \sum_{j=k+1}^d \lam[\eta]{j}(x_0) \V[\eta]{j}(x_0)\V[\eta]{j}(x_0)^\top$ by Lemma~\ref{misc1} and $u$ is a unit vector in the space spanned by $\V[\eta]{k+1}(x_0),\cdots,\V[\eta]{d}(x_0)$. Hence  $A \geq |\lam[\eta]{d}(x_0)| \geq \|\nabla^2\eta(x_0) u\| \geq |\lam[\eta]{k+1}(x_0)|\geq \alpha$, where $\alpha$ and $A$ are positive constants given in Lemma~\ref{misc1}. Thus, for $\epsilon > 0$ small enough (and hence $\delta_x$ is small), we have $\|\nabla \eta(x)\| > 0$ for $x \in \Seps{\eta,f}\setminus {\rm Ridge}(f)$. 

Next we show that $\frac{|\langle\nabla\eta(x),u(x)\rangle|}{\|\nabla\eta(x)\|}$ is bounded from below. Using (\ref{nablaetax}) we have that $|\langle\nabla\eta(x),u(x)\rangle| \geq \delta_x  |\lam[\eta]{k+1}(x_0)|  - \kappa_1 \delta_x^2,$ and 
\begin{align*}
\frac{|\langle\nabla\eta(x),u(x)\rangle|}{\|\nabla\eta(x)\|} \geq \frac{\delta_x  |\lam[\eta]{k+1}(x_0)|  - \kappa_1 \delta_x^2}{\delta_x  |\lam[\eta]{d}(x_0)|  + \kappa_1 \delta_x^2}.
\end{align*}
For $\epsilon$ (and hence $\delta_x$) small enough, the right-hand side is bounded by a positive constant  from below, say, 
\begin{align*}
\frac{|\langle\nabla\eta(x),u(x)\rangle|}{\|\nabla\eta(x)\|}  \geq \frac{1}{2} \frac{|\lam[\eta]{k+1}(x_0)|}{ |\lam[\eta]{d}(x_0)|}. 
\end{align*}
Since $u$ can be written as $\V[\eta]{\bot}(x_0) a$ with $a\in\mathbb{R}^{d-k}$ and $\|a\| = 1$, we have by using Cauchy-Schwarz inequality that $|\langle\nabla\eta(x),u(x)\rangle| = |\nabla\eta(x)^\top \V[\eta]{\bot}(x_0) a| \le \|\V[\eta]{\bot}(x_0)^\top\nabla\eta(x) \| = \|\V[\eta]{\bot}(x_0)\V[\eta]{\bot}(x_0)^\top\nabla\eta(x)\|,$ and thus
\begin{align}\label{lowerb1}
\frac{\|\V[\eta]{\bot}(x_0)\V[\eta]{\bot}(x_0)^\top \nabla\eta(x)\|}{\|\nabla\eta(x)\|} \geq \frac{1}{2} \frac{|\lam[\eta]{k+1}(x_0)|}{ |\lam[\eta]{d}(x_0)|}. 
\end{align}

Furthermore, the angle between the eigenspace spanned by $\V[\eta]{\bot}(x_0)$ and $\Vx[\eta]{\bot}$ should be small if $\delta_x$ is small because $\nabla^2\eta$ is a continuous function of $x$. This can be seen by using the Davis-Kahan inequality:
\begin{align}\label{lowerb2}
\|\V[\eta]{\bot}(x_0)\V[\eta]{\bot}(x_0)^\top - \Vx[\eta]{\bot}\Vx[\eta]{\bot}^\top\|_F \leq \frac{2\sqrt{2}\|\nabla^2 \eta(x_0) - \nabla^2 \eta(x)\|_F}{|\lam[\eta]{k+1}(x_0)|} \leq \frac{2\sqrt{2} \kappa_1\delta_x}{|\lam[\eta]{k+1}(x_0)|}.
\end{align}
Hence using (\ref{lowerb1}) and (\ref{lowerb2}),
\begin{align*}
\frac{\|\xix{\eta}\|}{\|\nabla \eta(x)\|} = \frac{\|\Vx[\eta]{\bot}\Vx[\eta]{\bot}^\top \nabla\eta(x)\|}{\|\nabla\eta(x)\|} \geq \frac{1}{2} \frac{|\lam[\eta]{k+1}(x_0)|}{ |\lam[\eta]{d}(x_0)|} - \frac{2\sqrt{2}\kappa_1\delta_x }{|\lam[\eta]{k+1}(x_0)|},
\end{align*}
which can be bounded from below by a positive constant when $\epsilon$ (and hence $\delta_x$) is small enough. This is (\ref{toShow}).\\ \hbox{$\;$}\hfill$\blacksquare$\\

{\bf Proof of Theorem~\ref{ridgeHausRate}.} Using Talagrand's inequality~\cite[see][Proposition A.5]{sriperumbudur2012consistency}, there exists a constant $C>0$ such that, for all $n\geq 1$, $h\in(0,1)$, $b>1$ and $|\alpha|\leq 4$ with $nh^{d+2|\alpha|}\geq (b\vee |\log h|)$, we have
\begin{align}\label{derivrate}
    \mathbb{P}\Big(\sup_{x\in[0,1]^d} |\partial^{(\alpha)}\wh f(x) -\mathbb{E}\partial^{(\alpha)}\wh f(x)|< C\sqrt{\frac{b\vee |\log h|}{nh^{d+2|\alpha|}}}\Big) \geq 1-e^{-b}.
\end{align}
It follows from standard calculation for kernel density estimation \cite[see, e.g., Lemma 2 of][]{arias2016estimation} that for all $|\alpha|\leq 4,$
\begin{align}\label{derivrate2}
    \sup_{x\in[0,1]^d} |\partial^{(\alpha)} f(x) -\mathbb{E}\partial^{(\alpha)}\wh f(x)| = O(h^{(4-|\alpha|)\wedge2}).
\end{align}
Then (\ref{derivrate}) and (\ref{derivrate2}) imply that for any $B>0$, on a set $A_n$ with probability at least $1-n^{-B}$, there exists a constant $C > 0$ such that
$$\sup_{x\in[0,1]^d} |\partial^{(\alpha)} f(x) -\partial^{(\alpha)}\wh f(x)| \le C\; \Big(\Big(\frac{\log n}{nh^{d+4}}\Big)^{1/2} + h^{(4-|\alpha|)\wedge2}\Big),$$ 
for all $|\alpha|\leq 4.$ The fact that the eigenvalues of a symmetric matrix $M$ are Lipschitz continuous functions of $M$ implies that 
for every $\delta > 0$ there exists $n_0$ such that for $n \ge n_0$, $\sup_{x\in [0,1]^d}|\lamx[\wh f]{k+1}-\lamx[f]{k+1}| \le \delta$ on $A_n.$ By assumption {\bf (A2)}$_f$, this implies that on $A_n$ we have $ \Seps{\wh \eta,\wh f} = \Seps{\wh \eta,f}$ for $\epsilon \geq 0$ small enough and $n$ large enough. Denote $\rho_n=\sup_{x\in [0,1]^d}|\sqrt{-\wh\eta(x)}-\sqrt{-\eta(x)}|$. It 
follows that for $n$ large enough we have on $A_n,$
$$\text{Ridge}(\wh f\,) = S_0^{\wh\eta,\wh f} = S_0^{\wh\eta,f} \subset \big\{x\in [0,1]^d:\; \eta(x)\geq - \rho_n^2,\, \lamx{k+1}<0\big\} = S^{\eta,f}_{\rho_n^2}.$$
We then have on $A_n,$
\begin{align}\label{fhattof}
    \sup_{x\in\text{Ridge}(\wh f )}d(x,\text{Ridge}(f)) \le \sup_{x\in\text{Ridge}(\wh f )}d(x,S^{\eta,f}_{\rho_n^2}) \le C_1\rho_n\le C_2\Big(\Big(\frac{\log n}{nh^{d+4}}\Big)^{1/2} + h^2\Big),
\end{align}
for some constants $C_1,C_2>0$, which follows from Lemma~\ref{ridgetube} and the rate of convergence of $\rho_n$. Indeed, note that
\begin{align*}
\sqrt{2}\rho_n &\leq \sup_{x\in [0,1]^d}\|\Proj{\wh f} \nabla \wh f(x)- \Proj{f} \nabla f(x)\|  \\
&\leq \sup_{x\in [0,1]^d}\|[\Proj{\wh f} - \Proj{f}] \nabla f(x)\| + \sup_{x\in [0,1]^d}\|\Proj{\wh f} [\nabla \wh f(x)- \nabla f(x)]\|\\
&\leq \sup_{x\in [0,1]^d}\|\Proj{\wh f} - \Proj{f}\|_F\;\| \nabla f(x)\| + \sup_{x\in [0,1]^d}\|\nabla \wh f(x)- \nabla f(x)\|\\
&\leq \frac{2\sqrt{2}}{\beta}\sup_{x\in [0,1]^d}\|\nabla^2 \wh f(x) - \nabla^2 f(x)\|_F\sup_{x\in [0,1]^d}\| \nabla f(x)\| + \sup_{x\in [0,1]^d}\|\nabla \wh f(x)- \nabla f(x)\|,
\end{align*}
where the last step follows from the Davis-Kahan theorem \cite[see][]{yu2015useful}. Using (\ref{derivrate}) and (\ref{derivrate2}) gives the asserted rate in (\ref{fhattof}). The result in Theorem~\ref{main1}(i) allows us to swap the roles of $f$ and $\wh f$ in (\ref{fhattof}) and get on a set with probability at least $1-n^{-B}$,
\begin{align}\label{ftofhat}
    \sup_{x\in\text{Ridge}(f )}d(x,\text{Ridge}(\wh f\,)) = C_3\Big(\Big(\frac{\log n}{nh^{d+4}}\Big)^{1/2} + h^2\Big),
\end{align}
for some positive constant $C_3$. We conclude the proof by combining (\ref{fhattof}) and (\ref{ftofhat}).\\ \hbox{$\;$}\hfill$\blacksquare$\\

{\bf Proof of Theorem~\ref{main1}.}
(i). Using (\ref{derivrate}) and (\ref{derivrate2}), properties {\bf (A1)}$_{\wh f,4\;}$ and {\bf (A2)}$_{\wh f}$ are consequences of the Lipschitz continuity of the eigenvalues as functions of symmetric matrices. For a symmetric matrix $A$, let $\lambda_{d-k}(A)$ be the $(d-k)$th largest eigenvalue of $A$. To show {\bf (A3)}$_{\wh f}$, notice that assumption {\bf (A3)}$_{f}$ implies that for $\delta>0$ small enough, there exists a constant $a_0>0$ such that 
\begin{align*}
    \inf_{x\in\text{Ridge}(f)^\delta}\lambda_{d-k} (\nabla \xi(x)^\top\nabla \xi(x)) \geq a_0.
\end{align*}
Due to the perturbation stability of $\lambda_{d-k}$, we have $\inf_{x\in\text{Ridge}(f)^\delta}\lambda_{d-k}(\nabla \xix{\wh f}^\top\nabla \xix{\wh f}) \geq \frac{1}{2}a_0$ with probability at least $1-n^{-B}$ when $n$ is large. This then implies the rank of $\nabla \xix{\wh f}$ is at least $d-k$ for $x\in\text{Ridge}(\wh f)$, if $\text{Ridge}(\wh f)\subset \text{Ridge}(f)^\delta$, which occurs with probability $1-n^{-B}$ according to Theorem~\ref{ridgeHausRate}. Furthermore, using the expression of $\nabla \xix{\wh f}$ in (\ref{nablaxiexp}) and the calculation in (\ref{xigrad}), it can be seen that $\Vx[\wh f]{i}^\top \nabla \xix{\wh f}=0$, $x\in\text{Ridge}(\wh f)$, for all $i=1,\cdots,k$, which implies that the rank of $\nabla \xix{\wh f}$ is at most $d-k$ for $x\in\text{Ridge}(\wh f)$. Hence {\bf (A3)}$_{\wh f}$ is satisfied.\\

(ii). Given part (i), this is a consequence of Lemma~\ref{ridge-ridgeness} and Theorem~\ref{pathconvergence}.\\

(iii). Given part (i), we have the following: For some $\delta>0$, we have for $|\alpha|=0,1,2,$
\begin{align}\label{pertridgeeta}
    \sup_{x\in[0,1]^d} |\partial^{(\alpha)}\wh \eta(x) - \partial^{(\alpha)}\wh \eta_\tau(x)| = O(\tau^2).
\end{align}
This can be shown by using standard techniques for calculating the rate of the bias in kernel density estimation. See, e.g., proof of Lemma 2 in \cite{arias2016estimation}. For $\tau$ small, we have $\text{Ridge}(\wh \eta_\tau)=\{\lim_{t\rightarrow\infty} \gamxt[\wh \eta_\tau],x\in\partial\Seps{\wh \eta_\tau,\wh f}\}$, as a result of Theorem~\ref{stability}(iii). This is a). Part b) is a consequence of Theorem~\ref{stability} (iv). Part c) immediately follows from part b) and Theorem~\ref{ridgeHausRate}.\\ \hbox{$\;$}\hfill$\blacksquare$\\

{\bf Proof of Theorem~\ref{main2}} 
For part (i), the convergence of the sequence follows from Theorem~\ref{main1}(i) and Facts 1 and 2 in the proof of Theorem~\ref{discretecomp}. Then the rates of convergence is a result after applying Theorem~\ref{main1}(i) and Corollary~\ref{perturbeddiscrete}. The result in part (ii) is a consequence of Theorem~\ref{main1} and Corollary~\ref{perturbeddiscrete}.\\
\hbox{$\;$}\hfill$\blacksquare$\\

{\bf Proof of Theorem~\ref{pathconvergence}} The assertion of part (i) is very similar to Lemma 2 in \cite{genovese2014nonparametric} and the proof is similar, too. Details are omitted. Part (ii) follows from LaSalle's Invariance Principle given in Theorem~\ref{LaSalle}. See the beginning of Section~\ref{ODEtheory} for how it is applied to this setting. Next we prove part (iii), for which we will use the technique in the proof of Theorem 2.6 in \cite{nicolaescu2011invitation}. Let
\begin{align}
    \xi^\eta_r : = \frac{\xinox{\eta}}{\|\xinox{\eta}\|^2},
\end{align}
and let $\gamma^\eta_r$ be the flow generated by this rescaled vector field as defined in Section~\ref{continuous}, that is, $\frac{\partial \gamma^\eta_r (x,t)}{\partial t} = \xi^\eta_r(\gamma^\eta_r(x,t)),\quad \gamma^\eta_r (x,0)=x$ for all $x \in \Sep{\eta,f}{\epsilon}.$ Here we require $\epsilon$ to be small enough that $\Sep{\eta,f}{\epsilon}\subset(0,1)^d$, which is possible following Lemma~\ref{ridgetube}. 
We first show part (iii) with $\gamma^{\eta}$ replaced by $\gamma^\eta_r.$ Observe that
\begin{align}\label{unitspeed}
    \frac{\partial \eta(\gamma^\eta_r (x,t)) }{\partial t} = [\nabla \eta(\gamxt[\eta])]^\top \xi^\eta_r(\gamma^\eta_r(x,t))= 1.
\end{align}
In other words, the level of $\eta$ can be used to parametrize the integral curves $\gamma^\eta_r (x,t)$, so that for any $\epsilon_1, \epsilon_2\in(0,\epsilon)$ with $ \epsilon_1>\epsilon_2$, we have
\begin{align}
    &\partial \Sep{\eta,f}{\epsilon_1} = \{\gamma^\eta_r (x,\epsilon_2-\epsilon_1): x\in \partial \Sep{\eta,f}{\epsilon_2}\}\label{epsilon1to2},\\
    &\partial \Sep{\eta,f}{\epsilon_2} = \{\gamma^\eta_r (x,\epsilon_1-\epsilon_2): x\in \partial \Sep{\eta,f}{\epsilon_1}\}\label{epsilon2to1}.
\end{align}
Both (\ref{epsilon1to2}) and (\ref{epsilon2to1}) are similar to Theorem 2.6 in \cite{nicolaescu2011invitation}. We only show (\ref{epsilon1to2}). Using (\ref{unitspeed}), it is clear that for any $x\in {\partial}\Sep{\eta,f}{\epsilon_2}$, $\gamma^\eta_r (x,\epsilon_2-\epsilon_1)\in \partial \Sep{\eta,f}{\epsilon_1}$. This means $\{\gamma^\eta_r (x,\epsilon_2-\epsilon_1): x\in \partial \Sep{\eta,f}{\epsilon_2}\}\subset \partial \Sep{\eta,f}{\epsilon_1}$. 
Note that for any $x\in \partial \Sep{\eta,f}{\epsilon_1}$, there exists $\tilde x = \gamma^\eta_r (x,\epsilon_1-\epsilon_2)$ so that $x = \gamma^\eta_r (\tilde x,\epsilon_2-\epsilon_1) $. This means that $\partial \Sep{\eta,f}{\epsilon_1} \subset \{\gamma^\eta_r (x,\epsilon_2-\epsilon_1): x\in \partial \Sep{\eta,f}{\epsilon_2}\}$. Hence (\ref{epsilon1to2}) is verified.

By using (\ref{epsilon1to2}) and (\ref{epsilon2to1}), Lemma~\ref{ridgetube} implies that, for all $\epsilon^\prime >0$ small enough
\begin{align}\label{hausbound}
   L_1\sqrt{\epsilon^\prime} \leq d_H(\{\gamma^\eta_r (x,{\epsilon-\epsilon^\prime}): x\in {\partial}\Seps{\eta,f}\},\text{Ridge}(f)) \leq L_2\sqrt{\epsilon^\prime}.
\end{align}

Next we show that 
\begin{align}\label{ridgeequiv}
\text{Ridge}(f) = \{\lim_{t\rightarrow \epsilon^-}\gamma^\eta_r(x,t): x \in \partial \Seps{\eta,f}\}.
\end{align}
As shown below, $\gamma^\eta_r(x,\cdot)$ and $\gamx[]{\eta}{\cdot}$ have the same trajectories, and in particular, \begin{align}
\label{samelimit}
\lim_{t\rightarrow\infty} \gamx[]{\eta}{t} = \lim_{s\rightarrow \epsilon^-} \gamx[r]{\eta}{s} \text{ for all } x\in \Seps{\eta,f}.
\end{align}
Recall that $\text{Ridge}(f)$ is a compact set by Lemma~\ref{misc1}. The set $\{\lim_{t\rightarrow \epsilon^-}\gamma^\eta_r(x,t): x
    \in \partial \Seps{\eta,f}\}$ is also compact. This is because $\lim_{t\rightarrow \epsilon^-}\gamma^\eta_r(x,t)$ is a continuous function of $x$ (see Fact~\ref{fact6} below) and $\partial  \Seps{\eta,f}$ is a compact set. Suppose that \eqref{ridgeequiv} is not true. Then there must exist a constant $\delta_0>0$ such that 
    \begin{align*}
        \delta_0 \leq d_H(\{\lim_{t\rightarrow \epsilon^-}\gamma^\eta_r(x,t): x
    \in \partial \Seps{\eta,f}\}, \text{Ridge}(f)).
    \end{align*}
It follows from follows from \eqref{samelimit} and (ii) that $\{\lim_{t\rightarrow \epsilon^-}\gamma^\eta_r(x,t): x
    \in \partial \Seps{\eta,f}\} \subset \text{Ridge}(f)$, and hence there exists $x_0\in\text{Ridge}(f)$ such that 
    \begin{align}\label{contradiction}
    \inf_{x\in \partial \Seps{\eta,f}}\| x_0 - \lim_{t\rightarrow \epsilon^-}\gamma^\eta_r(x,t)\| \ge \delta_0.
    \end{align}
    By \eqref{hausbound}, there exists $x_1\in\partial \Seps{\eta,f}$ and $\epsilon^\prime=(\delta_0/(3(L_2\vee C^\prime)))^2$, where $C^\prime$ is given in Fact~\ref{fact4} below, such that with $x_2=\gamma^\eta_r (x_1,{\epsilon-\epsilon^\prime})$, $\|x_2 - x_0\| \le L_2\sqrt{\epsilon^\prime} \le \delta_0/3$. Let $x_3=\lim_{t\rightarrow \epsilon^-}\gamma^\eta_r(x_0,t)$. Using Fact~\ref{fact4}, we then have $\|x_2-x_3\|\le \int_0^{\epsilon^\prime}\|\gamma^\eta_r(x,t)\|dt\le C^\prime\sqrt{\epsilon^\prime} \le \delta_0/3$. The triangle inequality gives $\|x_3-x_0\|\le 2\delta_0/3 < \delta_0$, which  
    contradicts (\ref{contradiction}) and therefore (\ref{ridgeequiv}) has to be true, which by \eqref{samelimit} implies (iii).

Next we show that $\gamma^\eta_r(x,\cdot)$ and $\gamx[]{\eta}{\cdot}$  have the same trajectories, which makes the proof of (iii) complete. To this end we show a reparameterization relation between the two flows. For each $x\in\partial \Seps{\eta,f},$ let
\begin{align*}
    s(t)  &:= \int_0^t \|\xiii[]{\eta}{u}\|^2 du \\%
    & =  \int_0^t \left[\frac{\partial \eta(\gamx[]{\eta}{u})}{\partial u} \right]  du \\
   & =  \eta(\gamx[]{\eta}{t}) - \eta(\gamx[]{\eta}{0}) \\
    &=  \eta(\gamx[]{\eta}{t}) - \eta(x),
\end{align*}
where the first equality is using (\ref{prop}). Note that we have suppressed the dependence of $s(t)$ on $x$ in the notation. Let $t(s)$ be the inverse of $s(t)$. Then $t(0)=0,$ and
\begin{align}
    \frac{dt(s)}{ds} = \frac{1}{  \|\xiii[]{\eta}{t(s)}\|^2}. 
\end{align}
We obtain $\gamma^\eta_r(x,s) = 
\gamx[]{\eta}{t(s)}$, because
\begin{align}
    \frac{\partial\gamx[]{\eta}{t(s)}}{\partial s} = \xi_r^\eta(\gamma^\eta(x,t(s))),\; \;\;\gamx[]{\eta}{t(0)}=x.
\end{align}
Note that as $t\rightarrow\infty$, we have $s(t)\rightarrow - \eta(x) = \epsilon$ for all $x\in \Seps{\eta,f}$, because $\lim_{t\to\infty}\gamx[]{\eta}{t}\in$ Ridge($f$). Hence we get \eqref{samelimit}.
\\ \hbox{$\;$}
\hfill$\blacksquare$\\

{\bf Proof of Theorem~\ref{stability}.}
First we show that 
\begin{align}\label{wtetaal}
\text{Ridge}(\wt\eta)\cap \Seps{\wt\eta,f} = R(\wt \eta)\cap \Seps{\wt\eta,f},
\end{align}
where $R(\wt \eta) = \{x\in[0,1]^d: \xix{\wt\eta} =0 \}$. Using Lemma~\ref{misc1} and the continuity of eigenvalues as functions of symmetric matrices, when $\max(\delta_0,\delta_1,\delta_2)$ is small enough, we have that for all $x \in \text{Ridge}(f)^{\delta^\prime},$ 
\begin{align}\label{eigenrange2}
 - \frac{1}{2}\alpha \ge \lamx[\wt\eta]{k+1} \ge \cdots\geq\lamx[\wt\eta]{d} \ge - 2A.
\end{align} 
Then for any $\epsilon$ small enough we can choose $\delta_0\in [0,\epsilon)$ small enough such that
\begin{align}\label{sepsilonincl}
     \Sep{\eta,f}{\epsilon-\delta_0} \subset \Seps{\wt\eta,f}\subset \Sep{\eta,f}{\epsilon+\delta_0}\subset \text{Ridge}(f)^{\delta^\prime},
\end{align}
and hence we get (\ref{wtetaal}). Then (i) and (ii) follows from similar arguments for Theorem~\ref{pathconvergence} (i) and (ii).  %

Next we prove (iii). %

Following a similar argument as in the proof of Lemma~\ref{ridgetube}, we can show that for all $\epsilon^\prime>0$ small enough 
\begin{align}\label{Lipschitz2}
    \wt L_1\sqrt{\epsilon^\prime} \leq d_H(\partial \Sep{\wt\eta,f}{\epsilon^\prime},\text{Ridge}(\wt \eta)\cap \Seps{\wt\eta,f}) \leq \wt L_2\sqrt{\epsilon},
\end{align}
for some constants $\wt L_1, \wt L_2>0.$ Then (iii) can be proved following the similar arguments for Theorem~\ref{pathconvergence} (iii). 
Using similar arguments given in the proof of Theorem~\ref{ridgeHausRate}, we can show that there exists a constant $C>0$ such that when $\max(\delta_0,\delta_1,\delta_2)$ and $\epsilon>0$ are small enough, 
\begin{align}\label{2etadiff}
    d_H(\text{Ridge}(\wt\eta)\cap \Seps{\wt\eta,f}, \text{Ridge}(\eta)\cap \Seps{\wt\eta,f}\big) \leq C\max(\delta_1,\delta_2).
\end{align}
The result in (iv) follows from \eqref{sepsilonincl}, (\ref{2etadiff}) and Lemma~\ref{ridge-ridgeness}.

\hfill $\blacksquare$\\

{\bf Proof of Theorem~\ref{discretecomp}.} We will first prove several facts that then lead to Theorem~\ref{discretecomp}. Throughout the proofs below we assume without further mention that assumptions {\bf (A1)$_{f,6}$},  {\bf (A2)$_f$}, and {\bf (A3)$_f$} hold. 

{\bf Notation:} In order to ease the notation in this long proof, we will drop the superscript $\eta$ and write, for instance, $\xi, \gamma, \gamma_a, \lamx[]{k+1},$ $\Proj[x]{},$ $\Vx[]{k+1},$ $\Vbot{}$ for $\xi^\eta, \gamma^\eta, \gamma_a^\eta, \lamx[\eta]{k+1}, \Proj[x]{\eta},$ $\Vx[\eta]{k+1}$ and $\Vbot{\eta}$, respectively. We also write $\Seps{}$ for $\Seps{\eta,f}$. Recall that the superscript notation is defined in Section~\ref{notation} and in Sections~\ref{continuous} and \ref{discrete}.

\begin{customfact}{1}
\label{fact1}
If $\epsilon, a > 0$ are small enough, then, for any starting point $x\in \partial \Seps{}$, the sequence $\gamx[a]{}{\ell},\ell = 0,1,2,\cdots$ stays in $\Seps{}$ and converges to a ridge point in $\Seps{}$.
\end{customfact}

\begin{proof} By (\ref{hausdist}), we can choose $\epsilon$ small enough that $\Seps{}\subset\text{Ridge}(f)^{(\frac{1}{2}\delta^\prime)}$, where $\delta^\prime$ is given in Lemma~\ref{misc1}. Denote $\Sep{}{\epsilon,\delta^\prime}=\big(\Seps{}\big)^{(\frac{1}{2}\delta^\prime)}$. Under assumption {\bf (A2)$_f$} we have 
\begin{align}\label{eigenvaluebound}
    \sup_{x\in \Sep{}{\epsilon,\delta^\prime}} \lamx[f]{k+1}\leq -\beta <0.
\end{align} 
Since $\Seps{}$ is a compact set when $\epsilon$ is small enough and  $\|\xinox{}\|$ is a continuous function on $\Seps{}$ with $\|\xinox{}\|=0$ on $\text{Ridge}(f)$, we can choose $\epsilon$ small enough that $\sup_{x\in \Seps{}} \|\xix{}\|<\frac{1}{2}\delta^\prime$. 

Let $\kappa(y) = \sup\{ \| \nabla^2\eta(z) \|_{\text{op}}: z\in \mathcal{B}(y,\|\xiy{ }\| ) \},$ where $\|\cdot\|_{\text{op}}$ is the operator norm of a matrix. From Lemma~\ref{misc1} it is known that $0<\alpha<\kappa(y)<A<\infty$, for all $y\in \Seps{}$. Choose $0<a<A^{-1}\wedge 1$. Using a Taylor expansion, we have for any $y\in\Seps{}$,
\begin{align*}
\eta(y + a \xiy{}) = \eta(y) + a \nabla\eta(y)^\top \xiy{} + R(y,a) = \eta(y) + a\|\xiy{ }\|^2+ R(y,a),
\end{align*}
where $|R(y,a)| \leq \frac{1}{2} a^2 \kappa(y) \|\xiy{ }\|^2 \leq \frac{1}{2} a\|\xiy{ }\|^2.$ Therefore 
\begin{align*}
\eta(y + a \xiy{ }) \ge \eta(y) + \frac{1}{2} a\|\xiy{ }\|^2 \ge \eta(y).
\end{align*}
Thus, the sequence $\eta(\gamx[a]{ }{\ell}),\ell=0,1,2,\cdots$ is upper bounded by 0 and increasing, and therefore convergent. We can see that if $\gamx[a]{ }{\ell}\in \Seps{}$ then $\gamx[a]{ }{\ell+1}\in \Sep{}{\epsilon,\delta^\prime}$ and hence $\gamx[a]{ }{\ell+1}\in \Seps{}$ using (\ref{eigenvaluebound}). In other words, $\gamx[a]{ }{\ell}$ stays in $\Seps{}$ for all $\ell\geq0.$ Moreover, using the above inequality, we have
\begin{align*}
\eta(\gamx[a]{ }{\ell + 1}) -  \eta(\gamx[a]{ }{\ell}) \geq \frac{1}{2} a \|\xinox{ }(\gamx[a]{ }{\ell})\|^2 ,
\end{align*} 
and therefore
\begin{align*}
\lim_{\ell\rightarrow\infty}  \|\xinox{ }(\gamx[a]{ }{\ell})\| \rightarrow 0.
\end{align*}
Observe further that by using Fact~\ref{fact2} below, $\{\gamx[a]{ }{\ell}\}_\ell$ is a Cauchy sequence, and from what we have just shown, its limit has to be a point on the ridge (and also in $\Seps{}$).
\end{proof}

\begin{customfact}{2}
\label{fact2}
There exists a constant $c_1 > 0$ such that for $a, \epsilon > 0$ small enough, we have for all $x \in \partial\Seps{},$ and $\ell = 0,1,2,\cdots,$
\begin{align}\label{decreseq}
\| \gamx[a]{ }{\ell+2} -\gamx[a]{ }{\ell+1} \| \leq (1- c_1a) \| \gamx[a]{ }{\ell+2} -\gamx[a]{ }{\ell} \|.
\end{align}
As a consequence, the maximal length of the discretized paths with starting points in $\partial\Seps{}$ is bounded by $C\sqrt{\epsilon}$ for a constant $C>0$ not depending on $a$ , i.e., 
\begin{align*}
\sup_{x\in \partial \Seps{}} \sum_{\ell=0}^\infty \| \gamx[a]{ }{\ell+1} -\gamx[a]{ }{\ell} \| \le C\sqrt{\epsilon}.
\end{align*}
\end{customfact}

\begin{proof} We assume that the $a$ and $\epsilon$ are small enough that the sequence $\gamx[a]{}{\ell},\ell = 0,1,2,\cdots$ stays in $\Seps{}$, as given in Fact~\ref{fact1} throughout the proof. First notice that if there exists $\ell_0\geq 0$ such that $\xinox{ }(\gamx[a]{ }{\ell_0})=0$, then $\gamx[a]{ }{\ell_0}=\gamx[a]{ }{\ell_0 +1 }=\gamx[a]{ }{\ell_0 + 2}=\cdots$, and the conclusion of this fact is valid. We thus can assume that $\xinox{ }(\gamx[a]{ }{\ell})\ne 0$ for all $\ell\geq0.$ 
We have the following Taylor expansion
\begin{align*}
\|\xinox{ }(\gamx[a]{ }{\ell + 1})\|^2 &= \|\xinox{ }(\gamx[a]{ }{\ell}\|^2 + \big\langle\nabla \|\xinox{ }(y)\|^2 |_{y=\gamx[a]{ }{\ell}} , \gamx[a]{ }{\ell + 1} - \gamx[a]{ }{\ell} \big\rangle \\
& \quad + \frac{1}{2} [\gamx[a]{ }{\ell + 1} - \gamx[a]{ }{\ell}]^\top \{\nabla^2 \|\xinox{ }(y)\|^2 |_{y=\wt y_\delta}\} [\gamx[a]{ }{\ell + 1}- \gamx[a]{ }{\ell}] \\[5pt]
& = \|\xinox{ }(\gamx[a]{ }{\ell}\|^2 + a \big\langle \nabla \|\xinox{ }(y)\|^2 |_{y=\gamx[a]{ }{\ell}}, \xinox{ }(\gamx[a]{ }{\ell} \big\rangle \\
& \qquad\quad+ a^2 \frac{1}{2} [\xinox{ }(\gamx[a]{ }{\ell}]^\top \{\nabla^2 \|\xinox{ }(y)\|^2 |_{y=\wt y_{t,\ell}}\} [\xinox{ }(\gamx[a]{ }{\ell}] 
\end{align*}
where $\wt y_{t,\ell} = t \gamx[a]{ }{\ell} + (1-t)\gamx[a]{ }{\ell + 1}$ for some $t\in(0,1)$. From the proof of Fact~\ref{fact1}, we see that $a$ and $\epsilon$ can be chosen small enough that $\{\wt y_{t,\ell}:t\in[0,1]\}\subset \Seps{}$ for all $\ell\geq0,$ which will also be assumed for the remaining proofs. 

Let $\nu(y) = \lambda_{\max}[\nabla^2 \|\xinox{ }(y)\|^2]$, where $\lambda_{\rm max}(B)$ denotes the maximum eigenvalue of a symmetric matrix $B$. For each $y\in\text{Ridge}(f)$, because $\nabla\eta(y)=0$ and $\xi(y)=0$, we can write 
\begin{align*}
\nabla^2 \|\xinox{ }(y)\|^2 = 2[\nabla\xinox{ }(y)]^\top\nabla\xinox{ } (y)=2\nabla^2\eta(y) \Vbot[y]{ }[\Vbot[y]{ }]^\top\nabla^2\eta(y).
\end{align*}
Recalling that $\Vbot[y]{ }$ is the matrix built by the trailing $k-d$ (unit) eigenvectors of $\nabla^2\eta(y)$, we see that $\nu(y)=2[\lam[ ]{d}(y)]^2$ for $y\in\text{Ridge}(f).$ Since $\nu$ is a continuous function on $\text{Ridge}(f)^{\delta^\prime}$, where $\delta^\prime$ is given in Lemma~\ref{misc1}, we can find an $\epsilon>0$ small enough that $0<\lambda_{\max}^* : = \sup_{y\in \Seps{}}\nu(y)<\infty.$ 
We thus obtain
\begin{align}\label{squaredbound}
&\|\xinox{ }(\gamx[a]{ }{\ell + 1}\|^2  \nonumber\\ &\leq\|\xinox{ }(\gamx[a]{ }{\ell}\|^2 + 2a [\xinox{ }(\gamx[a]{ }{\ell}]^\top[\nabla \xinox{ }(\gamx[a]{ }{\ell}] [\xinox{ }(\gamx[a]{ }{\ell}] + \frac{1}{2}a^2 \|\xinox{ }(\gamx[a]{ }{\ell })\|^2 \lambda_{\max}^*.
\end{align}
Hence
\begin{align}\label{ratiotest}
\frac{\| \gamx[a]{ }{\ell + 2} - \gamx[a]{ }{\ell + 1} \|^2}{\| \gamx[a]{ }{\ell + 1} - \gamx[a]{ }{\ell} \|^2} &= \frac{\|\xinox{ }(\gamx[a]{ }{\ell + 1 })\|^2 }{\|\xinox{ }(\gamx[a]{ }{\ell })\|^2} \nonumber \\
&\hspace*{-2cm}\leq 1 + 2 a \frac{[\xinox{ }(\gamx[a]{ }{\ell })]^\top[\nabla \xinox{ }(\gamx[a]{ }{\ell })] [\xinox{ }(\gamx[a]{ }{\ell })]}{\|\xinox{ }(\gamx[a]{ }{\ell })\|^2} + \frac{1}{2}a^2\lambda_{\max}^*.
\end{align}
For any $y\in\Seps{}$, let $\mathscr{N}(y)=\{\Proj[y]{} u:\; u\in\mathbb{R}^d\}$, which is the $(d-k)$-dimensional subspace spanned by the orthonormal eigenvectors $\Vy[  ]{k+1},\cdots,\Vy[  ]{d}$. Also let ${\mathcal S}_\circ(y) = \{v/\|v\|, v \in \mathscr{N}(y) \setminus \{\mathbf{0}\}\}$, which is a $(d-k)$-dimensional unit sphere in $\R^{d}$. Define
\begin{align}\label{barbeta}
    \bar\beta(y) = \sup_{v\in \mathscr{N}(y)\backslash\{ \mathbf{0}\}} \frac{v^\top \nabla \xiy{} v}{\|v\|^2} = \sup_{w\in {\mathcal S}_\circ(y)} w^\top \nabla \xiy{} w.
\end{align}
Then from (\ref{ratiotest}) we can write 
\begin{align}
\label{ratiobound}
\frac{\| \gamx[a]{  }{\ell + 2} - \gamx[a]{  }{\ell + 1} \|}{\| \gamx[a]{  }{\ell + 1} - \gamx[a]{  }{\ell} \|} \leq 1 +  a \bar\beta(\gamx[a]{ }{\ell }))+ \frac{1}{4}a^2\lambda_{\max}^* .
\end{align}
For any $y\in\text{Ridge}(f)$, we can write $\nabla \eta(y)=0$ and  
    $\nabla \xiy{  } = \Vbot[y]{  }[\Vbot[y]{  }]^\top \nabla^2 \eta(y) = \sum_{j=k+1}^d \lam[  ]{j}(y) \Vy[  ]{j}\Vy[  ]{j}^\top,$ 
and therefore
\begin{align}\label{betaupperbound}
      \bar\beta(y) = \lam[  ]{k+1}(y) <0.
\end{align} 
Next we will show that $\bar\beta$ is continuous in $\text{Ridge}(f)^\delta$. To this end, we will, for any $y,y_0\in\text{Ridge}(f)^\delta$, find a bijective map $q_{y,y_0}: {\mathcal S}_\circ(y) \rightarrow {\mathcal S}_\circ(y_0)$. Let $\pi/2\geq \theta_1\geq\cdots\geq\theta_{d-k}\geq0$ be the principal angles between $\mathscr{N}(y)$ and $\mathscr{N}(x_0)$, and suppose that $u_1(\tilde y),\cdots,u_k(\tilde y)$ are the associated principal vectors for $\mathscr{N}(\tilde y)$, where $\tilde y\in\{y,y_0\},$ respectively. In other words, if the singular value decomposition of $\Vbot[y]{  }^\top \Vbot[y_0]{  }$ is given by $P\Sigma R^\top$, where $P$ and $R$ are $(d-k)\times(d-k)$ orthogonal matrices and $\Sigma$ is a $(d-k)\times(d-k)$ diagonal matrix, then $[u_1(y),\cdots,u_k(y)]=\Vbot[y]{  }P$, $\Sigma=\cos\Theta$, and $[u_1(y_0),\cdots,u_k(y_0)]=\Vbot[y_0]{  }R$, where $\Theta=\text{diag}(\theta_1,\cdots,\theta_{d-k}).$ 
Using the Davis-Kahan theorem and Lemma~\ref{misc1}, we have 
\begin{align}\label{sintheta}
\|\sin\Theta\|_F \leq \frac{2}{\beta}\|\nabla^2\eta(y) - \nabla^2\eta(y_0)\|_F.
\end{align}
We choose $\|y-y_0\|$ small enough that $\theta_1\leq 2\pi/3$, so that $\sin\theta_i\geq \sin(\theta_i/2)$ for $i=1,\cdots,d-k$. Let $U(y,y_0)=[\bar u_1(y,y_0),\cdots,\bar u_{d-k}(y,y_0)]$, where $\bar u_i(y,y_0)=\frac{1}{\sqrt{2(1+\cos\theta_i)}}[u_i(y) + u_i(y_0)]$, $i=1,\cdots,d-k,$ and $\bar{\mathscr{N}}(y,y_0)$ be the column space of $U(y,y_0)$. Note that the columns of $U(y,y_0)$ are orthonormal and $\bar{\mathscr{N}}(y,y_0)$ is a subspace about which $\mathscr{N}(y)$ and $\mathscr{N}(y_0)$ are symmetric. Therefore the images of ${\mathcal S}_\circ(x)$ and ${\mathcal S}_\circ(x_0)$ under the projection map $\Omega_{y,y_0}:=U(y,y_0)[U(y,y_0)]^\top$ are the same. For $w\in{\mathcal S}_\circ(y)$, define $q_{y,y_0}(w)$ in such a way that $\Omega_{y,y_0}q_{y,y_0}(w)=w$. Here $q_{y,y_0}(w)$ is uniquely defined because $\Omega_{y,y_0}$ is bijective from either $\mathscr{N}(y)$ or $\mathscr{N}(y_0)$ to $\bar{\mathscr{N}}(y,y_0)$. For $w\in{\mathcal S}_\circ(y)$, using (\ref{sintheta}), we have
\begin{align}\label{qyydiff}
\|q_{y,y_0}(w) - w\| = 2\sqrt{\sin^2\Big(\frac{\theta_1}{2}\Big) + \cdots + \sin^2\Big(\frac{\theta_{d-k}}{2}\Big)} \leq 2\|\sin\Theta\|_F\leq \frac{4}{\beta}\|\nabla^2\eta(y) - \nabla^2\eta(y_0)\|_F.
\end{align}
Then we can write
\begin{align}\label{barbetadiff}
|\bar\beta(y) - \bar\beta(y_0)| 
& \leq  \Big|\sup_{w\in {\mathcal S}_\circ(y)} w^\top \nabla \xiy{  } w - \sup_{w\in {\mathcal S}_\circ(y)} w^\top \nabla \xii{  }{y_0} w\Big| \nonumber \\
&\hspace{1cm} + \Big|\sup_{w\in {\mathcal S}_\circ(y)} w^\top \nabla \xii{  }{y_0} w - \sup_{w\in {\mathcal S}_\circ(y_0)} w^\top \nabla \xii{  }{y_0} w\Big| \nonumber\\
& \leq \sup_{w\in {\mathcal S}_\circ(y)}\Big| w^\top \nabla \xiy{  } w - w^\top \nabla \xii{  }{y_0} w\Big| \nonumber \\
&\hspace{1cm} + \sup_{w\in {\mathcal S}_\circ(y)}\Big| w^\top \nabla \xii{  }{y_0} w -  q_{y,y_0}(w)^\top \nabla \xii{  }{y_0} q_{x,x_0}(w)\Big| \nonumber\\
& \leq \|\nabla \xiy{  } - \nabla \xii{  }{y_0}\|_F + 2 \|\nabla \xii{  }{y_0})\|_F \sup_{w\in {\mathcal S}_\circ(y)} \|w-q_{y,y_0}(w)\| \nonumber\\
&\leq \|\nabla \xiy{  } - \nabla \xii{  }{y_0}\|_F +  \frac{8\|\nabla \xii{  }{y_0})\|_F}{\beta}\|\nabla^2\eta(y) - \nabla^2\eta(y_0)\|_F.
\end{align}
Using the boundedness of $\|\nabla \xinox{  }\|_F$, we obtain that $\bar\beta(y)$ is a uniformly continuous function on $\text{Ridge}(f)^\delta$. Because the ridge is a compact set, using (\ref{betaupperbound}), we are able to find $\epsilon>0$ such that 
\begin{align}\label{beta0def}
    \beta_0: = \sup_{y\in \Seps{}} \bar\beta(y) <0,
\end{align}
and from (\ref{ratiobound}) we have that for all $\ell\geq 0$,
\begin{align*}
\frac{\| \gamx[a]{  }{\ell + 2} - \gamx[a]{  }{\ell + 1} \|}{\| \gamx[a]{  }{\ell + 1} - \gamx[a]{  }{\ell} \|} \leq 1 +  a \bar\beta(\gamx[a]{ }{\ell }))+ \frac{1}{4}a^2\lambda_{\max}^* \leq 1+ a\beta_0+ \frac{1}{4}a^2\lambda_{\max}^*.
\end{align*}
We require that $a\in(0,-\frac{2\beta_0}{\lambda_{\max}^*})$. Then
\begin{align*}
\frac{\| \gamx[a]{  }{\ell + 2} - \gamx[a]{  }{\ell + 1} \|}{\| \gamx[a]{  }{\ell + 1} - \gamx[a]{  }{\ell} \|} \leq 1+ \frac{1}{2}a\beta_0.
\end{align*}
This is (\ref{decreseq}) with $c_1=-\beta_0/2$. Let 
\begin{align}\label{kappadagger}
\kappa^\dagger(\epsilon)=\sup_{x\in \partial\Seps{}}\|\xix{  }\|.
\end{align}
Notice that $0 < 1 + a\beta_0 < 1$. We can then write for any $x\in\partial\Seps{}$,
\begin{align}\label{tabound}
    T_a(x) := \sum_{\ell=0}^\infty \| \gamx[a]{  }{\ell + 1} - \gamx[a]{  }{\ell} \| \leq & \sum_{\ell=0}^\infty a \|\xix{  }\| \Big(1+ \frac{1}{2}a\beta_0\Big)^\ell \leq \frac{2\kappa^\dagger(\epsilon)}{-\beta_0}.
\end{align}

For any $x\in\partial\Seps{}$,
\begin{align}
\label{xixbound}
\|\xix{  }\| \le \|\Vbot{}\Vbot{}^\top\|_F \|\nabla\eta(x)\| = \sqrt{d-k}\|\nabla\eta(x)\|.
\end{align}
Recall that the notation used in the proof of Lemma~\ref{ridge-ridgeness}, in particular, $x_0$ is the projection of $x$ onto $\text{Ridge}(f)$. It follows from \eqref{nablaetax} that 
\begin{align*}
\|\nabla\eta(x)\| \le A \|x-x_0\| + \kappa_1 \|x-x_0\|^2,
\end{align*}
where $A$ is given in Lemma~\ref{misc1}. Then using \eqref{hausdist}, we get
\begin{align*}
\|\nabla\eta(x)\| \le A \sqrt{\frac{8}{\alpha}} \sqrt{\epsilon} + \kappa_1 \frac{8}{\alpha} \epsilon \leq (A+\kappa_1) \sqrt{\frac{8}{\alpha}} \sqrt{\epsilon},
\end{align*}
where the last inequality is true when $8\epsilon\le\alpha$. Combined with \eqref{xixbound}, this leads to 
\begin{align}
\label{kappadaggerbound}
\kappa^\dagger(\epsilon) \le (A+\kappa_1) \sqrt{\frac{8(d-k)}{\alpha}} \sqrt{\epsilon}.
\end{align}
It then follows from \eqref{tabound} that $T_a(x) \le C\sqrt{\epsilon}$, where $C = \frac{2}{-\beta_0} (A+\kappa_1) \sqrt{\frac{8(d-k)}{\alpha}}$.
\end{proof}

Next we will show the continuity of $\lim_{\ell\rightarrow\infty} \gamx[a]{  }{\ell}$ as a function of $x$. \\

\begin{customfact}{3}
\label{fact3}
The following holds when $\epsilon$ is small enough: Let $x,x_0\in\partial \Seps{}$ be two starting points.  For any $\omega>0$, there exists $\delta_\omega>0$ such that when $\|x-x_0\|\leq \delta_\omega$, we have 
\begin{align*}
\|\lim_{\ell\rightarrow\infty} \gamx[a]{  }{\ell} - \lim_{\ell\rightarrow\infty} \gamma_a^{  }(x_0,\ell)\| \leq \omega.
\end{align*}
\end{customfact}

\begin{proof} Note that for $\tilde x\in\{x,x_0\}$,
\begin{align*}
\gamma_a^{  }(\tilde x,\ell+1) = \tilde x + \sum_{i=0}^\ell[\gamma_a^{  }(\tilde x,i+1) - \gamma_a^{  }(\tilde x,i)] = \tilde x + a \sum_{i=0}^\ell  \xi^{  }(\gamma^{  }_a(\tilde x,i)).
\end{align*}
Then using a Taylor expansion we have
\begin{align*}
 \gamma^{  }_a(x,\ell+1) - \gamma^{  }_a(x_0,\ell+1) 
=\; & (x-x_0) + a \sum_{i=0}^\ell \{ \xi (\gamma_a(x,i)) - \xi(\gamma_a(x_0,i)) \} \\
=\; & (x-x_0) + a \sum_{i=0}^\ell \int_0^1\nabla \xi(\gamma_{a,i}(t,x,x_0)) dt \{ \gamma_a(x,i) - \gamma_a(x_0,i) \},
\end{align*}
where $\gamma_{a,i}(t,x,x_0) = t \gamma_a(x,i) +(1-t) \gamma_a(x_0,i)$. As in the proof of Fact~\ref{fact1}, we can choose $\epsilon$ small enough such that $\Seps{}\subset\text{Ridge}(f)^{(\frac{1}{2}\delta^\prime)}$. Recalling that $\|\xix{}\|=0$ for $x\in\text{Ridge}(f)$, we obtain from Fact~\ref{fact2} that $\epsilon$ and $a$ can be chosen small enough that $\max\{T_a(x),T_a(x_0)\}\leq \frac{1}{6}\delta^\prime$ for all $x,x_0\in\partial \Seps{}$, where $T_a$ is defined in \eqref{tabound}. Suppose $\|x-x_0\|\leq \frac{1}{6}\delta^\prime$ so that by elementary geometric facts,  $\gamma_{a,i}(t,x,x_0) \in \text{Ridge}(f)^{\delta^\prime}$ for all $i\geq 0$, and $t\in[0,1]$. Hence with 
\begin{align}\label{kappastar}
\kappa^*:=\sup_{x \in \text{Ridge}(f)^{\delta^\prime}}\|\nabla \xix{  })\|_F<\infty,
\end{align}
we have
\begin{align}\label{sequencenorm}
\|\gamma_a(x,\ell+1) - \gamma_a(x_0,\ell+1)\| 
\leq \;& \|x-x_0\| + a \kappa^* \sum_{i=0}^\ell \| \gamma_a(x,i) - \gamma_a(x_0,i)\|. 
\end{align}

We use the following discrete Gronwall's inequality \cite[see][]{holte2009discrete}:
{\em Let $\{y_n\}$ and $\{g_n\}$ be nonnegative sequences and $c$ a nonnegative constant. If $$y_n\leq c + \sum_{0\leq k <n} g_k y_k, \; n\geq 0,$$ then $$y_n\leq c\exp\left( \sum_{0\leq j<n} g_j\right),\; n\geq 0.$$
}
Applying this inequality to (\ref{sequencenorm}), we get
\begin{align}\label{Gronwallineq}
\|\gamma_a(x,\ell+1) - \gamma_a(x_0,\ell+1)\| \leq \|x-x_0\|\exp(a\kappa^* \ell).
\end{align}
Recall $\kappa^\dagger=\kappa^\dagger(\epsilon)$ given in (\ref{kappadagger}). Using the argument in the proof of Fact 2 (in particular, see (\ref{tabound})), we have that for any positive integer $N$,
\begin{align}\label{totalengthbound}
   \|\gamx[a]{  }{N} - \lim_{\ell\rightarrow\infty} \gamx[a]{  }{\ell}\| \leq &\sum_{\ell=N}^\infty \| \gamx[a]{  }{\ell+1} -\gamx[a]{  }{\ell} \| \nonumber \\
    \leq & a \kappa^\dagger \sum_{\ell=N}^\infty  [1- ac_1]^\ell \nonumber \\
    =& \kappa^\dagger \frac{[1- ac_1 ]^N}{c_1},
\end{align}
where $c_1$ is given in \eqref{decreseq} and $a$ is small enough that $0 < 1- ac_1 < 1.$ We then obtain for $N_\omega:=\log(\frac{1}{3\kappa^\dagger } c_1\omega )/\log(1-ac_1)$, and any $x,x_0\in\partial \Seps{}$,
\begin{align}
&\|\gamx[a]{  }{N_\omega + 1} - \lim_{\ell\rightarrow\infty} \gamx[a]{  }{\ell}\| \leq \frac{\omega}{3}, \label{gammadistx}
\\
&\|\gamma_a(x_0,N_\omega + 1) - \lim_{\ell\rightarrow\infty} \gamma_a(x_0,\ell)\| \leq \frac{\omega}{3}.\label{gammadistx0}
\end{align}
Using the fact 
  $  \frac{t}{1+t} \leq \log(1+t) \leq t,$ for all $t>-1$,
we get
\begin{align*}
    1-ac_1 \leq \frac{-ac_1}{\log(1-ac_1)} \leq 1,
\end{align*}
and consequently
\begin{align*}
    \exp(a\kappa^* N_\omega ) = \exp \left(\frac{-ac_1}{\log(1-ac_1)} \Big[-\frac{\kappa^*}{c_1}\log( \frac{c_1\omega}{3\kappa^\dagger }) \Big]\right)\leq \exp \left(-\frac{\kappa^*}{c_1}\log\Big(  \frac{c_1\omega}{3\kappa^\dagger}\Big) \right).
\end{align*}
Using (\ref{Gronwallineq}), we choose $\delta_\omega>0$ small enough (independent of $a$ once it satisfied the above requirements) such that when $\|x-x_0\|\leq \delta_\omega$,
\begin{align*}
\|\gamx[a]{  }{N_\omega+1} - \gamma^*_a(x_0,N_\omega+1)\| &\leq \|x-x_0\|\exp(a\kappa^* N_\omega )\\
&\leq \delta_\omega \exp \left(-\frac{\kappa^*}{c_1}\log\Big( \frac{c_1 \omega}{3\kappa^\dagger}\Big) \right) \leq \frac{\omega}{3}.
\end{align*}
Combining this with (\ref{gammadistx}) and (\ref{gammadistx0}), we have
\begin{align*}
& \|\lim_{\ell\rightarrow\infty} \gamx[a]{  }{\ell} - \lim_{\ell\rightarrow\infty} \gamma_a(x_0,\ell)\| \\
\leq &\, \|\gamx[a]{  }{N_\omega+1} - \lim_{\ell\rightarrow\infty} \gamx[a]{  }{\ell}\| + \|\gamx[a]{  }{N_\omega + 1} - \gamma_a(x_0,N_\omega+1)\| \\
&\hspace*{3cm}+\|\gamma_a(x_0,N_\omega+1) - \lim_{\ell\rightarrow\infty} \gamma_a(x_0,\ell)\| \\
 \leq & \,\omega.
\end{align*}
This proves the assertion of this lemma. Note that the set $\partial \Seps{}$ is a compact set, so the continuity of the function $x\mapsto \lim_{\ell\rightarrow\infty} \gamx[a]{  }{\ell}$ on this set is equivalent to its uniform continuity. 
\end{proof}

The following fact is a continuous version of Fact~\ref{fact2}.

\begin{customfact}{4}
\label{fact4}
When $\epsilon>0$ is small enough, the maximal length of the paths $\gamxt[  ]$ with starting points in $\partial\Seps{}$ is bounded by $C^\prime\sqrt{\epsilon}$ for a constant $C^\prime>0$ not depending on $a$ , i.e., 
\begin{align*}
\sup_{x\in \Seps{}}\int_0^\infty \|\xi(\gamxt[])\|dt \le C^\prime\sqrt{\epsilon}.
\end{align*}
\end{customfact}

\begin{proof} The proof is similar to that of Fact 2, but more involved. Let $t_\ell = a \ell$. Then notice that 
\begin{align*}
    \int_0^\infty \|\xi(\gamxt[  ])\|dt= \sum_{\ell=0}^\infty \int_{t_{\ell}}^{t_{\ell+1}} \|\xi(\gamxt[  ])\|dt.
\end{align*}
For each $\ell\geq 0$, using a Taylor expansion of the function $s \mapsto \int_{t_{\ell}}^{s} \|\xi(\gamxt[  ])\|dt$, we have
\begin{align}\label{tellint}
\int_{t_{\ell}}^{t_{\ell+1}} \|\xi(\gamxt[  ])\|dt = a\|\xi(\gamx{  }{t_\ell})\| + \frac{1}{2}a^2 \frac{\partial}{\partial t} \|\xi(\gamxt[  ])\| \Big|_{t=\tilde t_\ell},
\end{align}
where $\tilde t_\ell=(1-\delta_\ell) t_\ell + \delta_\ell t_{\ell+1}=t_\ell + a\delta_\ell$ for some $\delta_\ell\in(0,1).$ Here
\begin{align}\label{lengthderiv}
    \frac{\partial}{\partial t} \|\xi(\gamxt[  ])\| = & \left\{ \left[\nabla \|\xii{  }{y}\| \right]^\top \xii{  }{y} \right\}\Big|_{y=\gamxt[  ]} \nonumber\\
    = & \left\{\|\xii{  }{y}\| \theta(y) \right\}|_{y=\gamxt[  ]},
\end{align}
where $\theta(y) = \xi_\diamond   (y)^\top \nabla \xii{  }{y}  \xi_\diamond   (y)$ with $\xi_\diamond   (y) = \xii{   }{y}/\|\xii{   }{y}\|$. So from (\ref{tellint}) we can write
\begin{align}\label{xiintegral}
\int_{t_{\ell}}^{t_{\ell+1}} \|\xi(\gamxt[  ])\|dt = a\|\xi(\gamx{  }{t_\ell}\| + \frac{1}{2}a^2 \|\xi(\gamx{  }{\tilde t_\ell)}\| \theta(\gamx{  }{\tilde t_\ell}).
\end{align}
Recall that in the proof of Fact~\ref{fact2} we have shown that when $\epsilon>0$ is small enough,
\begin{align}\label{thetanegative}
    \sup_{y\in \Seps{}\backslash\text{Ridge}(f)} \theta(y) \leq \sup_{y\in \Seps{}}\bar\beta(y) = \beta_0 <0,
\end{align}
which, by (\ref{lengthderiv}), implies that 
\begin{align}
\label{xinormdecrease}
\|\xi(\gamxt[  ])\| \text{ is a decreasing function of }t.
\end{align}
Corresponding to the definition of $\bar \beta$ in (\ref{barbeta}), define
\begin{align}\label{barbelowbeta}
    \barbelow\beta(y) = \inf_{v\in \mathscr{N}(y)\backslash\{ \mathbf{0}\}} \frac{v^\top \nabla \xiy{  } v}{\|v\|^2} = \inf_{w\in \mathscr{S}_\circ(y)} w^\top \nabla \xiy{  }w.
\end{align}
Using an argument similar to (\ref{qyydiff}) and (\ref{barbetadiff}), we can show that $\barbelow\beta$ is a continuous function on $\text{Ridge}(f)^{\delta^\prime}$. Similar to (\ref{betaupperbound}), we see that for any $y\in\text{Ridge}(f)$, $\barbelow\beta(y) = \lambda_{d}(y) <0$. Therefore we can find $\epsilon>0$ small enough that $\beta_1 := \inf_{y\in \Seps{}} \barbelow\beta(y) >-\infty.$ Note that $\beta_1\leq \beta_0<0$. Then 
\begin{align}
\label{thetalower}
\inf_{y\in \Seps{}\backslash\text{Ridge}(f)} \theta(y) \geq \beta_1.
\end{align}

Let $\rho_\ell(x) = \gamx{  }{t_{\ell+1}} - \gamx{  }{t_{\ell}}$. Similar to (\ref{squaredbound}) we have
\begin{align}\label{xicontibound}
&\|\xi(\gamx{  }{t_{\ell+1}})\|^2 \nonumber\\
 & \leq \|\xi(\gamx{  }{t_{\ell}})\|^2 + 2 [\rho_\ell(x)]^\top[\nabla \xi(\gamx{  }{t_{\ell}})] \xi(\gamx{  }{t_{\ell}}) + \frac{1}{2}\|\rho_\ell(x)\|^2 \lambda_{\max}^*.
\end{align}
We have the following Taylor expansion
\begin{align*}
    \rho_\ell(x) = \int_{t_\ell}^{t_{\ell+1}} \xi(\gamxt[  ]))dt = a \xi(\gamx{  }{t_{\ell}}) + a^2 [\nabla \xii{  }{y}\xii{  }{y}]] \big|_{y=\gamma(x,\bar t_\ell)},
\end{align*}
where $\bar t_\ell = t_\ell + \bar\delta_\ell \times a$ for some $\bar\delta_\ell\in[0,1].$ Note that $\bar t_\ell$ in the Taylor expansion may be different for the entries of the vector $\rho_\ell$. Recall $\kappa^*$ given in (\ref{kappastar}). Then we have
\begin{align*}
    \Big\|[\nabla \xii{  }{y}\xii{  }{y}] \big|_{y=\gamx{  }{\bar t_\ell}} \Big\| \leq \kappa^* \|\xi(\gamx{  }{\bar t_\ell})\| \leq \kappa^* \|\xi(\gamx{  }{t_\ell})\|,
\end{align*}
where we have used \eqref{xinormdecrease}, and hence
\begin{align*}
    \|\rho_\ell(x)\| \leq  \|\xi(\gamx{  }{t_\ell})\| (a+ a^2 \kappa^*).
\end{align*}
Then from (\ref{xicontibound}) we have
\begin{align*}
\|\xi(\gamma&(x,t_{\ell+1}))\|^2 \\
&\leq  \|\xi(\gamx{  }{t_{\ell}})\|^2 + 2 a [\xi(\gamx{  }{t_{\ell}})]^\top[\nabla \xi
  (\gamx{  }{t_{\ell}})] \xi(\gamx{  }{t_{\ell}}) \\
& \hspace{0.8cm} + 2 a^2 [[\nabla \xii{  }{y}\xii{  }{y}] \big|_{y=\gamx{  }{\bar t_\ell}}]^\top[\nabla \xi  (\gamx{  }{t_\ell})] \xi(\gamx{  }{t_\ell}) + \frac{1}{2}\|\rho_\ell(x)\|^2 \lambda_{\max}^* \\
&\leq \|\xi(\gamx{  }{t_\ell})\|^2 + 2a [\xi(\gamx{  }{t_\ell})]^\top[\nabla \xi(\gamx{  }{t_\ell})] \xi(\gamx{  }{t_\ell}) \\
&\hspace{0.8cm} + 2a^2(\kappa^*)^2  \| \xi(\gamx{  }{t_\ell})\|^2 + \frac{1}{2} (a+a^2\kappa^*)^2\lambda_{\max}^* \; \| \xi(\gamx{  }{t_\ell})\|^2 .
\end{align*}
This leads to
\begin{align*}
    &\frac{\|\xi   (\gamx{   }{t_{\ell + 1}})\|^2}{\|\xi  (\gamx{   }{t_\ell})\|^2}\\ &\leq 1 + 2a \frac{[\xi  (\gamx{   }{t_\ell})]^\top[\nabla \xi  (\gamx{   }{t_\ell})] \xi  (\gamx{   }{t_\ell})}{\|\xi  (\gamx{   }{t_\ell})\|^2} + 2a^2(\kappa^*)^2 + \frac{1}{2} (a+a^2\kappa^*)^2\lambda_{\max}^* \\
    &= 1 + 2a\theta(\gamx{   }{t_\ell})+ 2a^2(\kappa^*)^2 + \frac{1}{2} (a+a^2\kappa^*)^2\lambda_{\max}^*.
\end{align*}
Therefore
\begin{align}\label{xistarratio}
    \frac{\|\xi   (\gamx{   }{t_{\ell + 1}})\|}{\|\xi   (\gamx{   }{t_\ell})\|} \leq  1 + a\theta(\gamx{   }{t_\ell})+ a^2(\kappa^*)^2 + \frac{1}{4} (a+a^2\kappa^*)^2\lambda_{\max}^*.
\end{align}
From (\ref{xiintegral}) we have
\begin{align*}
\int_{t_{\ell+1}}^{t_{\ell+2}} \|\xi   (\gamxt[   ])\|dt = a\|\xi   (\gamx{   }{t_{\ell + 1}})\| + \frac{1}{2}a^2 \|\xi   (\gamx{   }{\tilde t_{\ell+1}})\| \theta(\gamx{   }{\tilde t_{\ell+1}}).
\end{align*}
As we have shown in (\ref{thetanegative}), $\theta(\gamx{   }{\tilde t_{\ell+1}})<0$, for all $\ell\geq0$. Hence by (\ref{xistarratio}),
\begin{align}\label{numeratorbound}
\int_{t_{\ell+1}}^{t_{\ell+2}} \|\xi   (\gamxt[   ])\|dt \leq &\; a\; \|\xi   (\gamx{   }{t_{\ell+1}})\| \nonumber\\
\leq &\; a\;\|\xi   (\gamx{   }{ t_\ell})\| \left[1 + a\theta(\gamx{   }{t_{\ell}})+ a^2(\kappa^*)^2 + \frac{1}{4} (a+a^2\kappa^*)^2\lambda_{\max}^* \right].
\end{align}

Now we turn back to $\int_{t_{\ell}}^{t_{\ell+1}} \|\xi   (\gamxt[   ])\|dt$. Using (\ref{xiintegral}) -- \eqref{xinormdecrease}, we have
\begin{align}\label{tellfirstbound}
\int_{t_{\ell}}^{t_{\ell+1}} \|\xi   (\gamxt[   ])\|dt \geq a\|\xi   (\gamx{   }{t_\ell})\| + \frac{1}{2}a^2 \|\xi   (\gamx{   }{t_\ell})\| \theta(\gamx{   }{\tilde t_\ell}).
\end{align}
A Taylor expansion for $\theta(\gamx{   }{\tilde t_\ell})$ gives
\begin{align}\label{thetaexpansion}
    \theta(\gamx{   }{\tilde t_\ell} = \theta(\gamx{   }{t_\ell} + (\delta_\ell a) \frac{\partial}{\partial t} \theta(\gamxt[   ]) \big|_{t=\tilde t_\ell^*},
\end{align}
where $\tilde t_\ell^* =(1-\delta_\ell^*) t_\ell + \delta_\ell^* t_{\ell+1}=t_\ell + \delta_\ell^* \times a$ for some $\delta_\ell^*\in(0,\delta_\ell).$ Here it can be shown that
\begin{align*}
    &\frac{\partial}{\partial t} \theta(\gamxt[   ])
     = \pi(\gamxt[   ]),
\end{align*}
where 
\begin{align*}
\pi(y) &= \xi_\diamond   (y)^\top \{[\nabla \xii{   }{y}]^\top \nabla \xii{   }{y} + \nabla \xii{   }{y}\nabla \xii{   }{y}\}\xi_\diamond   (y) + \xi_\diamond   (y)^\top[\xii{   }{y}^\top \otimes \mathbf{I}_d]\nabla(\nabla\xii{   }{y}) \xi_{\diamond}   (y)  \\
&\hspace{1cm} - 2[\xi_\diamond   (y)^\top \nabla\xii{   }{y} \xi_{\diamond}   (y)]^2.
\end{align*}
When $\epsilon>0$ is small enough, we have that 
\begin{align*}
 \kappa^\ddagger := \sup_{y\in \Seps{}\backslash\text{Ridge}(f)}\left|\pi(y) \right| <\infty.
\end{align*}
Then from (\ref{thetaexpansion}) we have
\begin{align*}
    \theta(\gamx{   }{\tilde t_\ell}) \geq  \theta(\gamx{   }{t_\ell}) - a \kappa^\ddagger.
\end{align*}
Plugging this result into (\ref{tellfirstbound}) we get
\begin{align}\label{demoniatorbound}
\int_{t_{\ell}}^{t_{\ell+1}} \|\xi   (\gamxt[   ])\|dt \geq a\|\xi   (\gamx{   }{t_\ell})\| \Big\{ 1+ \frac{1}{2}a  [\theta(\gamx{   }{t_\ell}) - a \kappa^\ddagger]\Big\}.
\end{align}
Here we require $a< \min\{1, ( \kappa^\ddagger -\beta_1)^{-1}\}$ so that the right-hand side of (\ref{demoniatorbound}) is positive, by \eqref{thetalower}. Now combining (\ref{numeratorbound}) and (\ref{demoniatorbound}) we have
\begin{align}\label{xistaratio}
\frac{\int_{t_{\ell+1}}^{t_{\ell+2}} \|\xi   (\gamxt[   ])\|dt}{\int_{t_{\ell}}^{t_{\ell+1}} \|\xi   (\gamxt[   ])\|dt} & \leq   \frac{1 + a\theta(\gamx{   }{t_\ell})+ a^2(\kappa^*)^2 + \frac{1}{4} (a+a^2\kappa^*)^2\lambda_{\max}^* }{1+ \frac{1}{2}a  [\theta(\gamx{   }{t_\ell}) - a \kappa^\ddagger]}.
\end{align}
Using (\ref{thetanegative}), we can show that if we further require $a\leq \frac{-\beta_0}{4(\kappa^*)^2 + (1+\kappa^*)^2\lambda_{\max}^* +2 \kappa^\ddagger}$, then 
\begin{align}\label{ratio-bound}
\frac{\int_{t_{\ell+1}}^{t_{\ell+2}} \|\xi   (\gamxt[   ])\|dt}{\int_{t_{\ell}}^{t_{\ell+1}} \|\xi   (\gamxt[   ])\|dt} & \leq   1+\frac{1}{4}a\beta_0.
\end{align}
Then
\begin{align}\label{totalengthbound2}
    \int_0^\infty \|\xi   (\gamxt[   ])\|dt & \leq \int_0^{t_1} \|\xi   (\gamxt[   ])\|dt \sum_{i=0}^\infty  \Big(1+ \frac{1}{4}a\beta_0\Big)^i \nonumber\\
    & \leq \frac{\int_0^{t_1} \|\xi   (\gamxt[   ])\|dt}{1-(1+ \frac{1}{4}a\beta_0)} 
    \le   \frac{4\kappa^\dagger(\epsilon)}{-\beta_0} \le C^\prime\sqrt{\epsilon},
\end{align}
where $\kappa^\dagger(\epsilon)$ is given in (\ref{kappadagger}) and $C^\prime=2C$ with $C$ given in Fact~\ref{fact2}, by using \eqref{kappadaggerbound}. 
\end{proof}

The next fact concerns the comparison of two sequences: $\{\gamx{   }{t_\ell},\;\ell\geq0\}$ and $\{\gamx[a]{   }{\ell},\;\ell\geq 0\}$, where as above, $t_\ell = a \ell$.

\begin{customfact}{5}
\label{fact5}
When $\epsilon>0$ is small enough, there exists $a_0>0$ such that when $a \leq a_0$, we have 
\begin{align*}
\sup_{x\in\partial \Seps{}}\|\lim_{\ell\rightarrow\infty} \gamx{   }{t_\ell} - \lim_{\ell\rightarrow\infty} \gamx[a]{   }{\ell}\| \leq C a^{1-\sigma_0}.
\end{align*}
for some constant $C>0$, where $0<\sigma_0<1$ is given in (\ref{sigma0exp}).
\end{customfact}

\begin{proof} We will use some similar arguments as in the proof of Theorem 1 in~\cite{arias2016estimation}.  Let $e_\ell = \gamx{   }{t_\ell} - \gamx[a]{   }{\ell}$. Then
\begin{align}
\label{eell}
    e_{\ell+1} &=  \gamx{   }{t_{\ell+1}} - \gamx[a]{   }{\ell+1} \nonumber\\
    &= e_\ell + [\gamx{   }{t_{\ell+1}} - \gamx{   }{t_{\ell}} ] - [\gamx[a]{   }{\ell+1} - \gamx[a]{   }{\ell}] \nonumber\\
    &= e_\ell + [\gamx{   }{t_{\ell+1}} - \gamx{   }{t_{\ell}} - a\xi   (\gamx{   }{t_{\ell}}) ] + a[\xi   (\gamx{   }{t_{\ell}}) - \xi   (\gamx[a]{   }{\ell})].
\end{align}
As in the proof of Fact~\ref{fact1}, we choose $\epsilon$ small enough that $\Seps{}\subset\text{Ridge}(f)^{(\frac{1}{2}\delta^\prime)}$. By using Facts~\ref{fact2} and \ref{fact4}, we can choose $\epsilon$ small enough that $\alpha\gamx{   }{t_{\ell}} + (1-\alpha)\gamx[a]{   }{\ell}\in \text{Ridge}(f)^{\delta^\prime}$. for all $\alpha\in[0,1]$ and $\ell\ge0$. We then have
\begin{align*}
    \|\xi   (\gamx{   }{t_{\ell}}) - \xi   (\gamx[a]{   }{\ell})\| \leq \kappa^* \|\gamx{   }{t_{\ell}}) - \gamx[a]{   }{\ell}\| = \kappa^* \|e_\ell\|,
\end{align*}
where $\kappa^*$ is defined in (\ref{kappastar}). Moreover, we also have
\begin{align*}
    \gamx{   }{t_{\ell+1}} - \gamx{   }{t_{\ell}} - a\xi   (\gamx{   }{t_{\ell}}) = \int_{t_\ell}^{t_{\ell+1}} [\xi   (\gamxt[   ])-\xi   (\gamx{   }{t_\ell})]dt.
\end{align*}
Hence,
\begin{align*}
    \|\gamx{   }{t_{\ell+1}} - \gamx{   }{t_{\ell}} - a\xi(\gamma(x,t_{\ell}))\| 
    &\leq  \int_{t_\ell}^{t_{\ell+1}} \|\xi   (\gamxt[   ])-\xi   (\gamx{   }{t_\ell})\| dt \\
    &\leq  \kappa^* \int_{t_\ell}^{t_{\ell+1}} \|\gamxt[   ]-\gamx{   }{t_\ell}\| dt \\
    &\leq  \kappa^*\kappa^\dagger \int_{t_\ell}^{t_{\ell+1}} |t-t_\ell|dt \\
    &= \frac{1}{2}\kappa^*\kappa^\dagger (t_{\ell+1} - t_\ell)^2 = \frac{1}{2}a^2\kappa^*\kappa^\dagger,
\end{align*}
where $\kappa^\dagger=\kappa^\dagger(\epsilon)$ is given in (\ref{kappadagger}). It then follows from \eqref{eell} that
\begin{align*}
    \|e_{\ell+1}\| \leq (1+a\kappa^*)\|e_\ell\| + \frac{1}{2}a^2\kappa^*\kappa^\dagger.
\end{align*}
Using Gronwall's inequality~\cite[see][Lemma 4]{arias2016estimation}, we have
\begin{align*}
    \|\gamx{   }{t_\ell} - \gamx[a]{   }{\ell}\| =\|e_{\ell}\|\leq \frac{1}{2} [e^{a\ell\kappa^*}-1]\kappa^\dagger a.
\end{align*}
From (\ref{totalengthbound}) in the proof of Fact 3, we know
\begin{align*}
   \|\gamx[a]{   }{\ell} - \lim_{\ell^\prime \rightarrow\infty} \gamx[a]{   }{\ell^\prime}\| \leq  \kappa^\dagger \frac{[1- ac_1]^\ell}{c_1} \leq \frac{\kappa^\dagger }{c_1} e^{-\ell a c_1}.
\end{align*}
Using (\ref{ratio-bound}) in the proof of Fact 4 and noticing that $c_1=-\beta_0/2$, the same arguments as in the derivation of (\ref{totalengthbound}) give that
\begin{align*}
   \|\gamx{   }{t_{\ell}} - \lim_{\ell^\prime\rightarrow\infty} \gamx{   }{t_{\ell^\prime}}\| \leq  2\kappa^\dagger \frac{[1- \frac{1}{2}ac_1]^\ell}{c_1} \leq \frac{2\kappa^\dagger }{c_1} e^{-\frac{1}{2}\ell a c_1}.
\end{align*}
Hence combining the above three inequalities, we get
\begin{align*}
 \| \lim_{\ell\rightarrow\infty} \gamx{   }{t_\ell} - \lim_{\ell^\prime \rightarrow\infty} \gamx[a]{   }{\ell^\prime} \| \leq  \min_{\ell\geq 0} \psi(\ell), \qquad\text{where}\quad \psi(\ell) =  \frac{1}{2} \kappa^\dagger a e^{a\ell\kappa^*} + \frac{3\kappa^\dagger }{c_1} e^{-\frac{1}{2}\ell a c_1} .
\end{align*}
Denote the upper bound we have imposed on $a$ by $a_0$. By choosing 
\begin{align}\label{sigma0exp}
\ell= \ell_a(\sigma_0):=\Big\lceil \frac{\sigma_0}{a\kappa^*} \log{\frac{1}{a}} \Big\rceil, \text{ where }    \sigma_0 = \frac{1}{1+c_1/(2\kappa^*)},
\end{align}
we have
\begin{align}
 \| \lim_{\ell\rightarrow\infty} \gamx{   }{t_\ell} - \lim_{\ell\rightarrow\infty} \gamx[a]{   }{\ell} \| \leq \psi(\ell_a(\sigma_0)) \leq \kappa^\dagger  \Big(\frac{1}{2} e^{a_0\kappa^*}  + \frac{3}{c_1} \Big) a^{1-\sigma_0}.
\end{align}
\end{proof}

We have the following continuous version of Fact~\ref{fact3}.
\begin{customfact}{6}
\label{fact6}
The following holds when $\epsilon$ is small enough: Let $x,x_0\in\partial \Seps{}$ be two starting points. For any $\omega>0$, there exists $\delta_\omega>0$ such that when $\|x-x_0\|\leq \delta_\omega$, we have 
\begin{align*}
\big\|\lim_{t\rightarrow\infty} \gamxt[   ] - \lim_{t\rightarrow\infty} \gamma^{   }(x_0,t)\big\| \leq \eta.
\end{align*}
\end{customfact}

\begin{proof}
Note that for $\tilde x\in\{x,x_0\}\subset \partial \Seps{}$,
\begin{align*}
    \gamma   (\tilde x,t) = \tilde x + \int_0^t \xi   (\gamma   (\tilde x,s))ds.
\end{align*}
When $\epsilon$ is small enough, we obtain by a Taylor expansion that for any $t\geq 0$,
\begin{align}
    \|\gamxt[   ] - \gamma   (x_0,t)\|
    & \leq \|x-x_0\| + \int_0^{t} \|\xi   (\gamx{   }{s}) - \xi   (\gamma   (x_0,s))\|ds \nonumber\\
  &\leq  \|x-x_0\| + \kappa^* \int_0^{t} \|\gamx{   }{s} - \gamma   (x_0,s)\| ds,
\end{align}
where $\kappa^*$ is given in \eqref{kappastar}. Using Gronwall's inequality, we have 
\begin{align*}
\|\gamxt[   ] - \gamma   (x_0,t)\| \leq \|x-x_0\| e^{\kappa^* t}.
\end{align*}
This is similar to (\ref{Gronwallineq}). The rest of the proof is similar to the part below (\ref{Gronwallineq}) in the proof of Fact~\ref{fact3}, where we replace the application of Fact~\ref{fact2} by that of Fact~\ref{fact4}. Details are omitted.
\end{proof}

With all the above facts, now we are ready to complete the proof of Theorem~\ref{discretecomp}.

\begin{proof}
Part (i) follows from Fact~\ref{fact1}. Part (ii) is a consequence of Theorem~\ref{pathconvergence} and Fact~\ref{fact5}.

To show (iii), we only need to show that $\text{Ridge}(f) \subset R_a(f)$ when $k = 1$, because of (i). In other words, for any $\bar x\in\text{Ridge}(f)$, we want to show that there exists $x\in\partial \Seps{}$ such that $\bar x = \lim_{\ell\rightarrow\infty}\gamx[a]{   }{\ell}$.

Note that $\text{Ridge}(f)$ is a union of finitely many 1-dimensional closed curves when $k=1$ under our assumptions (see the discussion of the assumptions in Sec~\ref{assumptions}). We focus on one of the closed curves (also call it $\text{Ridge}(f)$), and parametrize it by $\zeta: [0,1]\rightarrow \text{Ridge}(f) $ with $\zeta(0)=\zeta(1)$. Without loss of generality, we set $\bar x = \zeta(0)$, and assume that the total length of the ridge is 1 and the parametrization is by the arclength. Under our assumptions, $\text{Ridge}(f)$ is $C^2$-smooth, and has positive reach~\citep{scholtes2013hypersurfaces}.
Due to Fact~\ref{fact5}, for a fixed (small) $\epsilon_0$, when $a$ and $\epsilon$ are small enough, there exist $x_1,x_2\in \partial \Seps{}$ such that 
\begin{align*}
& \lim_{\ell\rightarrow\infty} \gamma   _a(x_1,\ell) \in \{\zeta(\alpha):\; \alpha\in(0,\epsilon_0]\},\\
& \lim_{\ell\rightarrow\infty} \gamma   _a(x_2, \ell) \in \{\zeta(\alpha):\; \alpha\in[1-\epsilon_0,1)\}.
\end{align*}
In other words, the limit points of $\gamma   _a(x_1,\ell)$ and $\gamma   _a(x_2,\ell)$ are not far away from $\bar x$ and they are located on two sides of $\bar x$.  Let $\alpha_1\in(0,\epsilon_0]$ and $\alpha_2\in[1-\epsilon_0,1)$ be such that $\lim_{\ell\rightarrow\infty} \gamma   _a(x_1,\ell) = \zeta(\alpha_1)$ and $\lim_{\ell\rightarrow\infty} \gamma   _a(x_2,\ell) = \zeta(\alpha_2)$. Let $\overline{x_1x_2}$ be the shortest curve (a connected set) in $\partial \Seps{}$ connecting $x_1$ and $x_2$. When $a$ is small enough, we have 
\begin{align}\label{gammaxtrange}
\big\{\lim_{t\rightarrow\infty} \gamxt[   ]: \; x\in\overline{x_1x_2}\big\} \subset \big\{\zeta(\alpha):\; \alpha\in[0,2\epsilon_0]\cup[1-2\epsilon_0,1]\big\}.
\end{align}
Since the image of the continuous map from a connected set is a connected set, the set $$\gamma   _a(\overline{x_1x_2},\infty): = \big\{\lim_{\ell\rightarrow\infty} \gamx[a]{   }{\ell}:\; x\in \overline{x_1x_2}\big\}$$ 
is also connected. Since $\gamma   _a(\overline{x_1x_2},\infty)\subset \text{Ridge}(f),$ we must have $\bar x\in\gamma   _a(\overline{x_1x_2},\infty)$, because otherwise we have 
\begin{align*}
\{\zeta(\alpha):\; \alpha\in[\alpha_1,\alpha_2]\} \subset \gamma   _a(\overline{x_1x_2},\infty),
\end{align*}
which, however, contradicts (\ref{gammaxtrange}) and Fact~\ref{fact5}. Thus there exists $\bar x_0\in\overline{x_1x_2}$ such that $\lim_{\ell\rightarrow\infty}\gamma_a(\bar x_0,\ell)=\bar x,$ and we complete the proof.
\end{proof}

\section*{Acknowledgments}

We would like to thank the anonymous referees and the Action Editor for insightful comments that help significantly improve the quality of this paper. This work was partially supported by the US National Science Foundation (DMS 1821154).

\bibliography{library}

\end{document}